\documentclass{article}

\usepackage[preprint]{neurips_2026}

\usepackage[utf8]{inputenc}
\usepackage[T1]{fontenc}
\usepackage[activate=false]{microtype}
\usepackage{graphicx}
\usepackage{wrapfig}
\usepackage{subcaption}
\usepackage{booktabs}
\usepackage{url}
\usepackage{xcolor}
\usepackage{nicefrac}
\usepackage{amsmath}
\usepackage{amssymb}
\usepackage{amsfonts}
\usepackage{mathtools}
\usepackage{amsthm}
\usepackage{xspace}
\usepackage{bm}
\usepackage{thmtools}
\usepackage{thm-restate}
\usepackage{algorithm}
\usepackage{algorithmic}
\usepackage{longtable}
\usepackage{booktabs}
\usepackage{chngcntr}

\usepackage{amsmath,amsfonts,bm}

\def\eqref#1{equation~\ref{#1}}

\def\floor#1{\lfloor #1 \rfloor}
\def\1{\bm{1}}

\def\mB{{\mathbb{B}}}

\def\mE{{\mathbb{E}}}

\def\mR{{\mathbb{R}}}
\def\mS{{\mathbb{S}}}

\DeclareMathAlphabet{\mathsfit}{\encodingdefault}{\sfdefault}{m}{sl}
\SetMathAlphabet{\mathsfit}{bold}{\encodingdefault}{\sfdefault}{bx}{n}

\newcommand{\KL}{D_{\mathrm{KL}}}

\DeclareMathOperator*{\argmax}{arg\,max}
\DeclareMathOperator*{\argmin}{arg\,min}

\usepackage{hyperref}
\usepackage[capitalize,noabbrev]{cleveref}

\hypersetup{
    colorlinks=true,
    linkcolor=[rgb]{0.0,0.0,0.9},
    filecolor=[rgb]{0.0,0.0,0.9},
    urlcolor=[rgb]{0.0,0.0,0.9},
    citecolor=[rgb]{0.0,0.0,0.9},
    pdftitle={Learning What to Recommend: Minimax Optimal Simple Regret in Logistic Bandits},
    pdfpagemode=FullScreen,
}

\newcommand{\todos}[2][]{\todo[size=\scriptsize,color=green!20!white,#1]{SL: #2}}

\renewcommand{\algorithmiccomment}[1]{\hfill\textcolor{gray}{\footnotesize\texttt{// #1}}}
\newcommand{\Log}{\mathrm{Log}}
\newcommand{\Lin}{\mathrm{Lin}}

\newcommand{\mcL}{\mathcal{L}}

\newcommand{\bF}{\mathbf{F}}

\newcommand{\Ber}{\mathrm{Ber}}

\newcommand{\SR}{\mathfrak{R}} 

\newcommand{\conf}[1]{\cC_{#1}(\hat\theta_{#1},\delta)}
\newcommand{\MU}{\mathrm{MU}}

\newcommand{\bH}{\bm{H}}

\newcommand{\cA}{\mathcal{A}}

\newcommand{\cC}{\mathcal{C}}
\newcommand{\cD}{\mathcal{D}}
\newcommand{\cE}{\mathcal{E}}
\newcommand{\cF}{\mathcal{F}}
\newcommand{\cG}{\mathcal{G}}

\newcommand{\cK}{\mathcal{K}}

\newcommand{\cM}{\mathcal{M}}
\newcommand{\cN}{\mathcal{N}}
\newcommand{\cO}{\mathcal{O}}

\newcommand{\cU}{\mathcal{U}}
\newcommand{\cV}{\mathcal{V}}
\newcommand{\cW}{\mathcal{W}}

\newcommand{\cY}{\mathcal{Y}}

\newcommand{\EE}{\mathbb{E}}
\newcommand{\HH}{\mathbb{H}}

\newcommand{\II}{\mathbb{I}}

\newcommand{\NN}{\mathbb{N}}
\newcommand{\PP}{\mathbb{P}}

\newcommand{\RR}{\mathbb{R}}

\newtheorem{definition}{Definition}
\newtheorem{theorem}{Theorem}
\newtheorem{lemma}{Lemma}
\newtheorem{corollary}{Corollary}
\newtheorem{proposition}{Proposition}

\newtheorem{assumption}[definition]{Assumption}

\counterwithout{theorem}{section}
\counterwithout{lemma}{section}
\newcommand{\norm}[1]{||#1||}

\DeclarePairedDelimiterX{\inp}[2]{\langle}{\rangle}{#1, #2}

\usepackage[disable,textsize=tiny]{todonotes}

\newcommand{\MULin}{\textnormal{\textsc{MULin}}\xspace}
\newcommand{\MULog}{\textnormal{\textsc{MULog}}\xspace}
\newcommand{\LinTS}{\textnormal{\textsc{SimpleLinTS}}\xspace}
\newcommand{\THATS}{\textnormal{\textsc{THaTS}}\xspace}

\title{Learning What to Recommend: Minimax Optimal Simple Regret in Logistic Bandits}

\author{%
  Shuai Liu \\
  University of Alberta \\
  \texttt{shuailiu725@gmail.com}
  \And
  Alireza Bakhtiari\thanks{Work done during study at University of Alberta.} \\
  University of Washington, Seattle \\
  \texttt{alirezaa@cs.washington.edu}
  \And
  Alex Ayoub \\
  University of Alberta \\
  \texttt{aayoub@ualberta.ca}
  \And
  Botao Hao \\
  OpenAI \\
  \texttt{haobotao000@gmail.com}
  \And
  Csaba Szepesv\'ari \\
  University of Alberta \\
  \texttt{csaba.szepesvari@gmail.com}
}

\begin{document}

\maketitle

\begin{abstract}
We study stochastic logistic bandits with $d$-dimensional action features under the simple-regret objective, where a learner uses $T$ rounds of exploration to output a single final action. The logistic structure is essential here: because the informativeness of an action depends on the local curvature of the sigmoid, actions that are best for immediate reward need not be the most useful for identifying the best final recommendation. We show that the first-order minimax difficulty is governed by $\kappa_*$, the inverse slope of the sigmoid at the optimal action. The lower bound is realized by a shifted saturated hard family in which saturation simultaneously limits the information available about the final decision and controls the value loss from a wrong recommendation. This reveals a hard mechanism distinct from cumulative-regret constructions, even though online-to-batch reductions recover the same leading order in expectation. We then develop two curvature-aware algorithms: \MULog, a pure-exploration method whose final recommendation satisfies a high-probability upper bound of order $\tilde O(d/\sqrt{\kappa_* T})$, matching the lower bound up to logarithmic factors, and \THATS, a Thompson-sampling-style method that provides a computationally lighter alternative. Experiments on both hard and easy geometries support the same picture: informative low-reward actions can make instances substantially easier, and the curvature-aware methods exploit this structure especially effectively.
\end{abstract}

\section{Introduction}\label{sec:intro}

We study stochastic logistic bandits \citep{Filippi2010parametric} under the simple-regret objective \citep{bubeck2009pure}. In each round, the learner selects an action and observes a Bernoulli reward whose mean is a sigmoid function of the inner product between a known feature vector and an unknown parameter in $\RR^d$. After $T$ rounds, the learner is evaluated only by the quality of its final recommendation.

Logistic bandits are a basic model for sequential decisions with binary feedback: clicks, conversions, preferences, successes, and adverse events. They are also the simplest generalized-linear bandit model in which information is genuinely nonuniform. In saturated regions of the sigmoid, rewards are nearly deterministic and each sample carries little information; closer to the origin, the same feature direction can be much more informative. Thus the logistic link creates a geometry that is absent from linear bandits and central to pure exploration: high-reward actions need not be the most useful actions for identifying the final recommendation.

The simple-regret objective is natural when exploration and deployment are separated. A learner may receive a limited experimental budget for data collection and then deploy a single action or policy. This setting appears in A/B testing, clinical dose finding, recommendation pilots, preference learning, and reward-model data collection. In these applications, rewards gathered during exploration are secondary; what matters is the final decision.

\begin{wrapfigure}[18]{r}{0.52\textwidth}
    \vspace{-1.2em}
    \centering
    \includegraphics[width=0.52\textwidth]{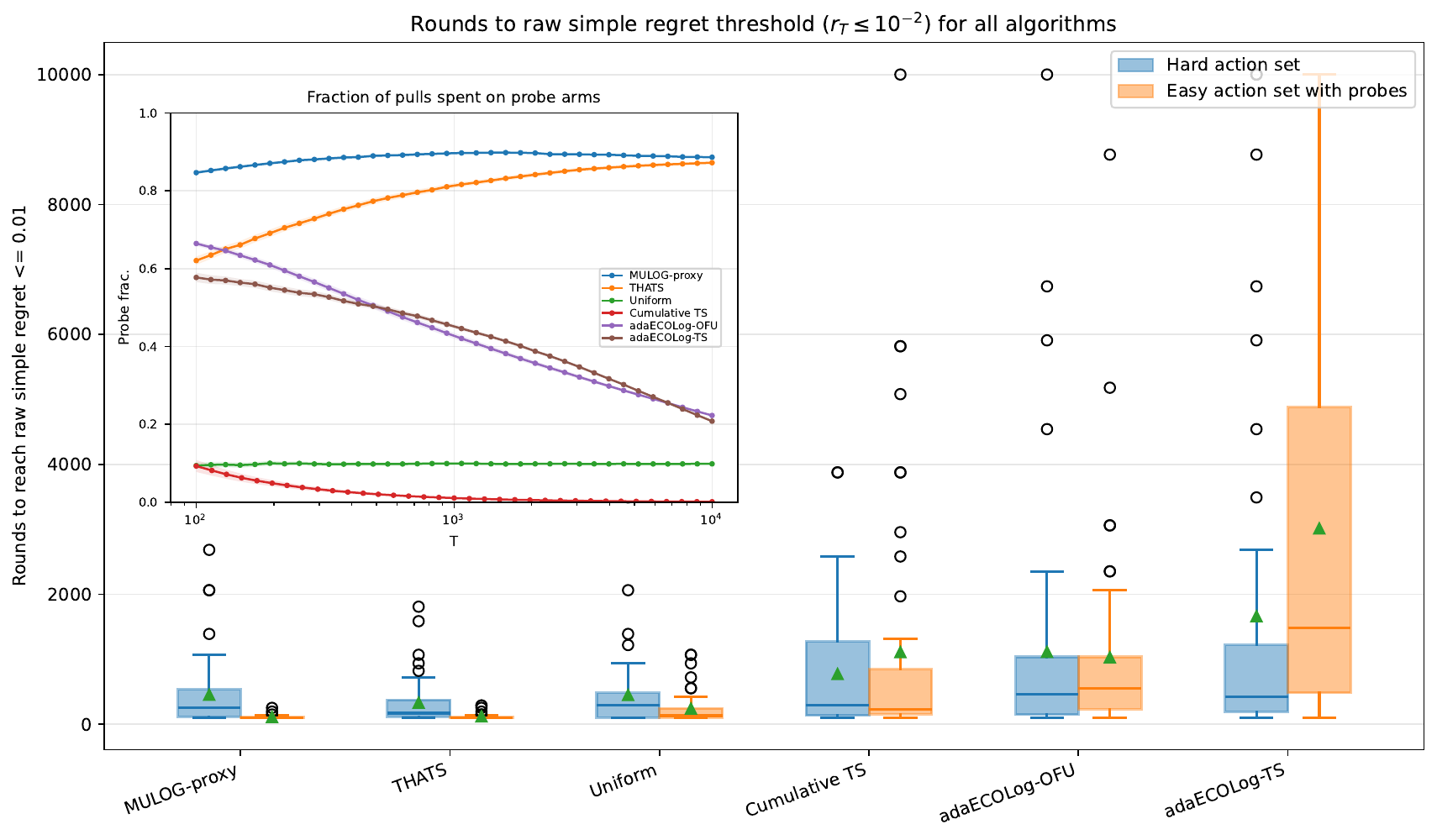}
    \caption{\textbf{Informative probes change the problem.} Adding low-reward probe arms can substantially reduce the rounds needed to reach raw simple regret $10^{-2}$. The inset shows the mechanism: pure-exploration methods keep sampling probes, while cumulative-regret methods avoid them.}
    \label{fig:intro-probes}
    \vspace{-1.0em}
\end{wrapfigure}

\Cref{fig:intro-probes} shows the phenomenon we want to understand. In a logistic instance with low-reward but informative probe arms, pure-exploration methods can identify a good final action much faster than methods designed for cumulative regret. The point is not merely that different algorithms have different constants. The point is that, for logistic feedback, the actions that are statistically useful can be different from the actions that are good to deploy. The full experimental setup and per-algorithm diagnostics are given in \cref{sec:experiment} and \cref{appendix:experiment}.

This raises a basic question: what is the first-order difficulty of logistic simple regret when the learner is free to sample for information? Cumulative-regret lower bounds for logistic bandits, such as that of \mbox{\citet{abeille2021instance}}, are tied to the cost of information: a low-regret learner must play near-optimal actions often, and therefore cannot freely sample directions that are useful but suboptimal. Simple regret removes that cost. The learner may spend the entire budget on actions that are bad to deploy if they are informative.

Our lower bound identifies a different obstruction. We construct shifted instances in a common saturated region of the sigmoid, where both statistical information and value sensitivity are controlled by the same local curvature. This makes the problem hard even for a learner that is allowed to sample purely for information. The resulting lower bound is of order $\Omega(\frac{d}{\sqrt{T\kappa_*}})$, where $\kappa_*$ is the inverse slope of the sigmoid at the optimal action. Equivalently, if $\sigma_* = 1/\kappa_*$ is the reward variance at the optimal action, the bound is minimax over instances with the same $\sigma_*$.

This minimax rate captures the first-order saturated worst case, but it does not exhaust the geometry of logistic pure exploration. Suboptimal actions can be highly informative, as in \cref{fig:intro-probes}, and this can make an instance much easier than the shifted lower-bound family. The algorithmic problem is therefore to control uncertainty in the right curvature-weighted feature directions while still allowing the learner to exploit such informative probes.

Motivated by this geometry, we design Max-Uncertainty-Log (\MULog), a deterministic pure-exploration algorithm with simple regret $\tilde \cO(\frac{d}{\sqrt{T\kappa_*}})$, matching the lower bound up to polylogarithmic factors and lower-order terms. \MULog maintains confidence sets for the unknown parameter, uses them to build a conservative curvature-weighted proxy for the logistic Hessian, and samples the action-parameter pair with largest certified uncertainty. This gives the first direct pure-exploration guarantee matching the minimax logistic simple-regret rate up to logarithmic factors, while controlling the final returned action with high probability. For finite action sets, each action-selection step reduces to a finite collection of convex optimization problems.

We also design and analyze Try Hard Thompson Sampling (\THATS), a randomized logistic algorithm that is computationally lighter than \MULog. \THATS samples a zero-centered Gaussian direction with covariance given by the curvature-weighted design matrix and selects the action with the highest randomized uncertainty score. Its analysis compares this one-sample randomized score to the deterministic max-uncertainty direction, yielding a regret bound of $\tilde \cO(\frac{d^{3/2}}{\sqrt{T\kappa_*}})$. The extra $\sqrt d$ factor comes from using a geometry-agnostic randomized approximation, which is the price of avoiding additional structure assumptions on the action set.

Pure exploration for generalized linear bandits has been studied in fixed-confidence and fixed-budget best-arm identification, primarily for finite-arm structured problems \citep{kazerouni2021best,azizi2022fixed}. These methods are not designed for the fixed-budget, high-probability final-recommendation guarantees we need here, and adapting them to our setting would still have to address the logistic curvature issue that drives our analysis.

The logistic bandit model was introduced by \citet{Filippi2010parametric}. For cumulative regret, early bounds carried a potentially large factor $\kappa$; this factor was later pushed to lower-order terms by \citet{faury2020improved}. \citet{abeille2021instance} obtained an instance-wise cumulative-regret bound of $\tilde O(d\sqrt{T/\kappa_*})$ with a matching lower bound. Together with the standard online-to-batch conversion, this known cumulative-regret guarantee gives a useful expected simple-regret comparison: the conversion returns a uniformly random past action, giving only an expectation guarantee, and the resulting methods, as illustrated in \cref{fig:intro-probes}, fail to exploit the informative low-reward actions.
The cumulative-regret literature also includes Thompson-sampling and randomized-exploration methods for logistic and generalized linear bandits \citep{Dong2019Performance,kveton2020randomized,ding2021efficient,jun2017scalable}, building on the linear Thompson-sampling analysis of \citet{abeille2017linear}. The ECOLog algorithms of \citet{faury2022jointly}, including the adaptive ECOLog baselines used in \cref{fig:intro-probes}, achieve near-optimal cumulative-regret guarantees.
Unlike these cumulative-regret methods, \MULog and \THATS exploit this finer geometry and give direct high-probability guarantees for the final recommendation.

\textbf{Our main contributions are as follows.}
\textbf{First,} we show that simple regret in logistic bandits has a distinct hard mechanism from cumulative regret. We give a shifted saturated lower-bound construction in which value sensitivity and statistical information are controlled by the same local curvature, yielding a lower bound of order $\Omega(\frac{d}{\sqrt{T\kappa_*}})$. This identifies the correct first-order worst-case geometry at fixed optimal-arm variance.
\textbf{Second,} we design curvature-aware pure-exploration algorithms tailored to this geometry. \MULog matches the lower bound up to logarithmic factors via an online conservative curvature-weighted Hessian proxy and a decreasing-uncertainty argument. Its guarantee is direct and holds with high probability for the final returned action. \THATS provides a computationally lighter randomized alternative with guarantee $\widetilde O(d^{3/2}/\sqrt{T\kappa_*})$, and its analysis relates randomized exploration to deterministic max-uncertainty rules.
\textbf{Third,} we show that the same geometric viewpoint goes beyond the worst-case lower-bound family. Our experiments demonstrate that informative low-reward actions can make the problem substantially easier, and that \MULog and \THATS exploit this finer structure especially effectively.

\section{Preliminaries}
We introduce notation, the problem formulation, and the logistic-regression objects used later.
\subsection{Notation}
For a real-valued single-variable differentiable function $f:\RR\to\RR$, we use $\dot f, \ddot f$ to denote the first and second order derivatives of $f$, respectively.
The $\ell_2$-norm on $\RR^d$ is denoted by $\|\cdot\|$.
We use $\mS^{d-1}$ and $\mB_d(r)$ to denote the unit sphere and the $d$-dimensional ball with radius $r$ in $(\mR^d,\|\cdot\|)$, respectively.
For a positive definite matrix $A$, we define the norm induced by $A$ to be $\|x\|_A = \sqrt{x^\top Ax}$.
For a set $\cK$, let $\cM_1(\cK)$ denote the set of all distributions over $\cK$;
we assume the associated measurability structure is clear from context.  
For a distribution $P \in \cM_1(\cK)$, we write $\mathrm{supp}(P) \subseteq \cK$ to denote its support.
The abbreviation ``a.s.'' refers to statements that hold almost surely. For positive semidefinite matrices $P,S$, we use $P\preceq S$ to denote that $S-P$ is positive semidefinite. We use $\RR^+$ to denote the positive reals. For $a,b\in \RR$, $a\asymp b$ denotes $a$ equals to $b$ up to constant factors.

\subsection{Problem Setup}
A logistic bandit instance is described by a tuple $(\cA,\phi,\theta_*)$, where $\cA$ is the action space (possibly infinite),
Next, $\phi:\cA\to\RR^d$ is a feature map and $\theta_*$ is the underlying parameter vector such that 
given an action $a\in \cA$, the reward is a Bernoulli distribution with mean $\mu(\phi(a)^\top\theta_*)=\frac{1}{1+\exp(-\phi(a)^\top\theta_*)}$.
For $T\ge 1$, the learner interacts with the environment in rounds $t=1,2,\ldots,T$.
The learner knows $\cA$ and $\phi$, but does not know $\theta_*$. 
At the beginning of each round $t$, 
given its past information, the learner
chooses an action $A_t \in \cA$, after which they receive 
the reward 
\begin{align*}
    X_t \sim P_{\theta_*}(\cdot \vert A_t)=\Ber(\mu(\phi(A_t)^\top\theta_*))\,.
\end{align*}

After $T$ rounds, the learner returns a possibly randomized recommendation represented by $\pi\in \cM_1(\cA)$.
%
We assume the existence of an optimal action $a_*(\theta_*)\in \argmax_{a\in \cA}\mu(\phi(a)^\top\theta_*)$. When the identity of $\theta_*$ is clear from the context, we denote the optimal action as $a_*$.
The simple regret of $\pi$ is defined to be
\begin{align}
    \SR(\pi, \theta_*) &= \int_{\cA}\left[\mu(\phi(a_*)^\top\theta_*) - \mu(\phi(a)^\top\theta_*)\right]\pi(da)\,.
\end{align}
This is the loss compared to playing an optimal action. 

%

We assume that the features are normalized in the following sense:
\begin{assumption}
\label[assumption]{ass:bounded_feature}
    For all $a\in \cA$, $\|\phi(a)\|\le 1$. 
\end{assumption}
The algorithms also need to know an upper bound on the norm of the unknown parameter:
\begin{assumption}
\label[assumption]{ass:bounded_para}
    There is a \emph{known} constant $S>0$ such that $\|\theta_*\|\le S$.
\end{assumption}
These assumptions are standard when studying generalized linear bandits \citep{faury2020improved,abeille2021instance,janz2024exploration,lee2024unified,liu2024almost}. 

The nonlinearity of the mean function $\mu$ with respect to the parameter $\theta$ is measured by the following quantity \citep{Filippi2010parametric,faury2020improved}
\begin{align}
    \kappa = \sup_{u\in \phi(\cA),\theta\in \mB_d(S)}\frac{1}{\dot\mu(u^\top\theta)}\,.
    \label{eq:kappas}
\end{align}
Here, $\phi(\cA) = \{\phi(a)\,:a\in \cA\}$, as usual.
In particular, the nonlinearity at the optimal action, as discussed in \cref{sec:intro}, is important for characterizing the instance-dependent behavior: 
\begin{align*}
    \kappa_*(\theta_*)=\frac{1}{\dot\mu(a_*^\top\theta_*)}.
\end{align*}
When the identity of $\theta_*$ is clear from context, we abbreviate $\kappa_*(\theta_*)$ as $\kappa_*$.
For the sigmoid link, $1/\dot\mu(z)=2+\exp(z)+\exp(-z)$.
Thus $\kappa$ can scale exponentially with the size of the admissible parameter set $S$, for example, when 
$\phi(\cA)=\mB_d(1)$.

\subsection{Logistic regression}
The algorithms estimate $\theta_*$ via regularized maximum likelihood. Given data $\cD=\{(\phi(A_i),X_i)\}_{i=1}^{t-1}$, define the regularized negative log-likelihood



\begin{align}
    \label{eq:logistic-loss}
    \mcL_t^\lambda(\theta) = \lambda \|\theta\|^2
    - \sum_{i=1}^{t-1} \ell(\mu(\phi(A_i)^\top \theta),X_i),
\end{align}
where $\ell(x,y)=y\log(x)+(1-y)\log(1-x)$ is the binary cross-entropy function. 
The regularized maximum likelihood estimate (MLE) is $\hat\theta_t = \argmin_{\theta\in\mR^d} \mcL_t^\lambda(\theta)$.
Since $\mcL_t^\lambda$ is strongly convex, $\hat\theta_t$ is its unique stationary point.
Defining
$g_t(\theta)=\sum_{i=1}^{t-1}\mu(\phi(A_i)^\top \theta)\phi(A_i)+\lambda \theta$,
it follows that
\begin{align}\label{eq:rmle_grad}
    \nabla_\theta\mcL_t^\lambda(\hat\theta_t)=g_t(\hat\theta_t)-\sum_{i=1}^{t-1} \phi(A_i)X_i=0,
\end{align}
and this equation has no other solution. The Hessian is
\begin{align}\label{eq:losshess}
    H_t(\theta) &:=\nabla_{\theta}^2 \mcL_t(\theta)
    = \lambda I+\sum_{i=1}^{t-1} \dot\mu(\phi(A_i)^\top \theta)\phi(A_i)\phi(A_i)^\top\,.
\end{align}

\section{Main Results}\label{sec:algorithms}
We first give the lower bound, which isolates the curvature-driven hardness of logistic simple regret and separates it from cumulative-regret mechanisms. We then give two curvature-aware upper bounds — a deterministic algorithm matching the lower bound up to logarithmic factors and a Thompson-sampling variant lighter at the cost of an extra $\sqrt d$ factor — both borrowing tightly coupled components from cumulative-regret work \citep{faury2020improved,abeille2021instance}; the central challenge is coordinating proxy growth and confidence-set shrinkage so that the global $\kappa$ cancels, which in the linear case is automatic but under the logistic link requires new analysis.
\subsection{Lower Bound}

In this section we present a lower-bound construction that yields a
hypercube-type simple regret lower bound of order
\(
    \Omega\left(\frac{d}{\sqrt{\kappa_* T}}\right).
\)
Although online-to-batch reductions from cumulative regret can already recover this leading expected order, they do not explain what makes it unavoidable under the simple-regret objective.
This construction is conceptually different from the cumulative-regret
lower bound of \cite{abeille2021instance}. Their lower bound is based on
trajectory tracking: if a learner has small cumulative regret, then it must
play close to the optimal arm for most rounds. Consequently, some directions
remain under-explored, and perturbations in those directions generate hard
nearby alternatives. This mechanism captures the exploration--exploitation
tradeoff inherent to cumulative regret.

For simple regret, however, this argument does not directly apply. The
learner is evaluated only through its final recommendation, and therefore it
may spend many rounds on informative but instantaneously suboptimal actions.
The relevant obstruction is instead information-theoretic: the interaction can reveal only limited information about the unknown parameter, and residual ambiguity after $T$ rounds propagates into final recommendation error.
The construction below isolates this worst-case geometry at a fixed value of $\kappa_*$; it does not rule out easier instances where other actions provide more informative observations about the final recommendation.

\begin{theorem}[Informal lower bound]
    Fix \(d\ge 1\). There exists a normalized shifted hypercube action set
    \(\phi(\mathcal A)\subseteq \mathbb B_{d+1}(1)\), a family of instances
    \(\Theta\), and universal constants
    \(c,C>0\) such that, for all \(T\ge C\kappa_* d^2\),
    \[
        \inf_\pi \sup_{\theta\in\Theta}
        \mathbb E_{\theta,\pi}\!\left[\SR^{\Log}(\theta,\pi)\right]
        \ge
        c\,\frac{d}{\sqrt{\kappa_* T}},
    \]
    where \(\kappa_*(\theta)= \kappa_*\) uniformly over
    \(\theta\in\Theta\) and $\EE_{\theta_*,\pi}$ is the expectation under the probability measure induced by the interconnection between $\pi$ and the bandit.
\end{theorem}

The action set is a normalized hypercube:
\(
    \phi(\mathcal A)
    =
    \left\{
    \phi(a_u)
    :=
    \frac{1}{\sqrt 2}
    \left(
        e_0+\frac{1}{\sqrt d}\sum_{j=1}^d u_j e_j
    \right)
    :
    u\in\{\pm1\}^d
    \right\}.
\)
For \(m,\varepsilon>0\), consider the family
\(
    \Theta
    =
    \left\{
    \theta_v
    :=
    \sqrt 2
    \left(
        m e_0+\frac{\varepsilon}{\sqrt d}
        \sum_{j=1}^d v_j e_j
    \right)
    :
    v\in\{\pm1\}^d
    \right\}.
\)
Then, for every pair \(u,v\in\{\pm1\}^d\),
\(
    \phi(a_u)^\top\theta_v
    =
    m+\varepsilon\frac{u^\top v}{d}
    \in [m-\varepsilon,m+\varepsilon].
\)
The optimal arm under \(\theta_v\) is \(a_v\), and 
\(
    \phi(a_v)^\top\theta_v=m+\varepsilon.
\)
Hence
\(
    \kappa_*(\theta_v)
    =
    \frac{1}{\dot\mu(m+\varepsilon)}=:\kappa_*
\)
and 
\(
    \dot\mu(\phi(a_u)^\top\theta_v)
    \asymp
    \dot\mu(m+\varepsilon)
    =
    \frac{1}{\kappa_*(\theta_v)}
\)
for all \(u,v\) as long as \(\varepsilon=O(1)\).

For simplicity of exposition, suppose that the final recommendation is a
single action \(a_{\hat u}\), where
\(\hat u\in\{\pm1\}^d\). Randomized recommendations can be handled by
conditioning on the algorithm's final randomization. If the true instance is
\(\theta_v\), then
\(
    (\phi(a_v)-\phi(a_{\hat u}))^\top\theta_v
    =
    \varepsilon
    \left(
        1-\frac{\hat u^\top v}{d}
    \right)
    =
    \frac{2\varepsilon}{d}H(\hat u,v),
\)
where \(H(\hat u,v)\) denotes Hamming distance. By the mean value
theorem,
\begin{align}
    \mathbb E_{\theta_v,\pi}\!\left[\SR^{\Log}(\theta_v,\pi)\right]
    =
    \mathbb E_{\theta_v,\pi}
    \left[
        \mu(\phi(a_v)^\top\theta_v)
        -
        \mu(\phi(a_{\hat u})^\top\theta_v)
    \right] \notag&\gtrsim
    \frac{1}{\kappa_*}
    \mathbb E_{\theta_v,\pi}
    \left[
        \phi(a_v)^\top\theta_v
        -
        \phi(a_{\hat u})^\top\theta_v
    \right] \notag\\
    &=
    \frac{2\varepsilon}{\kappa_*d}
    \mathbb E_{\theta_v,\pi}\!\left[H(\hat u,v)\right].
    \label{eq:lower_hamming}
\end{align}
Thus, each wrong sign produces a loss of
order \(\varepsilon/(\kappa_* d)\).
We now control the amount of information available to distinguish neighboring
hypercube vertices. Let \(v^{(i)}\) denote the sign vector obtained from
\(v\) by flipping only the \(i\)-th coordinate. For any played arm
\(a_u\),
\(
    \phi(a_u)^\top(\theta_v-\theta_{v^{(i)}})
    =
    \frac{2\varepsilon}{d}u_i v_i.
\)
Consequently, the one-step KL divergence satisfies for $Z_t=\phi(a_u)^\top\theta_v$ and $Z_t'=\phi(a_u)^\top\theta_{v^{(i)}}$,
\[
    \KL\!\bigg(
        \Ber(\mu(Z_t)),\Ber(\mu(Z_t'))
    \bigg)
    \lesssim
    \frac{(\mu(Z_t)-\mu(Z_t'))^2}{\dot\mu(Z_t')}
    \lesssim
    \frac{1}{\kappa_*}
    (Z_t-Z_t')^2
    \lesssim
    \frac{\varepsilon^2}{\kappa_* d^2}.
\]
By the chain rule for KL over the adaptive interaction,
\(
    \KL\!\left(
        \mathbb P_{\theta_v,\pi}
        ,
        \mathbb P_{\theta_{v^{(i)}},\pi}
    \right)
    \lesssim
    \frac{T\varepsilon^2}{\kappa_* d^2}.
\)
Assouad's lemma now converts this pairwise indistinguishability into a
lower bound on Hamming error. If
\(
    \frac{T\varepsilon^2}{\kappa_*d^2}
    \lesssim 1,
\)
then no policy can reliably distinguish \(v\) from all of its single-bit
flips. Formally,
\(
    \frac{1}{2^d}\sum_{v\in\{\pm1\}^d}
    \mathbb E_{\theta_v,\pi}
    \left[
        H(\hat u,v)
    \right]
    \gtrsim d.
\)
Choosing
\(
    \varepsilon
    \asymp
    d\sqrt{\frac{\kappa_*}{T}}
\)
makes the neighboring instances statistically indistinguishable, provided
\(T\gtrsim \kappa_*d^2\) so that \(\varepsilon=O(1)\). Therefore there
exists some \(v\in\{\pm1\}^d\) such that
\(
    \mathbb E_{\theta_v,\pi}
    \left[
        H(\hat u,v)
    \right]
    \gtrsim d.
\)
Plugging this into \cref{eq:lower_hamming} gives
\[
    \sup_{v\in\{\pm1\}^d}
    \mathbb E_{\theta_v,\pi}
    \left[
        \SR^{\Log}(\theta_v,\pi)
    \right]
    \gtrsim
    \frac{\varepsilon}{\kappa_*}
    \asymp
    \frac{d}{\sqrt{\kappa_* T}}.
\]
Since this holds for every policy \(\pi\), the desired minimax lower bound
follows.

        \begin{algorithm}[t]
        \caption{Max-Uncertainty-Log (\MULog)}
        \label[algorithm]{alg:max-uncertainty-logSR-GOAT}
        \begin{algorithmic}[1]
            \REQUIRE $\lambda > 0$, $S>0$
            \STATE $L_1=\lambda I,\cD = \{\}$, $\cW_0=\mathbb B_d(S)$ \COMMENT{$\mathcal D$ is a multiset}
            \FOR{$t: 1 \to T$}
                \STATE Solve for $\hat\theta_t=\argmin_{\theta\in \RR^d}\mcL_t^\lambda(\theta)$
                \STATE Build $\cW_t=\cW_{t-1}\cap \cC_t(\delta,\hat\theta_t)\cap \mB_d(S)$ (\cref{eq:wtconf})
                \STATE Select $(A_t,\theta_t) = \argmax_{a\in \cA,\theta\in \cW_t}U(a,\theta,L_t)$ (\cref{eq:action_selection_logMU}) and receive $X_t$
                \STATE Solve for $\theta_t'=\argmax_{\theta\in \cW_t}|\phi(A_t)^\top\theta|$ and update $L_{t+1}=L_t+\dot\mu(\phi(A_t)^\top\theta_t')\phi(A_t)\phi(A_t)^\top$ as in \cref{eq:proxy}
            \ENDFOR
            \STATE Solve for $\hat\theta_{T+1}=\argmin_{\theta\in \RR^d}\mcL_{T+1}^\lambda(\theta)$ and build $\cW_{T+1}$ (\cref{eq:wtconf})
            \STATE Pick any ${\theta}_{T+1}^{\Log}\in  \cW_{T+1}$.
            \STATE Return $\hat a:\argmax_{a\in \cA}\phi(a)^\top\theta_{T+1}^{\Log}$
        \end{algorithmic}
        \end{algorithm}

\subsection{Direct Uncertainty Maximization: Deterministic Algorithm \MULog}
\label[section]{sec:det_alg}
\MULog greedily selects the action with largest certified logistic uncertainty.
Max-uncertainty sampling is a standard pure-exploration idea \citep{kazerouni2021best,azizi2022fixed}; the new issue here is how to make the uncertainty score faithful to logistic curvature.
The two key steps are to construct a curvature-aware uncertainty score and to prove that this score decreases fast enough.
After exploration, it returns the action that is greedy for a parameter in the final confidence set. 
\label{section:MULog}
Define the uncertainty associated with an action
$a\in \cA$, a parameter vector $\theta\in \RR^d$ and a positive definite matrix $L$ using
\begin{align*}
    U(a,\theta,L):=\sqrt{\dot\mu(\phi(a)^\top\theta)}\|\phi(a)\|_{L^{-1}}.
\end{align*}
The factor $\|\phi(a)\|_{L^{-1}}$ measures directional parameter uncertainty, while $\sqrt{\dot\mu(\phi(a)^\top\theta)}$ accounts for the local sensitivity of the logistic mean.
The design matrix must therefore be curvature weighted: by \cref{eq:losshess}, each sample contributes $\dot\mu(\phi(A_i)^\top \theta_*)\phi(A_i)\phi(A_i)^\top$ to the local Hessian.
Curvature-weighted designs also appear in cumulative-regret analyses for logistic bandits \citep{faury2020improved,abeille2021instance}; here they are used to obtain a decreasing uncertainty certificate for the simple-regret objective.
Since $\theta_*$ is unknown, \MULog uses confidence sets to build a conservative lower bound.
In round $t$, \MULog uses the following matrix:
\begin{align}
    L_t = \lambda I + \sum_{i=1}^{t-1}\dot\mu(\phi(A_i)^\top\theta_i')\phi(A_i)\phi(A_i)^\top,\label{eq:proxy}
\end{align}
where $\theta_i'=\argmin_{\theta\in \cW_i} \dot\mu(\phi(A_i)^\top\theta)$
and $\cW_i$ is a confidence set available at the beginning of round $i$ that will be defined momentarily.
Since $\dot\mu(z)$ decreases as $|z|$ increases,
the minimization in \cref{eq:proxy} is equivalent to
\begin{align}
    \theta_i'=\argmax_{\theta\in\cW_i}|\phi(A_i)^\top\theta|\label{eq:min_variance}
\end{align}
For convex $\cW_i$, this reduces to two linear optimization problems over $\cW_i$.

The action-selection rule uses the largest plausible value of this uncertainty:
\begin{align}
    (A_t,\theta_t) = \argmax_{a\in \cA,\theta\in \cW_t} U(a,\theta,L_t)\,.
    \label{eq:action_selection_logMU}
\end{align}
For any fixed action, the optimization again reduces to a convex subproblem over $\cW_t$.

It remains to choose the confidence sets $(\cW_t)_{t\ge 1}$ and the policy returned at the end.
For the confidence set construction,  the algorithm first solves for the unconstrained MLE $\hat\theta_t$. 
Letting
\begin{align}
    \cC_t(\delta,\theta_{\circ}) = \left\{\theta:\mcL_t^\lambda(\theta) - \mcL_t^\lambda(\theta_{\circ})\le \beta_t(\delta)\right\},
    \label{eq:ctdef}
\end{align}
where for $\delta\in(0,1]$, $\beta_t(\delta):\RR^+\to\RR^+$ is an increasing function in $t$ whose value is introduced in \cref{appendix:confset} in detail, we choose 
\begin{align}
\cW_t = \cap_{i=1}^{t} \cC_i(\delta,\hat\theta_i) \cap \mB_d(S)\,, \label{eq:wtconf}
\end{align} 
so that $\{\cW_t\}_{t\ge 1}$ is decreasing---a key requirement for the uncertainty argument. Here $\cC_t(\delta,\hat\theta_t)$ is a confidence set containing $\theta_*$ with probability at least $1-\delta$ simultaneously for all $t\ge 1$.

Hence $\theta_*$ lies in $\cW_t$ for all $t\ge 1$ with probability at least $1-\delta$. At the end, we construct $\cW_{T+1}$, pick any vector $\theta_{T+1}^\Log\in \cW_{T+1}$, and return an action greedy for this vector. 
For finite action sets, \cref{alg:max-uncertainty-logSR-GOAT} can be implemented by solving a finite collection of convex subproblems: lines 3 and 5 are convex, and line 6 iterates over actions with a convex subproblem for each fixed action.

%
%
The following lemma is stated in a slightly more general form to cover the analysis of \THATS (\cref{sec:TS}) as well. The proof is deferred to \cref{appendix:proof-mulog}. When applying it to \MULog, we let $K=L_t$, $K' = L_{t+1}$, $\cY=\cW_t$ and $\cY' = \cW_{t+1}$. As the proof shows, the facts that $\cW_t$ is shrinking and $L_t$ is increasing are critical. One motivation for the specific sequence $L_t$ is precisely this increasing property. 
\begin{restatable}[Decreasing Uncertainty Lemma -- Logistic Bandits]{lemma}{decUncertaintiesGOAT}
\label[lemma]{lem:dec-uncertainties-GOAT}
Let $K'\succeq K$ be $d\times d$ positive definite matrices and 
$\cY'\subseteq \cY\subseteq \RR^d$ bounded closed sets.
Then, 
    \begin{align*}
     \max_{a\in \cA, \theta\in \cY'}\cU(a,\theta,K')
        \le 
        \max_{a\in \cA, \theta\in \cY}\cU(a,\theta,K)\,,
    \end{align*}
\end{restatable}

The regret bound follows from this lemma by a summing argument, though the nonlinearity of $\mu$ makes the proof significantly more technical.
\begin{theorem}[\MULog Simple Regret Bound]\label[theorem]{thm:mulog-regret}
    Under Assumptions \ref{ass:bounded_feature} and \ref{ass:bounded_para}, 
    there exists some universal constant $\mathfrak c>0$ such that the following holds:
    Let $\delta\in[0,1)$, $T\ge 1$ be arbitrary.
    Then, with probability at least $1-\delta$, it holds that
    the simple regret of the action $\hat a$ computed by 
    \MULog (\cref{alg:max-uncertainty-logSR-GOAT}) 
    with an appropriate choice of $\lambda$,
    after $T$ rounds 
    is upper bounded by
    \begin{align*}
        \SR^\Log(\hat a)&\le \mathfrak c\, d\sqrt{\frac{\log ({T}/{\delta})}{T\kappa_*}} + \frac{\mathrm{poly}\bigg(d,\kappa,\log(T),\log(1/\delta)\bigg)}{T^{3/4}\sqrt{\kappa_*}}\,.
    \end{align*}
\end{theorem}
The explicit choice of $\lambda$, which
depends only on $\delta,T$, and $S$, is given in \cref{appendix:proof-mulog}. The polynomial term is lower order in $T$; the leading term is the rate-relevant part and matches the lower bound up to logarithmic factors.
This is a first-order minimax guarantee; it does not preclude sharper instance-dependent rates on easier geometries where informative low-reward actions accelerate identification. 
Practitioners may prefer returning $\argmax_{a\in\cA}\phi(a)^\top\hat\theta_{T+1}$ (MLE-greedy), which carries no guarantee but tends to perform well empirically.

        \begin{algorithm}[t]
        \caption{Try Hard Thompson Sampling ($\THATS$)}
        \label{alg:THATS}
        \begin{algorithmic}[1]
            \REQUIRE $\lambda > 0, S>0$
            \STATE $L_1 = \lambda I$, $\cV_0=\mathbb B_d(S)$
            \FOR{$t: 1 \to T$}
            \STATE Sample $\tilde\theta_t\sim \cN(0,L_t^{-1})$
            \STATE Solve for $\bar\theta_t=\argmin_{\theta\in \mathbb B_d(S)}\mcL_t^\lambda(\theta)$
            \STATE Build $\cV_t=\cV_{t-1}\cap \cE_t(\delta,\bar\theta_t)$ by \cref{eq:contrained_rmle}
            \STATE Select $A_t=\argmax_{a\in \cA}{\dot\mu(\phi(a)^\top\bar\theta_t)}^{\frac{1}{2}}|\phi(a)^\top\tilde\theta_t|$ (\cref{eq:THATS-action}) and receive $X_t$
            \STATE Solve $\theta_t'=\argmax_{\theta\in \cV_t}|\phi(A_t)^\top\theta|$ and update $L_{t+1}=L_t+\dot\mu(\phi(A_t)^\top\theta_t')\phi(A_t)\phi(A_t)^\top$ as in \cref{eq:proxy-TS}
            \ENDFOR
            \STATE Solve $\bar\theta_{T+1}=\argmin_{\theta\in \mathbb B_d(S)}\mcL_{T+1}^\lambda(\theta)$ and construct $\cV_{T+1}$ by (\ref{eq:TS-intersection})
            \STATE Pick any ${\theta}_{T+1}^{\Log}\in  \cV_{T+1}$\\
            \STATE Return $\tilde a = \argmax_{a\in \cA}\phi(a)^\top\theta_{T+1}^{\Log}$
        \end{algorithmic}
        \end{algorithm}
\subsection[Randomized Algorithm THaTS]{Randomized Algorithm $\THATS$}\label[subsection]{sec:TS}
\THATS is a randomized, computationally lighter version of the same curvature-aware uncertainty principle used by \MULog.
Its pseudocode is shown in \cref{alg:THATS}.

\MULog searches over actions and confidence-set parameters to find the largest certified logistic uncertainty.
\THATS instead uses a randomized approximation: at round $t$, it samples
\[
    \widetilde\theta_t \sim \mathcal N(0,L_t^{-1}),
\]
where $L_t$ is the analogous increasing curvature-weighted design matrix,
\begin{align}
    L_t &= \lambda I + \sum_{i=1}^{t-1}\dot\mu(\phi(A_i)^\top\theta_i')\phi(A_i)\phi(A_i)^\top \quad \text{where} \quad\theta_i'=\argmin_{\theta\in \cE_i(\delta,\bar\theta_i)} \dot\mu(\phi(A_i)^\top\theta)
    \label{eq:proxy-TS}
\end{align}
so that for any fixed feature vector $u$,
\[
    \mathbb E\!\left[(u^\top \widetilde\theta_t)^2
    \mid \mathcal F_{t-1}\right]
    =
    \|u\|_{L_t^{-1}}^2 .
\]
Thus $|u^\top\widetilde\theta_t|$ is a one-sample Monte Carlo proxy for $\|u\|_{L_t^{-1}}$.
Unlike \MULog, \THATS centers its confidence sets at $\bar \theta_t=\argmin_{\theta\in \mB_d(S)}\mcL_t^\lambda(\theta)$, the MLE over the $S$-ball, and uses the inflated sets
\begin{align}
    \cE_t(\delta,\theta_{\circ})&=\{\theta\in \mB_d(S)\,:\, 
    \mcL_t^\lambda(\theta)-\mcL_t^\lambda(\theta_{\circ})
    \le 2\beta_t(\delta)^2\}.
    \label{eq:contrained_rmle}
\end{align}
These sets are convex and avoid searching over an additional intersection with the $S$-ball during action selection.
\THATS weights the randomized directional score by the estimated local curvature:
\begin{align}
    A_t
    =
    \argmax_{a\in \cA}
    \sqrt{\dot\mu(\phi(a)^\top \bar\theta_t)}
    \left|
        \phi(a)^\top \widetilde\theta_t
    \right|\quad \text{where} \quad \bar\theta_t = \argmin_{\theta\in \mathbb B_d(S)}\mcL_t^\lambda(\theta)\,.
    \label{eq:THATS-action}
\end{align}
We also save computation in \cref{eq:proxy-TS} by optimizing over $\cE_i$ rather than intersections of confidence sets.
At the end, the algorithm picks any vector in
\begin{align}
    \cV_{T+1}=\cap_{i=1}^{T}\cE_i(\delta,\bar\theta_i)\,,\label{eq:TS-intersection}
\end{align}
and returns an action greedy with respect to that vector; the intersection is what makes the decreasing-uncertainty argument go through.

We will find it useful to relate the new confidence sets with the ones used previously (see \cref{eq:ctdef}):
\begin{restatable}{lemma}{EtRelateCt} \label{lem:TS_conf_set}
    Let $\delta\in [0,1)$.
    With probability at least $1-\delta$,
    $\bar\theta_t\in \cC_t(\delta,\hat\theta_t)\cap \mB_d(S)$. 
    Furthermore, with probability at least $1-\delta$, 
    $\cC_t(\delta,\hat\theta_t)\cap \mB_d(S)\subseteq \cE_t(\delta,\bar\theta_t)$.
\end{restatable}
\THATS can be less statistically efficient than \MULog for general action sets, because a single Gaussian sample need not point toward the maximum-uncertainty direction $\phi(A_t^\MU)$. 
The key step is to lower bound its expected uncertainty through the normalized correlation between $M_t=L_t^{1/2}\tilde\theta_t$ and $L_t^{-1/2}\phi(A_t^\MU)$. 
Since $M_t/\norm{M_t}$ is uniform on the sphere, we get:
\begin{restatable}{lemma}{TsToMaxUnc}\label{lem:TS-to-max-uncertainty}
    Let $t\ge 1$ and
\begin{align*}
    A_t^{\MU} = \argmax_{a\in \cA}\max_{\theta\in \cV_t} \dot\mu(\phi(a)^\top\theta)\|\phi(a)\|_{L_t^{-1}}.
\end{align*}
Then, it holds almost surely that,
    \begin{align*}
       \min_{\theta\in \cE_t}
       \;U(A_t^\MU,\theta,L_t)
		\cdot 
		I(A_t^\MU,L_t)
        \le \mE\left[U(A_t,\bar\theta_t,L_t)\Big\vert\cF_{t-1}\right]\,,
    \end{align*}
where for $a\in \cA$ and $V\succ 0$, we define $I(a,V)=
    \int_{\mS^{d-1}}  
    \left|\left\langle x,\frac{V^{-1/2}\phi(a)}{\|V^{-1/2}\phi(a)\|}\right\rangle\right|
    dx$.
\end{restatable}
With standard tools in probability theory, $I(a,V)=\Omega(1/\sqrt{d})$ for all $a\in \cA$ and $V\succ 0$. Using this in the analysis gives the following result:
\begin{theorem}\label[theorem]{thm:THATS-regret}
    Under Assumptions \ref{ass:bounded_feature} and \ref{ass:bounded_para}, 
    there exists some universal constant $\mathfrak C>0$ such that the following holds:
    Let $\delta\in[0,1)$, $T\ge 1$ be arbitrary.
    Then, with probability at least $1-\delta$, it holds that
    the simple regret of the action $\tilde a$ computed by 
    \THATS (\cref{alg:THATS})
    with an appropriate choice of $\lambda$
    after $T$ rounds 
    is upper bounded by
    \begin{align*}
        \SR(\tilde a)&\le \mathfrak C d^{3/2}\sqrt{\frac{\log (T/\delta)}{T\kappa_*}} +
        \frac{\mathrm{poly}\bigg(\kappa,d,\log(T),\log(1/\delta)\bigg)}{T^{3/4}\sqrt{\kappa_*}}\,.
    \end{align*}
\end{theorem}
The full expression for the lower-order terms and the choice of $\lambda$ are given in the appendix. The leading term is worse than the deterministic rate by a factor $\sqrt d$, reflecting the one-sample randomized approximation to the max-uncertainty direction. The gap traces to $I(a,V)=\Omega(1/\sqrt{d})$ being tight, so closing it would require structural assumptions on the action set or a multi-sample scheme that better approximates the deterministic maximizer; the analogous gap for randomized exploration in linear bandits with smooth convex action sets has recently been closed by \citet{abeille2025randomised}.

\section{Numerical Experiment}
\label{sec:experiment}
\Cref{fig:intro-probes} summarizes the empirical message: the shifted saturated family is a worst-case geometry, not the whole problem.
We use a finite-arm version of the shifted saturated hypercube ($d=8$, saturation $m=4$, perturbation $\varepsilon=1$, $100$ runs) and compare two action sets: the hard set $\cA_{\mathrm{hard}}$ containing only the $2^7 = 128$ candidate arms in the saturated region, and the easy set $\cA_{\mathrm{easy}}$ that adds $14$ probe arms whose logits are near the center of the sigmoid and whose rewards are therefore much lower than those of the candidate arms.
When probe arms are present, $\MULog$ and $\THATS$ allocate a large fraction of samples to them and reach the simple-regret threshold $10^{-2}$ in a few hundred rounds in most runs, compared to several thousand rounds on $\cA_{\mathrm{hard}}$; sign-recovery error falls to near zero at a matching pace. Cumulative Thompson sampling and uniform exploration also benefit from the easier geometry, but the improvement is weaker and more variable, because neither algorithm actively seeks out the low-reward probes. The probe-pull fraction diagnostic confirms that the gains of \MULog\ and \THATS\ come from exploiting the informative structure of the action set, not merely from the change in recommendation set.
Full setup details and per-algorithm diagnostic plots are in \cref{appendix:experiment}.

\section{Conclusions}\label{sec:conclusions}

Pure exploration in logistic bandits is a natural model for high-impact decision problems where the goal is not to maximize reward during data collection, but to identify a strong final action. Our results show that this objective has its own geometry: local curvature at the optimal action, $\kappa_*$, controls the first-order minimax rate, and curvature-aware exploration matches that lower bound up to logarithmic factors in the deterministic case. Beyond the minimax construction, the experiments make clear that informative low-reward actions can dramatically improve identification, even when cumulative-regret methods avoid them. We view these results as a starting point for a sharper theory of decision-focused exploration in nonlinear bandits. The probe-arm phenomenon has direct practical implications: settings such as A/B testing and reward-model data collection may include actions that are poor to deploy yet highly informative, and a curvature-aware algorithm that actively seeks them out can substantially reduce the rounds needed to reach a reliable recommendation — a behavior cumulative-regret methods structurally cannot exhibit. Natural next questions include localized complexities that capture when such probes help, closing the $\sqrt d$ gap for randomized exploration, and extending the framework to contextual logistic bandits, where we expect the algorithms and analyses to generalize with appropriate adaptation.


\bibliographystyle{plainnat}
\bibliography{ref}

\clearpage
\appendix
\section{Regret Analysis of \MULin}
        \begin{algorithm}[t]
        \caption{Max-Uncertainty-Lin (\MULin)}
        \label[algorithm]{alg:max-uncertainty-Lin}
        \footnotesize
        \begin{algorithmic}[1]
            \REQUIRE $\lambda > 0$
            \STATE $V_1 = \lambda I$, $\cD = \{\}$ \COMMENT{$\cD$ is a multiset}
            \FOR{$t: 1 \to T$}
                \STATE Select $A_t=\argmax_{a\in \cA} \|\phi(a)\|_{V_{t}^{-1}}$
                \STATE Receive reward $X_t$
            \ENDFOR
            \STATE $\hat{\theta}_T = V_{T+1}^{-1} \sum_{t=1}^T X_t \phi(A_t)$
            \STATE Return $\hat a=\argmax_{a \in \cA} \phi(a)^\top \hat{\theta}_T$
        \end{algorithmic}
        \end{algorithm}

In this section, we will analyze the simple regret of \MULin (\cref{alg:max-uncertainty-Lin}).
\subsection{Confidence Set}
In the analysis, we will use the confidence set from \cite{abbasi2011improved}
\begin{lemma}[Theorem 2 of \cite{abbasi2011improved}]
    Let $\delta\in (0,1)$. Then with probability at least $1-\delta$, it holds that for all $t\ge 1$,\label{lem:linear-conf-width}
    \begin{align}
        \|\hat\theta_t - \theta_*\|_{V_t}\le \sqrt\lambda \|\theta_*\| + \sqrt{2\log(1/\delta) + d\log\left(1+\frac{t}{d\lambda}\right)}:=\tau_t(\delta).\label{eq:tau_t}
    \end{align}
\end{lemma}

\subsection{Decreasing Uncertainty}
In this section we restate our result on the decreasing uncertainty mentioned in \cref{sec:det_alg}, specifically, \cref{lem:dec_unc_lin}. This lemma allows us to relate the analysis techniques in the cumulative regret setting \cite{abbasi2011improved} to the analysis in the simple regret setting.
\begin{lemma}[Decreasing Uncertainty Lemma, Lemma 6 of \cite{zanette2021design}]\label{lem:dec_unc_lin}
    For every $t\ge 1$, it holds that
    \begin{align*}
        \max_{a\in \cA}\|\phi(a)\|_{V_{t+1}^{-1}} \le 
        \max_{a\in \cA}\|\phi(a)\|_{V_{t}^{-1}}\,.
    \end{align*}
\end{lemma}

\subsection{Proof of the regret bound }\label{appendix:MULin-proof}
In this section, we first state the formal theorem statementand then proofs are provided. \label{sec:lin-regret}
\begin{theorem}[\MULin Simple Regret Bound]
    Under Assumptions \ref{ass:bounded_feature} and \ref{ass:bounded_para}, 
    there exists some universal constant $\mathfrak c>0$ such that the following holds:
    Let $\delta\in[0,1)$, $T\ge 1$ be arbitrary.
    Then, with probability at least $1-\delta$, it holds that
    the simple regret of the policy $\hat a$ computed by 
    \MULin (\cref{alg:max-uncertainty-Lin}) 
    with an appropriate choice of $\lambda$, 
    after $T$ rounds 
    is upper bounded by
    \begin{align*}
        \SR(\hat a)&\le \frac{\tau_{T+1}(\delta)}{\sqrt T}\sqrt{d\log((d\lambda_T +T)/(d\lambda_T))}\;,
    \end{align*}
    where $\tau_{T+1}(\delta)= \tilde \cO (\sqrt{d+ \log(1/\delta)}+\|\theta_*\|)$ whose full expression can be found in \cref{eq:tau_t}.
\end{theorem}

A related objective to $\SR$, which is explained in the lemma that follows, when $\hat\theta_{T+1}$ 
is used, is the expected maximum prediction error
\begin{align*}
    D_\Lin(\hat\theta_{T+1})=\max_{a \in \cA} |\phi(a)^\top
        (\theta^* - \hat{\theta}_{T+1})|.
\end{align*}
\begin{lemma}
    For a vector $\bar\theta\in \RR^d$, let $\bar a$ be greedy w.r.t $\bar\theta$, i.e., 
    \begin{align*}
        \bar a = \argmax_{a \in \cA} \phi(a)^\top\bar\theta,
    \end{align*}
    it then follows that 
    \label{lem:SR-bound-Lin}
    \begin{align*}
        \SR(\bar a)
        \leq 2 D_\Lin(\bar\theta).
    \end{align*}
\end{lemma}
\begin{proof}
    The proof is a simple application of triangle inequality.
    We then add and subtract the same term below
    \begin{align*}
        &\; \phi(a_*)^\top\theta^* 
        -
        \phi(\bar a)^\top\theta^*\\
        =&\; \phi(a_*)^\top\theta^* 
        - 
        \phi(a_*)^\top\bar\theta 
        + 
        \phi(a_*)^\top\bar\theta 
        - 
        \phi(\bar a)^\top\theta^*\\
        \le&\; \phi(a_*)^\top\theta^* 
        - 
        \phi(a_*)^\top\bar\theta 
        + 
        \phi(\bar a)^\top\bar\theta 
        - 
        \phi(\bar a)^\top\theta^*\tag{Defn. of $\bar a$}\\
        \le&\; \left|\phi(a_*)^\top\theta^* 
        - \phi(a_*)^\top\bar\theta\right| 
        + 
        \left|\phi(\bar a)^\top\bar\theta 
        - 
        \phi(\bar a)^\top\theta^*\right|\\
        \le& \;2D_\Lin(\bar\theta).
    \end{align*}
\end{proof}
Hence we focus on bounding $D_\Lin(\hat\theta_{T+1})$. By Cauchy-Schwarz inequality, we have\todos{Change $s$ into $S_t$}
\begin{align*}
    D_\Lin(\hat\theta_{T+1})&= \max_{a \in \cA} |\phi(a)^\top
        (\theta^* - \hat{\theta}_{T+1})|\\
    &\le  \max_{a\in \cA}\|\phi(a)\|_{V_{T+1}^{-1}}\|\theta_* - \hat\theta_{T+1}\|_{V_{T+1}}\\
    &\le \tau_{T+1}(\delta)\max_{a\in \cA}\|\phi(a)\|_{V_{T+1}^{-1}},
\end{align*}
where in the last line we used \cref{lem:linear-conf-width}. Then by \cref{lem:dec_unc_lin} we have that for all $1\le t\le T$,
\begin{align*}
   \max_{a\in \cA}\|\phi(a)\|_{V_{T+1}^{-1}}\le \int \max_{a\in \cA}\|\phi(a)\|_{V_{t}^{-1}}.
\end{align*}
Hence,
\begin{align*}
    \max_{a\in \cA}\|\phi(a)\|_{V_{T+1}^{-1}}\le \frac{1}{T}\sum_{t=1}^T \max_{a\in \cA}\|\phi(a)\|_{V_{t}^{-1}}.
\end{align*}
We can then plug this into the bound for $D_\Lin(\hat\theta_{T+1})$ to get
\begin{align}
    \tau_{T+1}(\delta)\max_{a\in \cA}\|\phi(a)\|_{V_{T+1}^{-1}} &\le \tau_{T+1}(\delta)\frac{1}{T}\sum_{t=1}^T  \max_{a\in \cA}\|\phi(a)\|_{V_{t}^{-1}}\label{eq:first-reverse-bern}.
\end{align}
Then by \cref{lem:epl}
\begin{align}
    D_\Lin(\hat\theta_{T+1})\le \;&\frac{\tau_{T+1}(\delta)}{T}\sum_{t=1}^T {\|\phi(A_t)\|_{V_{t}^{-1}}} \\
    \le\; & \frac{\tau_{T+1}(\delta)}{T}\sqrt{T}\sqrt{\sum_{t=1}^T \|\phi(A_t)\|^2_{V_t^{-1}}},\\
    \le\; & \frac{\tau_{T+1}(\delta)}{T}\sqrt{T}\sqrt{d\log((d\lambda_T +T)/(d\lambda_T))},\\
    =\; & \frac{\tau_{T+1}(\delta)}{\sqrt T}\sqrt{d\log((d\lambda_T +T)/(d\lambda_T))},
\end{align}
where in the third line we used Cauchy-Schwarz inequality and the last line follows from elliptical potential lemma (\cref{lem:epl}).
Finally, chaining the above bound with \cref{lem:SR-bound-Lin} gives the desired result.

\section[Analysis of SimpleLinTS]{Analysis of \LinTS (\cref{alg:LinTS})}

        \begin{algorithm}[t]
        \caption{Simple Regret Linear Thompson Sampling ($\LinTS$)}
        \label{alg:LinTS}
        \footnotesize
        \begin{algorithmic}[1]
            \REQUIRE $\lambda > 0$
            \STATE $V_1 = \lambda I$
            \FOR{$t: 1 \to T$}
            \STATE Sample $\tilde{\theta}_t \sim \cN(0, V_{t}^{-1})$
            \STATE Select $A_t=\argmax_{a\in \cA}|\phi(a)^\top\tilde\theta_t|$ and receive $X_t$
            \STATE Update $V_{t+1}=V_t+\phi(A_t)\phi(A_t)^\top$
            \ENDFOR
            \STATE $\hat{\theta}_{T+1} = V_{T+1}^{-1} \sum_{t=1}^T X_t \phi(A_t)$
            \STATE Return $\hat a= \argmax_{a \in \cA} \phi(a)^\top \hat{\theta}_{T+1}$
        \end{algorithmic}
        \end{algorithm}

In this section we will analyze the simple regret of \LinTS (\cref{alg:LinTS}). 
\subsection{Analysis on the exploration done by \LinTS and \MULin}
The analysis of \LinTS is highly related to that of \MULin algorithm. Hence in order to identify the actions that gives the maximum uncertainty, we define it to be
\begin{align*}
    A_t^\MU=\argmax_{a \in \cA} \| \phi(a) \|_{V_t^{-1} }.
\end{align*}
We now present the lemma that allows us to relate the analysis of \MULin to \LinTS.
\begin{lemma}\label[lemma]{lem:TS-to-max-uncertainty-Lin}
Fix $t\ge 1$. Then, almost surely,
    \begin{align*}
       \|\phi(A_t^\MU)\|_{V_t^{-1}}\cdot I(A_t^\MU,V_t)\le \mE\left[ \|\phi(A_t)\|_{V_t^{-1}}\Big\vert\cF_{t-1}\right],
    \end{align*}
    where for $a\in \cA$ and $V\succeq 0$, we let $I(a,V)=
    \int_{\mS^{d-1}}  
    \left|\left\langle x,\frac{V^{-1/2}\phi(a)}{\|V^{-1/2}\phi(a)\|}\right\rangle\right|
    dx$.
\end{lemma}
\begin{proof}
    We start by rewriting the right hand side of the inequality
    \begin{align*}
        &\mE\left[\|\phi(A_t)\|_{V_{t}^{-1}}\Big\vert \cF_{t-1}\right]\\
        \;=&\mE\left[\max_{x\in \mS^{d-1}}\left\langle V_t^{-1/2}x,\phi(A_t)\right\rangle \Big\vert \cF_{t-1}\right]\\
        \ge\;& \mE\left[\left|\left\langle\frac{V_t^{-1/2}\cdot V_t^{1/2}\tilde \theta_t}{\|V_t^{1/2}\tilde\theta_t\|_2},\phi(A_t)\right\rangle\right|\Big \vert \cF_{t-1}\right]\\
        \ge\;& \mE\left[\left|\left\langle\frac{\tilde \theta_t}{\|V_t^{1/2}\tilde\theta_t\|_2},\phi(A_t^\MU)\right\rangle\right|\Big \vert \cF_{t-1}\right],
    \end{align*}
    where in the last line we used the definition of $A_t$.
    Since $\tilde\theta_t\sim \cN(0,V_t^{-1})$, we can rewrite it as $\tilde\theta_t = V_t^{-1/2}M_t$ for $M_t\sim \cN(0,I)$ given the past. Then plug it back in the above expression,
    \begin{align*}
        &\mE\left[\left|\left\langle\frac{\tilde \theta_t}{\|V_t^{1/2}\tilde\theta_t\|_2},\phi(A_t^\MU)\right\rangle\right|\Big \vert \cF_{t-1}\right] \\
        = \;&\mE\left[\left|\left\langle\frac{V_t^{-1/2}M_t}{\|M_t\|_2},\phi(A_t^\MU)\right\rangle\right|\Big \vert \cF_{t-1}\right]\\
        = \;&\mE\left[\left|\left\langle\frac{M_t}{\|M_t\|_2},V_t^{-1/2}\phi(A_t^\MU)\right\rangle\right|\Big \vert \cF_{t-1}\right]\\
        = \;&\|\phi(A_t^\MU)\|_{V_t^{-1}}\mE\left[\left|\left\langle\frac{M_t}{\|M_t\|_2},\frac{V_t^{-1/2}\phi(A_t^\MU)}{\|\phi(A_t^\MU)\|_{V_t^{-1}}}\right\rangle\right|\Big \vert \cF_{t-1}\right]\\
        =\;&\|\phi(A_t^\MU)\|_{V_t^{-1}}I(A_t^\MU,L_t),
    \end{align*}
    where the third line used $V_t^{-1/2}$ is symmetric (as $V_t$ is positive definite); the fourth line follows by $A_t^\MU$ is $\cF_{t-1}$-measurable and \cref{prop:measurable-cond-exp} and the last line follows by the definition of $I(A_t^\MU,L_t)$.
\end{proof}
\begin{corollary}
    For all $t\ge 1$, it holds almost surely that
    \begin{align*}
        \|\phi(A_t^\MU)\|_{V_t^{-1}}\le \sqrt{\frac{\pi d}{2}}\|\phi(A_t)\|_{V_t^{-1}}.
    \end{align*}
\end{corollary}
\begin{proof}
    The proof follows by dividing both sides of the inequality showed in \cref{lem:TS-to-max-uncertainty-Lin} by $I(A_t^\MU, V_t)$ and then showing that $I(A_t^\MU, V_t)$ is lower bounded by $\sqrt{\frac{2}{\pi d}}$. The latter follows using \cref{prop:gaussian-to-uniform} and \cref{prop:uniform-sphere-inner-product}. Together we get,
    \begin{align*}
        \|\phi(A_t^\MU)\|_{V_t^{-1}} \le \sqrt{\frac{\pi d}{2}}\mE\left[\|\phi(A_t)\|_{V_t^{-1}}\Big\vert\cF_{t-1}\right].
    \end{align*}
\end{proof}
\subsection{Analysis of \LinTS}\label{appendix:LinTS-proof}
In this section, we first state formally our regret bound on \LinTS where detailed constant and polynomial dependency is presented. After that we give proof on it.
\begin{theorem}[\LinTS Simple Regret Bound]
    Under Assumptions \ref{ass:bounded_feature} and \ref{ass:bounded_para}, 
    there exists some universal constant $\mathfrak c>0$ such that the following holds:
    Let $\delta\in[0,1)$, $T\ge 1$ be arbitrary.
    Then, with probability at least $1-\delta$, it holds that
    the simple regret of the policy $\hat a$ computed by 
    \LinTS (\cref{alg:LinTS}) 
    with an appropriate choice of $\lambda$,
    after $T$ rounds 
    is upper bounded by
    \begin{align*}
        \SR(\hat\pi)&\le \frac{4\tau_{T+1}(\delta)d\sqrt{\log((d\lambda_T +T)/(d\lambda_T))}}{\sqrt T} 
        + \frac{64\tau_{T+1}(\delta)\log(\log(2T/\delta))}{T} \;,
    \end{align*}
    where $\tau_{T+1}(\delta)= \tilde \cO (\sqrt{d+ \log(1/\delta)}+\|\theta_*\|)$ whose full expression can be found in \cref{eq:tau_t}.
\end{theorem}

All we need to do is to plug \cref{lem:TS-to-max-uncertainty-Lin} into the analysis of \MULin. \cref{lem:SR-bound-Lin} still holds so we start by doing the same analysis as in \cref{sec:lin-regret} and then plug in the results of \cref{lem:TS-to-max-uncertainty-Lin} to get the final bound.
For clarity we copy the analysis in \cref{sec:lin-regret}; for readers who are familiar with the analysis of \MULin, they can skip to the end of this section where we highlight the step that is unique to \LinTS to be red.
By Cauchy-Schwarz inequality, we have
\begin{align*}
    D_\Lin(\hat\theta_{T+1})&= \max_{a \in \cA} |\phi(a)^\top
        (\theta^* - \hat{\theta}_{T+1})|\\
    &\le  \max_{a\in \cA}\|\phi(a)\|_{V_{T+1}^{-1}}\|\theta_* - \hat\theta_{T+1}\|_{V_{T+1}}\\
    &\le \tau_{T+1}(\delta)\max_{a\in \cA}\|\phi(a)\|_{V_{T+1}^{-1}},
\end{align*}
where in the last line we used \cref{lem:linear-conf-width}. Then by \cref{lem:dec_unc_lin} we have that for all $1\le t\le T$,
\begin{align*}
    \max_{a\in \cA}\|\phi(a)\|_{V_{T+1}^{-1}}\le  \max_{a\in \cA}\|\phi(a)\|_{V_{t}^{-1}}.
\end{align*}
Hence,
\begin{align*}
    \max_{a\in \cA}\|\phi(a)\|_{V_{T+1}^{-1}}\le \frac{1}{T}\sum_{t=1}^T \max_{a\in \cA}\|\phi(a)\|_{V_{t}^{-1}}.
\end{align*}
We can then plug this into the bound for $D_\Lin(\hat\theta_{T+1})$ to get
\begin{align*}
    \tau_{T+1}(\delta)\max_{a\in \cA}\|\phi(a)\|_{V_{T+1}^{-1}} \le \tau_{T+1}(\delta)\frac{1}{T}\sum_{t=1}^T \max_{a\in \cA}\|\phi(a)\|_{V_{t}^{-1}}
\end{align*}
From \cref{cor:Reversed-Bernstein-inequality-for-martingales},
\begin{align*}
    \frac{1}{T}\sum_{t=1}^T \max_{a\in \cA}\|\phi(a)\|_{V_{t}^{-1}} =&\; \frac{1}{T}\sum_{t=1}^T {\|\phi(A_t)\|_{V_{t}^{-1}}}\\
    \le \;&\textcolor{red}{\tau_{T+1}(\delta)\frac{1}{T}\sqrt{d}\sum_{t=1}^T \mE\left[\|\phi(A_t)\|_{V_{t}^{-1}}\Big\vert \cF_{t-1}\right]} \\
    \le \;&\frac{2\tau_{T+1}(\delta)\sqrt d}{T}\left(16\log(\log(2T/\delta))+\sum_{t=1}^T {\|\phi(A_t)\|_{V_{t}^{-1}}} \right)\\
    \le\; & \frac{2\tau_{T+1}(\delta)\sqrt d}{T}\left(16\log(\log(2T/\delta))+\sqrt{T}\sqrt{\sum_{t=1}^T {\|\phi(A_t)\|^2_{V_{t}^{-1}}}} \right)\\
    \le\; & \frac{2\tau_{T+1}(\delta)\sqrt d}{T}\left(16\log(\log(2T/\delta))+\sqrt{T}\sqrt{d\log((d\lambda_T +T)/(d\lambda_T))} \right).
\end{align*}
Finally, chaining the above bound with \cref{lem:SR-bound-Lin} gives the desired result.

\section[Regret Analysis of MULog]{Regret Analysis of \MULog (\cref{thm:mulog-regret})}
The sigmoid function is known to be (generalized) self-concordant \citep{Bach2010Self,faury2020improved,liu2024almost}, to be more specific, 
\begin{align}
    |\ddot\mu(z)|\le \dot\mu(z) \quad \text{for all }z \in \RR.\label{eq:self-concordant}
\end{align}
The logistic function is also $1/4$-Lipschitz, i.e.,
\begin{align}
    \dot\mu\le 1/4\label{eq:logistic-lipschitz}
\end{align}
which follows from the decomposition of $\dot\mu=\mu(1-\mu)$ and $\mu\in [0,1]^{\RR}$.

We also consider the Hessian of the loss, which takes the form of 
\begin{align*}
    \nabla^2\mcL_t^\lambda(\theta)=:H_t(\theta)=\sum_{s=1}^{t-1}\dot\mu(\phi(A_s)^\top\theta)\phi(A_s)\phi(A_s)^\top + \lambda I
\end{align*}
In the analysis, we also consider the following matrices that are closely related to $H_{t}(\theta;\{A_s\}_{s=1}^{t-1})$.
\begin{align*}
    G_t(\theta_1,\theta_2)&=\lambda I+\sum_{s=1}^{t-1} \alpha(\phi(A_s),\theta_1,\theta_2)\phi(A_s)\phi(A_s)^\top\\
    \tilde G_t(\theta_1,\theta_2)&=\lambda I+\sum_{s=1}^{t-1}\tilde \alpha(\phi(A_s),\theta_1,\theta_2)\phi(A_s)\phi(A_s)^\top,
\end{align*}
where for $x,\theta_1,\theta_2\in \RR^d$,
\begin{align*}
    \alpha(x,\theta_1,\theta_2)&=\int_{v=0}^1 \dot\mu(x^\top\theta_1+vx^\top(\theta_2-\theta_1))dv\\
    \tilde \alpha(x,\theta_1,\theta_2)&=\int_{v=0}^1 (1-v)\dot\mu(x^\top\theta_1+vx^\top(\theta_2-\theta_1))dv.
\end{align*}
\subsection{Self-concordance control}
Self-concordance is a property of the logistic function $\mu(\cdot)$ that allows us to utilize the curvature information.
Here are the lemmas borrowed from \cite{abeille2021instance,faury2020improved} that we will use in the analysis.
\begin{lemma}[Lemma 9 of \cite{abeille2021instance}]
    \label{lem:mu_dot_relate}
    For all $z_1, z_2\in \RR$, it follows that
    \begin{align*}
        \dot\mu(z_2)\exp(-|z_2-z_1|)\le \dot\mu(z_1)\le \dot\mu(z_2)\exp(|z_2-z_1|).
    \end{align*}
\end{lemma}
\begin{lemma}[First order self-concordance control, Lemma 9 of \cite{faury2020improved}]
    For all $z_1, z_2\in \RR$, it follows that
    \begin{align*}
        \dot\mu(z_1)\frac{1}{1+|z_1-z_2|}\le \int_{v=0}^1 \dot\mu(z_1+v(z_2-z_1))dv\le \dot\mu(z_1)\frac{\exp(|z_1-z_2|-1)}{|z_1-z_2|}.
    \end{align*}
\end{lemma}
\begin{lemma}[Second order self-concordance control, Lemma 8 of \cite{abeille2021instance}]
    For all $z_1,z_2\in \RR$,
    \begin{align*}
        \int_{v=0}^1 (1-v)\dot \mu(z_1+v(z_2-z_1))dv\ge \frac{\dot\mu(z_1)}{2+|z_1-z_2|}.
    \end{align*}
\end{lemma}
\begin{lemma}[Eqs.(7,8) of \cite{abeille2021instance}]\label{lem:selfconc}
    Let $\theta_1,\theta_2 \in \RR^d$. For $t\ge 1$
    \begin{align*}
        u_t=
        \begin{cases}
            0,& \text{ if } t = 1\\
            \max_{x\in \{\phi(A_s)\}_{s=1}^{t-1}}|x^\top(\theta_1-\theta_2)|, & \text{ if } t\ge 2
        \end{cases}
    \end{align*}
    Then it follows that 
    \begin{align*}
        G_t(\theta_1,\theta_2) &\succeq (1+2u_t)^{-1}H_t(\theta)\text{ for }\theta\in \{\theta_1,\theta_2\}\\
        \tilde G_t(\theta_1,\theta_2) &\succeq (2+2u_t)^{-1}H_t(\theta_1)
    \end{align*}
\end{lemma}
\subsection{Results on confidence set}\label{appendix:confset}
In this section, we state the lemmas on the confidence set that we will use in the analysis.
The following confidence set from \cite{faury2020improved} is also used in our analysis:
\begin{lemma}[Lemma 1 of \cite{faury2020improved}]\label{lem:faury_conf}
    Let $\delta\in(0,1]$. Under assumptions \ref{ass:bounded_feature}, \ref{ass:bounded_para}, with probability at least $1-\delta$,
    \begin{align*}
        \forall t \ge 1, \quad \|g_t(\hat\theta_t) - g_t(\theta_*)\|_{H_t^{-1}(\theta_*)}\le \rho_t(\delta),
    \end{align*}
    where 
    \begin{align}
        \lambda_T &= 1\vee \frac{2d}{S}\log\left(e\sqrt{1+\frac{T}{4d}}\vee 1/\delta\right)\label{eq:lambda},\\
        \rho_t(\delta)&=\left(\frac{1}{2}+S\right)\sqrt{\lambda_T} + \frac{4d}{\sqrt{\lambda_T}}\log\left(e\sqrt{1+\frac{t}{4d}}\vee 1/\delta\right).\label{eq:rho}
    \end{align}
\end{lemma}
\begin{lemma}[Lemma 1 of \cite{abeille2021instance}]
\label[lemma]{lem:confidence_set}
    Let $\delta\in[0,1)$. It follows that
    $
        \PP\big(\forall t\ge 1,\theta_*\in \cC_t(\delta,\hat\theta_t)\big)\ge 1-\delta.
    $
\end{lemma}
Recall that
\begin{align*}
    \cC_t(\theta_{\circ},\delta) = \left\{\theta:\mcL_t^\lambda(\theta) - \mcL_t^\lambda(\theta_{\circ})\le \beta_t(\delta)\right\},
\end{align*}
where according to \citet{abeille2021instance}, $\beta_t(\delta)$ is set to be 
\begin{align}
    \beta_t(\delta)=\rho_t(\delta)+\frac{\rho_t(\delta)^2}{\sqrt{\lambda_T}}.\label{eq:beta}
\end{align}
Combining the above two lemmas (\cref{lem:confidence_set,lem:faury_conf}) together, we have the following lemma.

\begin{lemma}\label{lem:conf-width-GOAT}
    Under \cref{ass:bounded_feature},\ref{ass:bounded_para},
    for all $\delta\in [0,1)$, with probability at least $1-2\delta$, 
    for all $t \ge 1$ and $\theta\in \conf{t}\cap \mB_d(S)$, we have that 
    \begin{align*}
        \|\theta-\theta_*\|_{H_t(\theta_*)}\le (4+4S)\rho_t(\delta)+\sqrt{(8S+8)}\beta_t(\delta)=:\gamma_t(\delta),
    \end{align*}
    where $\beta_t(\delta)$ is defined in \cref{eq:beta}.
\end{lemma}
\begin{proof}
    We start from Taylor expansion.
    For all $\theta'\in \RR^d$, we have that 
    \begin{align*}
        \mcL_t^\lambda(\theta')=\mcL_t^\lambda(\theta_*) + \nabla\mcL_t^\lambda(\theta_*)^\top(\theta'-\theta_*)+\frac{1}{2}\|\theta'-\theta_*\|^2_{\tilde G_t(\theta_*,\theta')}.
    \end{align*}
    Let $\theta\in\conf{t}\cap\mB_d(S)$. Rearrange the terms, apply absolute value and plug in $\theta$,
    \begin{align*}
        |\mcL_t^\lambda(\theta) - \mcL_t^\lambda(\theta_*) - \nabla \mcL_t^\lambda (\theta_*)(\theta-\theta_*)|=\frac{1}{2}\|\theta-\theta_*\|^2_{\tilde G_t(\theta_*,\theta)}
        \ge \frac{1}{2(2+2S)}\|\theta-\theta_*\|^2_{H_t(\theta_*)},
    \end{align*}
    where the last inequality follows from \cref{lem:selfconc} and $\theta,\theta_*\in \mB_d(S)$.
    It remains to upper bound the left most side of the above equation.
    By triangle inequality we can split it into two terms and we bound them separately.
    \begin{align*}
        |\mcL_t^\lambda(\theta) - \mcL_t^\lambda(\theta_*) - \nabla \mcL_t^\lambda (\theta_*)(\theta-\theta_*)| \le \underbrace{|\mcL_t^\lambda(\theta) - \mcL_t^\lambda(\theta_*) |}_{(a)}+\underbrace{|\nabla \mcL_t^\lambda (\theta_*)(\theta-\theta_*)|}_{(b)}.
    \end{align*}
    For $(a)$, with probability at least $1-\delta$ we have that $\theta_*\in \cC_t(\delta)\cap \mB_d(S)$, then
    \begin{align*}
        (a)&=|\mcL_t^\lambda(\theta) - \mcL_t^\lambda(\hat\theta_t)+\mcL_t^\lambda(\hat\theta_t)-\mcL_t^\lambda(\theta_*)|\\
        &\le |\mcL_t^\lambda(\theta) - \mcL_t^\lambda(\hat\theta_t)|+|\mcL_t^\lambda(\hat\theta_t)-\mcL_t^\lambda(\theta_*)|\\
        &= \mcL_t^\lambda(\theta) - \mcL_t^\lambda(\hat\theta_t) + \mcL_t^\lambda(\theta_*) - \mcL_t^\lambda(\hat\theta_t)\\
        &\le 2\beta_t(\delta)^2,
    \end{align*}
    where in the third line we used the fact that $\mcL_t^\lambda(\theta)\ge \mcL_t^\lambda(\hat\theta_t)$ for all $\theta\in \RR^d$; in the last line we used \cref{lem:confidence_set} and that $\theta\in \cC_t(\delta)$.
    For $(b)$, note that by definition of $\hat\theta_t$, $\nabla \mcL_t^\lambda(\theta_*)=g_t(\theta_*)-g_t(\hat\theta_t) $. To be more specific, for all $\theta\in \RR^d$,
    \begin{align}
        \nabla_\theta\mcL(\theta)=g_t(\theta)-\underbrace{\sum_{i=1}^{t-1} \phi(A_i)X_i}_{=g_t(\hat\theta_t)\text{ by \cref{eq:rmle_grad}}}.
    \end{align}
    Then by Cauchy-Schwarz, by \cref{lem:faury_conf}, with probability $1-\delta$, we have that $\|g_t(\hat\theta_t) -g_t(\theta_*)\|_{H_t^{-1}(\theta_*)}\le \rho_t(\delta)$, then
    \begin{align*}
        (b)&\le \|g_t(\hat\theta_t) - g_t(\theta_*)\|_{H_t^{-1}(\theta_*)}\|\theta-\theta_*\|_{H_t(\theta_*)}\\
        &\le \rho_t(\delta)\|\theta-\theta_*\|_{H_t(\theta_*)},
    \end{align*}
    where in the last inequality we used \cref{lem:faury_conf}.
    Chaining all the inequality together and use the fact that $\PP(A\cap B)\ge 1-\PP(A^c)-\PP(B^c)$, we have that with probability at least $1-2\delta$,
    \begin{align*}
        \frac{1}{2(2+2 S)}\|\theta-\theta_*\|_{H_t(\theta_*)}^2\le \rho_t(\delta)\|\theta-\theta_*\|_{H_t(\theta_*)} + 2\beta_t(\delta)^2.
    \end{align*}
    Solving the above inequality gives us
    \begin{align*}
        \|\theta-\theta_*\|_{H_t(\theta_*)}\le (4+4S)\rho_t(\delta)+\sqrt{(8S+8)}\beta_t(\delta).
    \end{align*}
\end{proof}

\subsection[Proof of the regret bound of MULog]{Proof of the regret bound of \MULog (\cref{thm:mulog-regret})}\label{appendix:proof-mulog}
In this section, we first state the formal version of \cref{thm:mulog-regret} where all the constants and dependencies are detailed then proofs are provided.
\begin{theorem}[Formal statement of \cref{thm:mulog-regret}]
    Let $\delta\in[0,1)$. 
    Under Assumptions \ref{ass:bounded_feature} and \ref{ass:bounded_para}, 
    there exists some universal constant $\mathfrak c>0$ such that the following holds:
    Let $\delta\in[0,1)$, $T\ge 1$ be arbitrary.
    Then, with probability at least $1-\delta$, it holds that 
    the simple regret of the action $\hat a$ output by 
    \MULog (\cref{alg:max-uncertainty-logSR-GOAT}) 
    with $\lambda$ chosen to be $\lambda_T$ in \cref{eq:lambda},
    after $T$ rounds 
    is upper bounded by
    \begin{align*}
        \SR^\Log(\hat a,\theta_*) &\le \frac{8\gamma_{T+1}(\delta)\sqrt{dY_T(\delta)}}{\sqrt{T\kappa^*}}\\
    &+\frac{8\kappa^{3/2}\gamma_{T+1}^{3/2}(\delta)(dY_T(\delta))^{3/4}}{T^{3/4}\sqrt{\kappa^*}}\\
    &+ \frac{\kappa^3\gamma_{T+1}(\delta)^2dY_T(\delta)}{8T} + \frac{4\gamma_{T+1}^2(\delta){dY_T(\delta)}}{T} \\
    &+\frac{4\kappa^{3}\gamma_{T+1}^{3}(\delta)(dY_T(\delta))^{3/2}}{T^{3/2}},
    \end{align*}
    where $\gamma_T(\delta)=\tilde \cO(\sqrt d)$ is as defined in \cref{lem:conf-width-GOAT},
    $Y_T(\delta)=\log((d\lambda_T +T)/(d\lambda_T))$ and $N_T(\delta)=\log(\log(2T/\delta))$.
\end{theorem}
\paragraph{Step 1: Regret decomposition}
When $\theta_{T+1}^\Log$ 
is used, a useful related quantity is the prediction error, which is explained in the lemma that follows,is the prediction error. For $\theta\in \RR^d$, and action $b\in \cA$, define

\begin{align}
    \label{eq:expected-maximum-prediciton-error-log}
    D_{\Log}(b,\theta_{T+1}^{\Log}) := \left|
        \mu\left(\phi(b)^\top \theta_*\right) - \mu\left(\phi(b)^\top \theta_{T+1}^{\Log}\right)
        \right|
\end{align}
Similar to the linear case, we do regret decomposition here. In order to introduce the instance specific quantity $\kappa_*$, we need to be more careful here:
\begin{lemma}
    \label{lem:SR-bound}
    For a vector $\hat\theta\in \RR^d$, let $\hat a$ be greedy w.r.t $\hat\theta$,
    it then follows that 
    \begin{align*}
       \SR_{\Log}(\hat a,\theta_*)\le D_{\Log}(a_*,\hat \theta) + D_{\Log}(\hat a,\hat \theta)
    \end{align*}
\end{lemma}
\begin{proof}
    The proof is a simple application of triangle inequality. 
    We add and subtract the same term below. 
    \begin{align*}
        &\; \mu\left(\phi(a_*)^\top\theta_* \right)
        -
        \mu\left(\phi(\hat a)^\top\theta_*\right)\\
        =&\; \mu\left(\phi(a_*)^\top\theta_* \right)
        - 
        \mu\left(\phi(a_*)^\top \hat\theta\right) 
        + 
        \mu\left(\phi(a_*)^\top\hat\theta\right) 
        - 
        \mu\left(\phi(\hat a)^\top\theta_*\right)\\
        \le&\; \mu\left(\phi(a_*)^\top\theta_* \right)
        - 
        \mu\left(\phi(a_*)^\top\hat\theta\right)
        + 
       \mu\left( \phi(\hat a)^\top\hat\theta\right)
        - 
        \mu\left(\phi(\hat a)^\top\theta_*\right)\\
        \le&\; \left|\mu\left(\phi(a_*)^\top\theta_* \right)
        - 
        \mu\left(\phi(a_*)^\top\hat\theta \right)\right| 
        + 
        \left|\mu\left( \phi(\hat a)^\top\hat\theta\right)
        - 
        \mu\left(\phi(\hat a)^\top\theta_*\right)\right|\\
    \end{align*}
    where in the third line we used the fact that $\mu$ is an increasing function and the definition of $\hat a$.
\end{proof}

\begin{align}
    D_{\Log}(a_*,\theta_{T+1}^\Log)&=|\mu(\phi(a_*)^\top \theta_{T+1}^{\Log})-\mu(\phi(a_*)^\top \theta_*)|\notag\\
    &\le \underbrace{\dot\mu(\phi(a_*)^\top \theta_*)|\phi(a_*)^\top (\theta_*-\theta_{T+1}^{\Log})|}_{R_1(a_*)} + \underbrace{\ddot\mu(\xi_{a_*})|\phi(a_*)^\top (\theta_*-\theta_{T+1}^{\Log})|^2}_{R_2(a_*)},\label{eq:regret_decomp_a_star}
\end{align}
where $\xi_{a_*}$ in the last line is some point in between $\phi(a_*)^\top \theta_*$ and $\phi(a_*)^\top \theta_{T+1}^\Log$.
Similarly for $D_{\Log}(\hat a,\theta_{T+1}^\Log)$, we have that 
\begin{align}
    D_{\Log}(\hat a,\theta_{T+1}^\Log)\le \underbrace{\dot\mu(\phi(\hat a)^\top \theta_*)|\phi(\hat a)^\top (\theta_*-\theta_{T+1}^{\Log})|}_{R_1(\hat a)} + \underbrace{\ddot\mu(\xi_{\hat a})|\phi(\hat a)^\top (\theta_*-\theta_{T+1}^{\Log})|^2}_{R_2(\hat a)},\label{eq:regret_decomp_a_hat}
\end{align}
\paragraph{Step 2: Bounding $R_1(a^*), R_1(\hat a)$}

We are going to use the following lemma to establish decreasing uncertainty. It is written in a compact form that is reusable in other contexts. When applying it to our setting, we set $\cY=\cW_{t}$, $\cY'=\cW_{T+1}$ and $K'=L_{T+1}$, $K=L_t$ for $t\le T$. The bounded closed set is for the purpose of ensuring the maximum is attained.

\decUncertaintiesGOAT*
\begin{proof}
Fix $a\in \cA$, $\theta\in \RR^d$.
Since $K'\succeq K$ and $\dot\mu$ is positive valued,
    \begin{align*}
    U(a,\theta,K')=
        \dot\mu(\phi(a)^\top \theta)        
        \|\phi(a)\|_{(K')^{-1}}\le
         \dot\mu(\phi(a)^\top \theta)
         \|\phi(a)\|_{K^{-1}}
         = U(a,\theta,K)\,.
    \end{align*}
    Since $\cY'\subseteq\cY$, by the definition of $U(a,\theta,K)$, we have
    \begin{align*}
        \max_{a\in \cA,\theta\in \cY'}U(a,\theta,K')
        \le 
        \max_{a\in \cA,\theta\in \cY} U(a,\theta,K)\;.
    \end{align*}
\end{proof}
We start by bounding $R_1(a^*)$. By Cauchy-Schwarz, using \cref{lem:conf-width-GOAT} to obtain a bound on $\|\theta_*-\theta_{T+1}^\Log\|_{H_{T+1}(\theta_*)}$, with probability at least $1-2\delta$,
\begin{align}
    R_1(a_*)&\le \dot\mu\left(\phi(a_*)^\top\theta_*\right)\|\phi(a)\|_{H_{T+1}^{-1}(\theta_*)}\|\theta_*-\theta_{T+1}^\Log\|_{H_{T+1}(\theta_*)}\notag\\
    &\le \gamma_{T+1}(\delta)\sqrt{\dot\mu\left(\phi(a_*)^\top\theta_*\right)}\sqrt{\dot\mu\left(\phi(a_*)^\top\theta_*\right)}\|\phi(a)\|_{L_{T+1}^{-1}}\notag\\
    &\le \gamma_{T+1}(\delta)\underbrace{\sqrt{\dot\mu\left(\phi(a_*)^\top\theta_*\right)}}_{\sqrt{1/\kappa^*}}\max_{a\in \cA, \theta\in \cW_{T+1}}\sqrt{\dot\mu\left(\phi(a)^\top\theta\right)}\|\phi(a)\|_{L_{T+1}^{-1}}\label{eq:R_1_max_uncer}
\end{align}
where in the second line we use the fact that $H_{T+1}(\theta_*)\succeq L_{T+1}$ and the last line follows from the definition of $\cW_{T+1}$.
We have that $\cW_{t+1}\subseteq \cW_t$ and $L_{t+1}\succeq L_t$ by definition. By letting $\cY=\cW_t$ and $\cY'=\cW_{t+1}$; $K=L_t$ and $K'=L_{t+1}$ in \cref{lem:dec-uncertainties-GOAT}, it then follows that for all $1\le t \le T$,
\begin{align*}
    \max_{a,\theta\in \cW_{T+1}}U(a,\theta,L_{T+1})\le \max_{a,\theta\in \cW_{t}}U(a,\theta,L_t).
\end{align*}
Hence,
\begin{align}
    \max_{a,\theta\in \cW_{T+1}}U(a,\theta,L_{T+1})
    &\le \frac{1}{T}\sum_{t=1}^T \max_{a,\theta\in \cW_{t}}U(a,\theta,L_t)\label{eq:bound_last_by_avg_1}\\
    &=\frac{1}{T}\sum_{t=1}^T U(A_t,\theta_t,L_t).\label{eq:bound_last_by_avg_2}
\end{align}
In order to reduce clutter, we let $\phi_t:=\phi(A_t)$.
It remains to bound $\sum_{t=1}^T \sqrt{\dot\mu(\phi_t^\top\theta_t)}\|\phi_t\|_{L_t^{-1}}$.
Applying Taylor expansion at $\theta_t'$, for $t\ge 1$ and $\zeta_t$ between $\phi_t^\top \theta_t$ and $\phi_t^\top\theta_*$, we have that
\begin{align}
    \sum_{t=1}^T \sqrt{\dot\mu(\phi_t^\top\theta_t)}\|\phi_t\|_{L_t^{-1}}&=\sum_{t=1}^T \sqrt{\dot\mu(\phi_t^\top\theta_t')+\ddot\mu(\zeta_t)\phi_t^\top(\theta_t-\theta_t')}\|\phi_t\|_{L_t^{-1}}\notag\\
    &\le \sum_{t=1}^T \sqrt{\dot\mu(\phi_t^\top\theta_t')}\|\phi_t\|_{L_t^{-1}}+\sum_{t=1}^T \sqrt{\ddot\mu(\zeta_t)|\phi_t^\top(\theta_t-\theta_t')|}\|\phi_t\|_{L_t^{-1}}\notag\\
    &\le \sum_{t=1}^T \|\tilde \phi_t\|_{\tilde V_t^{-1}} + \frac{\sqrt{2}}{2}\sum_{t=1}^T \sqrt{\|\phi_t\|_{H_t^{-1}(\theta_*)}\|\theta_t-\theta_t'\|_{H_t(\theta_*)}}\|\phi_t\|_{L_t^{-1}}\notag\\
    &\le \sum_{t=1}^T \|\tilde \phi_t\|_{\tilde V_t^{-1}} + \frac{\sqrt{2}}{2}\sqrt{\gamma_{T+1}(\delta)}\kappa^{3/2}\sum_{t=1}^T \|\phi_t\|^{3/2}_{V_t^{-1}}\notag\\
    &\le \sqrt{T}\sqrt{\sum_{t=1}^T \|\tilde\phi_t\|^2_{\tilde V_t^{-1}}} + \frac{\sqrt{2}}{2}\sqrt{\gamma_{T+1}(\delta)}\kappa^{3/2}T^{1/4}\left(\sum_{t=1}^T \|\phi_t\|^2_{V_t^{-1}}\right)^{3/4}\notag\\
    &\le \sqrt{T}\sqrt{dY_T(\delta)} + T^{1/4}\sqrt{\gamma_{T+1}(\delta)}\kappa^{3/2}\left(dY_T(\delta)\right)^{3/4},\label{eq:Taylor_EPL}
\end{align}
where in the third line we bounded $\dot\mu(\cdot)$ by $1/4$ (\cref{eq:logistic-lipschitz}) and defined $\tilde\phi_t:=\sqrt{\dot\mu(\phi_t^\top\theta_t')}\phi_t$ and $\tilde V_t:=\lambda I + \sum_{s=1}^{t-1}\tilde \phi_s\tilde\phi_s^\top=L_t$; 
in the fourth line we 
\begin{enumerate}
    \item used the fact that $H_t(\theta_*),L_t\succeq \frac{1}{\kappa} V_t$
    \item we applied \cref{lem:conf-width-GOAT} twice to bound $\|\theta_t'-\theta_t\|_{H_t(\theta_*)}$;
\end{enumerate}
in the last line define $Y_T(\delta):=\log((d\lambda_T +T)/(d\lambda_T))$ and we applied elliptical potential lemma (\cref{lem:epl}) twice as well as Hölder's inequality.
Putting everything together, we have that
\begin{align*}
    R_1(a_*)&\le \frac{\gamma_{T+1}(\delta)\sqrt{dY_T(\delta)}}{\sqrt{T\kappa^*}} + \frac{\kappa^{3/2}\gamma_{T+1}^{3/2}(\delta)(dY_T(\delta))^{3/4}}{T^{3/4}\sqrt{\kappa^*}}
\end{align*}
Now we move to $R_1(\hat a)$.
Recall that 
\begin{align*}
    R_1(\hat a)&=\dot\mu(\phi(\hat a)^\top\theta_*)\left|\phi(\hat a)^\top (\theta_*-\theta_{T+1}^\Log)\right|\\
    &\le \dot\mu(\phi(\hat a)^\top\theta_*) \|\phi(\hat a)\|_{H_{T+1}^{-1}(\theta_*)}\|\theta_*-\theta_{T+1}^\Log\|_{H_{T+1}(\theta_*)}\\
    &\le \sqrt{\dot\mu(\phi(\hat a)^\top\theta_*)}\gamma_{T+1}(\delta)\max_{a\in \cA, \theta\in \cW_{T+1}}\sqrt{\dot\mu(\phi(a)^\top\theta)}\|\phi(a)\|_{L_{T+1}^{-1}}.
\end{align*}
Following the same steps as in $R_1(a_*)$ we have that
\begin{align*}
    R_1(\hat a)\le  \sqrt{\dot\mu(\phi(\hat a)^\top\theta_*)}\left[\frac{\gamma_{T+1}(\delta)\sqrt{dY_T(\delta)}}{\sqrt{T}} + \frac{\kappa^{3/2}\gamma_{T+1}^{3/2}(\delta)(dY_T(\delta))^{3/4}}{T^{3/4}}\right]
\end{align*}
It remains to bound $\sqrt{\dot\mu(\phi(\hat a)^\top\theta_*)}$.
\begin{align*}
    \dot\mu(\phi(\hat a)^\top\theta_*)&\le \dot\mu(a_*^\top\theta_*) + \int_{\phi(\hat a)^\top\theta_*}^{\phi(a_*)^\top\theta_*}|\ddot\mu(u)|du\\
    &\le \dot\mu(\phi(a_*)^\top\theta_*) + \int_{\phi(\hat a)^\top\theta_*}^{\phi(a_*)^\top\theta_*}\dot\mu(u)du\\
    &\le \dot\mu(\phi(a_*)^\top\theta_*) + \mu(\phi(a_*)^\top\theta_*) - \mu(\phi(\hat a)^\top\theta_*)\\
    &=\dot\mu(\phi(a_*)^\top\theta_*) + \mathfrak{R}^{\Log}(\hat \pi)
\end{align*}
Hence we have that 
\begin{align*}
    \sqrt{\dot\mu(\phi(\hat a)^\top\theta_*)}\le \sqrt{\dot\mu(\phi(a_*)^\top\theta_*)}+ \sqrt{\mathfrak{R}^\Log(\hat\pi)}
\end{align*}
as well as 
\begin{align}
    R_1(\hat a) &\le \frac{\gamma_{T+1}(\delta)\sqrt{dY_T(\delta)}}{\sqrt{T\kappa^*}} + \frac{\kappa^{3/2}\gamma_{T+1}^{3/2}(\delta)(dY_T(\delta))^{3/4}}{T^{3/4}\sqrt{\kappa^*}}\\
    &+ \sqrt{\mathfrak{R}^\Log(\hat \pi)}\left[\frac{\gamma_{T+1}(\delta)\sqrt{dY_T(\delta)}}{\sqrt{T}} + \frac{\kappa^{3/2}\gamma_{T+1}^{3/2}(\delta)(dY_T(\delta))^{3/4}}{T^{3/4}}\right]\label{eq:get_the_regret_equation}
\end{align}
\paragraph{Step 3: Bounding $R_2(a_*), R_2(\hat a)$}
Now we move to $R_2(a_*)$. 
We need a lemma that's similar to \cref{lem:dec-uncertainties-GOAT} that we will use to bound $R_2(a_*)$.
\begin{lemma}[Second order decreasing uncertainty -- Logistic Bandits]\label{lem:dec-uncertainties-order-2GOAT}
    Let $K' \succeq K$ be $d\times d$ positive definite matrices and $\cY'\subseteq \cY\subseteq \RR^d$ be bounded closed sets. Then,
    \begin{align*}
        \max_{a\in \cA,\theta\in \cY'}U(a,\theta,K')^2\le \max_{a\in \cA,\theta\in \cY}U(a,\theta,K)^2.
    \end{align*}
\end{lemma}
\begin{proof}
    Note that for all $\cY\subseteq \RR^d$ and positive definite $K\in \RR^{d\times d}$, we have that
    \begin{align*}
        \argmax_{a\in \cA,\theta\in \cY} U(a,\theta,K)^2 &= \argmax_{a\in \cA,\theta\in \cY} U(a,\theta,K).
    \end{align*}
    Everything then follows from the proof of \cref{lem:dec-uncertainties-GOAT}.
\end{proof}
We can now bound $R_2(a_*)$.
By Cauchy-Schwarz, we have that
\begin{align*}
    R_2(a_*)&\le \frac{1}{4}\left|\phi(a_*)^\top (\theta_*-\theta_{T+1}^{\Log})\right|^2\\
    &\le \frac{1}{4}\max_a\|\phi(a)\|^2_{H_{T+1}^{-1}(\theta_*)} \|\theta_*-\theta_{T+1}^{\Log}\|_{H_{T+1}(\theta_*)}^2\\
    &\le \frac{1}{4}\kappa^2\gamma_{T+1}(\delta)^2 \max_{a,\theta\in \cW_{T+1}}\dot\mu(\phi(a)^\top\theta)^2\|\phi(a)\|^2_{H_{T+1}^{-1}(\theta_*)}\\
    &\le \frac{1}{4}\kappa^2\gamma_{T+1}(\delta)^2 \max_{a,\theta\in \cW_{T+1}}\dot\mu(\phi(a)^\top\theta)^2\|\phi(a)\|^2_{L_{T+1}^{-1}}\\
    &\le \frac{1}{4T}\kappa^2\gamma_{T+1}(\delta)^2\sum_{t=1}^T \dot\mu( \phi(A_t)^\top\theta_t)^2\| \phi(A_t)\|^2_{L_t^{-1}}\\
    &\le \frac{1}{64T}\kappa^3\gamma_{T+1}(\delta)^2\sum_{t=1}^T  \| \phi(A_t)\|^2_{V_t^{-1}}\\
    &\le \frac{1}{64T}\kappa^3\gamma_{T+1}(\delta)^2dY_T(\delta),
\end{align*}
where in the second line we used Cauchy-Schwarz; in the third line we applied definition of $\cW_{T+1}$ and the fact that $\theta_{T+1}\in \cW_{T+1}$ by construction; in the fourth line $H_{T+1}(\theta_*)\succeq L_{T+1}$; in the fifth line we used \cref{lem:dec-uncertainties-order-2GOAT}; in the last line we used \cref{lem:epl}.

Similarly we can bound $R_2(\hat a)$ with 
\begin{align*}
    R_2(\hat a)&\le \frac{1}{4}\|\phi(\hat a)\|^2_{H_{t+1}^{-1}(\theta_*)}\|\theta_* - \theta_{T+1}^\Log\|^2_{H_{t+1}(\theta_*)}\\
    &\le \frac{1}{4}\max_{a\in \cA}\|\phi(a)\|^2_{H_{t+1}^{-1}(\theta_*)}\|\theta_* - \theta_{T+1}^\Log\|^2_{H_{t+1}(\theta_*)}\\
    &\le \frac{1}{64T}\kappa^3\gamma_{T+1}(\delta)^2dY_T(\delta)
\end{align*}
\paragraph{Step 4: Chaining results}
Chaining the result for $R_1(a_*), R_1(\hat a)$ and $R_2(a_*), R_2(\hat a)$ together, the regret can be upper bounded by
\begin{align*}
    \SR^\Log(\hat\pi)&\le \frac{2\gamma_{T+1}(\delta)\sqrt{dY_T(\delta)}}{\sqrt{T\kappa^*}} + \frac{2\kappa^{3/2}\gamma_{T+1}^{3/2}(\delta)(dY_T(\delta))^{3/4}}{T^{3/4}\sqrt{\kappa^*}}\\
    &+ \sqrt{\mathfrak{R}^\Log(\hat \pi)}\left[\frac{\gamma_{T+1}(\delta)\sqrt{dY_T(\delta)}}{\sqrt{T}} + \frac{\kappa^{3/2}\gamma_{T+1}^{3/2}(\delta)(dY_T(\delta))^{3/4}}{T^{3/4}}\right]\\
    &+ \frac{1}{32T}\kappa^3\gamma_{T+1}(\delta)^2dY_T(\delta)
\end{align*}
We now use a fact that if for $b,c>0$,
\begin{align*}
    x^2\le bx + c,
\end{align*}
then it follows that 
\begin{align*}
    x\le b+\sqrt{c}.
\end{align*}
Take $\sqrt{\SR^\Log(\hat\pi)}$ as $x$, we have that 
\begin{align*}
    \sqrt{\SR^\Log(\hat\pi)}&\le \left[\frac{\gamma_{T+1}(\delta)\sqrt{dY_T(\delta)}}{\sqrt{T}} + \frac{\kappa^{3/2}\gamma_{T+1}^{3/2}(\delta)(dY_T(\delta))^{3/4}}{T^{3/4}}\right] \\
    &+ \bigg [\frac{2\gamma_{T+1}(\delta)\sqrt{dY_T(\delta)}}{\sqrt{T\kappa^*}} + \frac{2\kappa^{3/2}\gamma_{T+1}^{3/2}(\delta)(dY_T(\delta))^{3/4}}{T^{3/4}\sqrt{\kappa^*}}\\
    &+ \frac{1}{16T}\kappa^3\gamma_{T+1}(\delta)^2dY_T(\delta)\bigg]^{1/2}
\end{align*}
Then square both sides and use $(a+b)^2\le 2a^2 + 2b^2$:
\begin{align*}
    \SR^\Log(\hat\pi) &\le \frac{8\gamma_{T+1}(\delta)\sqrt{dY_T(\delta)}}{\sqrt{T\kappa^*}}\\
    &+\frac{8\kappa^{3/2}\gamma_{T+1}^{3/2}(\delta)(dY_T(\delta))^{3/4}}{T^{3/4}\sqrt{\kappa^*}}\\
    &+ \frac{1}{8T}\kappa^3\gamma_{T+1}(\delta)^2dY_T(\delta) + \frac{4\gamma_{T+1}^2(\delta){dY_T(\delta)}}{T} \\
    &+\frac{4\kappa^{3}\gamma_{T+1}^{3}(\delta)(dY_T(\delta))^{3/2}}{T^{3/2}}
\end{align*}
Finally replace $\delta$ with $\delta/2$ finishes the proof.
\subsection[Difference from Faury et al.]{Detailed explanation of on our difference with \citet{faury2020improved}}\label{appendix:comparison-faury}
As is mentioned before, 
the idea of using $L_t$ is inspired by \citet{faury2020improved}. However, \color{red}{we obtain a instance-wise minimax optimal upper bound while they only obtained a general $O(d\sqrt{T})$ type of bound so we made extra contributions in the analysis part. }\color{black}
What's more,
our algorithm is different from \citet{faury2020improved} in several aspects where we did novel algorithmic enhancements explained in the following:
\begin{enumerate}
    \item We do not need to solve a non-convex optimization problem compared to \citet{faury2020improved}, making our algorithm computationally tractable.
    \item Our purpose of $\theta_t$ is completely different from that of \citet{faury2020improved}. In \citet{faury2020improved}, it was used to shape an admissible parameter set. In our case, not only is our algorithm free of the admissible set shaped by $\theta_t$ which caused their procedure to be non-convex, but we also incorporated $\theta_t$ into the quantification of uncertainty, the key to max-uncertainty type algorithm, serving as a non-trivial extension from the linear case to the logistic case, obtaining an instance-wise optimal upper bound. As a result, \cref{lem:dec-uncertainties-GOAT,lem:dec-uncertainties-order-2GOAT} are novel.
    \item The matrix is directly used in the algorithm, allowing us to construct an estimation of Hessian in an online-fashion. Sherman-Morrison can then be used, whose computational cost is $O(d^2)$, to avoid matrix inversion, whose computational cost is $O(d^3)$, in each step. We also directly use as the uncertainty quantification (exploration bonus in cumulative regret setting) while \citet{faury2020improved} still uses $H_t$
    and $V_t$ as the uncertainty quantification.
\end{enumerate}


\section[Regret analysis of THaTS]{Regret analysis of \THATS (\cref{thm:THATS-regret})}
In this section, we first state formally our regret bound on \THATS where detailed constant and polynomial dependency is presented. After that we give proof on it.
\begin{theorem}[Formal statement of \cref{thm:THATS-regret}]
    Under Assumptions \ref{ass:bounded_feature} and \ref{ass:bounded_para}, 
    there exists some universal constant $\mathfrak c>0$ such that the following holds:
    Let $\delta\in[0,1)$, $T\ge 1$ be arbitrary.
    Then, with probability at least $1-\delta$, it holds that
    the simple regret of the action $\tilde a$ output by 
    \THATS (\cref{alg:THATS})
    with $\lambda$ chosen to be $\lambda_T$ in \cref{eq:lambda}, 
    after $T$ rounds 
    is upper bounded by
    \begin{align*}
        \SR(\tilde a,\theta_*)&\le \frac{4d\gamma_{T+1}(\delta)\sqrt{Y_T(\delta)}}{\sqrt{T{\kappa_*}}} \\
    &+ \frac{8\sqrt{\gamma_{T+1}^5(\delta)}\kappa^{3/2}\left(d^{5/3}Y_T(\delta)\right)^{3/4}}{T^{3/4}\sqrt{\kappa_*}}+\frac{4d\gamma_{T+1}(\delta)\kappa^2}{2T\sqrt{\kappa_*}}\left(16N_T(\delta)+2dY_T(\delta)\right)\\
    &+\frac{1}{T}\bigg(4d^2\kappa^3\gamma_{T+1}^2(\delta)Y_T(\delta)+ 128d\kappa^2\gamma_{T+1}^2(\delta)N_T(\delta) \bigg)\\
    &+ \frac{4d\kappa^4}{T\lambda_1^2}\gamma^4_{T+1}(\delta) \left(dY_T(\delta)+16N_T(\delta)\right)\\
    &+\frac{3d^2\gamma_{T+1}^2(\delta){Y_T(\delta)}}{{T}}+ \frac{12{\gamma_{T+1}^5(\delta)}\kappa^{3}\left(d^{5/3}Y_T(\delta)\right)^{3/2}}{T^{3/2}}\\
    &+ \frac{3d^2\gamma_{T+1}^4(\delta)\kappa^4}{4T^2}\Big(N_T(\delta)+dY_T(\delta)\Big)^2
    \end{align*}
    where $\gamma_T(\delta)=\tilde \cO(\sqrt d)$ is as defined in \cref{lem:conf-width-GOAT},
    $Y_T(\delta)=\log((d\lambda_T +T)/(d\lambda_T))$ and $N_T(\delta)=\log(\log(2T/\delta))$.
\end{theorem}
\subsection{New Confidence Set}
Recall the definition of $\cE_t(\delta,\bar\theta_t)$:
\begin{align*}
\cE_t(\delta,\bar\theta_t)=\{\theta\in \mB_d(S):\mcL^\lambda_t(\theta) - \mcL^\lambda_t(\bar \theta_t)\le 2\beta_t^2(\delta)\}
\end{align*}
As is promised in \cref{sec:TS}, we relate $\cE_t(\delta)$ to the confidence set $\cC_t(\delta)$.
\EtRelateCt*
\begin{proof}
    By definition of $\bar\theta_t$, we have $\bar\theta_t\in \mB_d(S)$. Since $\theta_*\in \mB_d(S)$, it then follows that with probability at least $1-\delta$,
    \begin{align*}
        \mcL_t^\lambda(\bar\theta_t) - \mcL_t^\lambda(\hat\theta_t) \le \mcL_t^\lambda(\theta_*)- \mcL_t^\lambda(\hat\theta_t)\le \beta_t(\delta)^2,
    \end{align*}
    where in the first inequality we used the fact that $\hat\theta_t$ is the global minimizer of $\mcL_t^\lambda$, with probability at least $1-\delta$, $\theta_*\in \cC_t(\delta,\hat\theta_t)\cap \mB_d(S)$ and $\mcL_t^\lambda(\bar\theta_t)\le \mcL_t^\lambda(\theta_*)$ by definition of $\bar\theta_t$; in the second inequality we used \cref{lem:confidence_set}.
    For the ``furthermore" part, let $\theta\in \cC_t(\delta,\hat\theta_t)\cap \mB_d(S)$, by construction $\mcL_t^\lambda(\bar\theta_t)\le \mcL_t^\lambda (\theta)$ hence $\mcL_t^\lambda(\theta)-\mcL_t^\lambda(\bar\theta_t)=|\mcL_t^\lambda(\theta)-\mcL_t^\lambda(\bar\theta_t)|$ and by triangle inequality,
    \begin{align*}
        \mcL_t^\lambda(\theta)-\mcL_t^\lambda(\bar\theta_t)&\le |\mcL_t^\lambda(\theta)-\mcL_t^\lambda(\hat\theta_t)| + |\mcL_t^\lambda(\bar\theta_t)-\mcL_t^\lambda(\hat\theta_t)|\\
        &=\mcL_t^\lambda(\theta)-\mcL_t^\lambda(\hat\theta_t) + \mcL_t^\lambda(\bar\theta_t)-\mcL_t^\lambda(\hat\theta_t)\\
        &\le 2\beta_t(\delta)^2.
    \end{align*}
\end{proof}
Given the new confidence set, we can show a similar lemma to \cref{lem:conf-width-GOAT}.
\begin{lemma}\label{lem:TS-conf-width}
    Let $\delta\in [0,1)$. With probability at least $1-2\delta$,  for all $t \ge 1$ and $\theta\in \cE_t(\delta,\bar\theta_t)$,
    \begin{align*}
        \|\theta - \theta_*\|_{H_t(\theta_*)}\le 2\gamma_t(\delta).
    \end{align*}
\end{lemma}
\begin{proof}
    The proof is almost exactly the same as that of \cref{lem:conf-width-GOAT}. To be more specific, $\hat\theta_t$ is replaced by $\bar\theta_t$ a few times when needed.
    We start from Taylor expansion.
    For all $\theta'\in \RR^d$, we have that 
    \begin{align*}
        \mcL_t^\lambda(\theta')=\mcL_t^\lambda(\theta_*) + \nabla\mcL_t^\lambda(\theta_*)^\top(\theta'-\theta_*)+\frac{1}{2}\|\theta'-\theta_*\|^2_{\tilde G_t(\theta_*,\theta')}.
    \end{align*}
    \textcolor{red}{Let $\theta\in \cE_t(\delta,\bar\theta_t)$.} Rearrange the terms, apply absolute value and plug in $\theta$,
    \begin{align*}
        |\mcL_t^\lambda(\theta) - \mcL_t^\lambda(\theta_*) - \nabla \mcL_t^\lambda (\theta_*)(\theta-\theta_*)|=\frac{1}{2}\|\theta-\theta_*\|^2_{\tilde G_t(\theta_*,\theta)}
        \ge \frac{1}{2(2+2S)}\|\theta-\theta_*\|^2_{H_t(\theta_*)},
    \end{align*}
    where the last inequality follows from \cref{lem:selfconc} and $\theta,\theta_*\in \mB_d(S)$.
    It remains to upper bound the left most side of the above equation.
    By triangle inequality we can split it into two terms and we bound them separately.
    \begin{align*}
        |\mcL_t^\lambda(\theta) - \mcL_t^\lambda(\theta_*) - \nabla \mcL_t^\lambda (\theta_*)(\theta-\theta_*)| \le \underbrace{|\mcL_t^\lambda(\theta) - \mcL_t^\lambda(\theta_*) |}_{(a)}+\underbrace{|\nabla \mcL_t^\lambda (\theta_*)(\theta-\theta_*)|}_{(b)}.
    \end{align*}
    For $(a)$, with probability at least $1-\delta$ we have that $\theta_*\in \cC_t(\delta)\cap \mB_d(S)$, then
    \begin{align*}
        (a)&=\color{red}{|\mcL_t^\lambda(\theta) - \mcL_t^\lambda(\bar\theta_t)+\mcL_t^\lambda(\bar\theta_t)-\mcL_t^\lambda(\theta_*)|}\\
        &\le |\mcL_t^\lambda(\theta) - \mcL_t^\lambda(\bar\theta_t)|+|\mcL_t^\lambda(\bar\theta_t)-\mcL_t^\lambda(\theta_*)|\\
        &=\mcL_t^\lambda(\theta) - \mcL_t^\lambda(\bar\theta_t)+\mcL_t^\lambda(\bar\theta_t)-\mcL_t^\lambda(\theta_*)\\
        &\le 4\beta_t(\delta)^2,
    \end{align*}
    where in the last line we used \cref{lem:TS_conf_set} and that $\theta_*\in \cC_t(\delta)\cap\mB_d(S)$.
    For $(b)$, note that by definition of $\hat\theta_t$, $\nabla \mcL_t^\lambda(\theta_*)=g_t(\theta_*)-g_t(\hat\theta_t) $. To be more specific, for all $\theta\in \RR^d$,
    \begin{align}
        \nabla_\theta\mcL(\theta)=g_t(\theta)-\underbrace{\sum_{i=1}^{t-1} \phi(A_i)X_i}_{=g_t(\hat\theta_t)\text{ by \cref{eq:rmle_grad}}}.
    \end{align}
    Then by Cauchy-Schwarz, by \cref{lem:faury_conf}, with probability $1-\delta$, we have that $\|g_t(\hat\theta_t) -g_t(\hat\theta_*)\|_{H_t^{-1}(\theta_*)}\le \rho_t(\delta)$, then
    \begin{align*}
        (b)&\le \|g_t(\hat\theta_t) - g_t(\theta_*)\|_{H_t^{-1}(\theta_*)}\|\theta-\theta_*\|_{H_t(\theta_*)}\\
        &\le \rho_t(\delta)\|\theta-\theta_*\|_{H_t(\theta_*)},
    \end{align*}
    where in the last inequality we used \cref{lem:faury_conf}.
    Chaining all the inequality and use the fact that $\PP(A\cap B)\ge 1-\PP(A^c)-\PP(B^c)$, we have that with probability at least $1-2\delta$,
    \begin{align*}
        \frac{1}{2(2+2 S)}\|\theta-\theta_*\|_{H_t(\theta_*)}^2\le \rho_t(\delta)\|\theta-\theta_*\|_{H_t(\theta_*)} + 4\beta_t(\delta)^2.
    \end{align*}
    Solving the above inequality gives us
    \begin{align*}
        \|\theta-\theta_*\|_{H_t(\theta_*)}\le (4+4S)\rho_t(\delta)+\sqrt{(16S+16)}\beta_t(\delta)\le 2\gamma_t(\delta).
    \end{align*}
\end{proof}

\subsection{Analysis on the exploration done by \THATS compared to \MULog}
The analysis of \THATS is highly related to that of \MULog. Hence in order to  identify the actions and parameter that gives the maximum uncertainty, we define them to be
    \begin{align}
        \theta_t^{\MU},A_t^{\MU} &= \argmax_{a\in \cA,\theta\in \cV_t} U(a,\theta,L_t)\label{eq:theta-t-MU-and-At-MU}\\
        \omega_t^{\MU}&=\argmin_{\theta\in \cE_t(\delta,\bar\theta_t)} \dot\mu(\phi(A_t^{\MU})^\top\theta).\label{eq:omega-t-MU}
    \end{align}
We would like to emphasize that $\theta_t^\MU,A_t^\MU$ are \textbf{not} actions and parameters picked by the max-uncertainty algorithm. They are simply the ones that give the maximum uncertainty at time $t$.
\textcolor{red}{The actions $A_t$ are pulled by \THATS instead of max uncertainty in this section.}
\TsToMaxUnc*

\begin{proof}
    We start by rewriting the right hand side of the inequality
    \begin{align*}
        &\mE\left[\sqrt{\dot\mu(\phi(A_t)^\top\bar\theta_t)}\|\phi(A_t)\|_{L_{t}^{-1}}\Big\vert \cF_{t-1}\right]\\
        \;=&\mE\left[\max_{x\in \mS^{d-1}}\left\langle L_t^{-1/2}x,\sqrt{\dot\mu(\phi(A_t)^\top\bar\theta_t)}\phi(A_t)\right\rangle \Big\vert \cF_{t-1}\right]\\
        \ge\;& \mE\left[\left|\left\langle\frac{L_t^{-1/2}\cdot L_t^{1/2}\tilde \theta_t}{\|L_t^{1/2}\tilde\theta_t\|_2},\sqrt{\dot\mu(\phi(A_t)^\top\bar\theta_t)}\phi(A_t)\right\rangle\right|\Big \vert \cF_{t-1}\right]\\
        \ge\;& \mE\left[\left|\left\langle\frac{\tilde \theta_t}{\|L_t^{1/2}\tilde\theta_t\|_2},\sqrt{\dot\mu(\phi(A_t^\MU)^\top\bar\theta_t)}\phi(A_t^\MU)\right\rangle\right|\Big \vert \cF_{t-1}\right],
    \end{align*}
    where in the last line we used the definition of $A_t$. Note that $\bar\theta_t$ and $\omega_t^{\MU}$ are both in the confidence set $\cE_t(\delta,\bar\theta_t)\cap\mB_d(S)$,
    \begin{align*}
        &\mE\left[\left|\left\langle\frac{\tilde \theta_t}{\|L_t^{1/2}\tilde\theta_t\|_2},\sqrt{\dot\mu(\phi(A_t^\MU)^\top\bar\theta_t)}\phi(A_t^\MU)\right\rangle\right|\Big \vert \cF_{t-1}\right]\\ 
        \ge\; & \mE\left[\left|\left\langle\frac{\tilde \theta_t}{\|L_t^{1/2}\tilde\theta_t\|_2},\sqrt{\dot\mu(\phi(A_t^\MU)^\top\omega_t^{\MU})}\phi(A_t^\MU)\right\rangle\right|\Big \vert \cF_{t-1}\right]
    \end{align*}
    Since $\tilde\theta_t\sim \cN(0,L_t^{-1})$, we can rewrite it as $\tilde\theta_t = L_t^{-1/2}Y_T(\delta)$ for $Y_T(\delta)\sim \cN(0,I)$ given the past. Then plug it back in the above expression,
    \begin{align*}
        &\mE\left[\left|\left\langle\frac{\tilde \theta_t}{\|L_t^{1/2}\tilde\theta_t\|_2},\sqrt{\dot\mu(\phi(A_t^\MU)^\top\omega_t^{\MU})}\phi(A_t^\MU)\right\rangle\right|\Big \vert \cF_{t-1}\right] \\
        = \;&\mE\left[\left|\left\langle\frac{L_t^{-1/2}Y_T(\delta)}{\|Y_T(\delta)\|_2},\sqrt{\dot\mu(\phi(A_t^\MU)^\top\omega_t^{\MU})}\phi(A_t^\MU)\right\rangle\right|\Big \vert \cF_{t-1}\right]\\
        = \;&\mE\left[\left|\left\langle\frac{Y_T(\delta)}{\|Y_T(\delta)\|_2},\sqrt{\dot\mu(\phi(A_t^\MU)^\top\omega_t^{\MU})}L_t^{-1/2}\phi(A_t^\MU)\right\rangle\right|\Big \vert \cF_{t-1}\right]\\
        = \;&\sqrt{\dot\mu(\phi(A_t^\MU)^\top\omega_t^{\MU})}\|\phi(A_t^\MU)\|_{L_t^{-1}}\mE\left[\left|\left\langle\frac{Y_T(\delta)}{\|Y_T(\delta)\|_2},\frac{L_t^{-1/2}\phi(A_t^\MU)}{\|\phi(A_t^\MU)\|_{L_t^{-1}}}\right\rangle\right|\Big \vert \cF_{t-1}\right]\\
        =\;&\sqrt{\dot\mu(\phi(A_t^\MU)^\top\omega_t^\MU)}\|\phi(A_t^\MU)\|_{L_t^{-1}}I( A_t^\MU,L_t),
    \end{align*}
    where the third line used $L_t^{-1/2}$ is symmetric (as $L_t$ is positive definite); the fourth line follows by $A_t^\MU,\omega_t^{\MU}$ are $\cF_{t-1}$-measurable and \cref{prop:measurable-cond-exp} and the last line follows by the definition of $I( A_t^\MU,L_t)$.
\end{proof}
\begin{corollary}\label{coro:TS-to-max-uncertainty}
    For all $t\ge 1$, it holds almost surely that
    \begin{align*}
        U(A_t^\MU,\omega_t^\MU,L_t)\le \sqrt{\frac{\pi d}{2}}\mE\left[U(A_t,\bar\theta_t,L_t)\Big\vert\cF_{t-1}\right].
    \end{align*}
\end{corollary}
\begin{proof}
    The proof follows by dividing both sides of the inequality showed in \cref{lem:TS-to-max-uncertainty} by $I( A_t^\MU, L_t)$ and then showing that $I( A_t^\MU, L_t)$ is lower bounded by $\sqrt{\frac{2}{\pi d}}$. The latter follows using \cref{prop:gaussian-to-uniform} and \cref{prop:uniform-sphere-inner-product}.
\end{proof}
\begin{lemma}\label{lem:TS-to-max-uncertainty-2nd-order}
    It holds almost surely that
    \begin{align*}
        \min_{\theta\in \cE_t}U(A_t^\MU,\theta,L_t)^2\cdot \bar I( A_t^\MU,L_t)\le \mE\left[U(A_t,\bar\theta_t ,L_t)^2\Big\vert\cF_{t-1}\right],
    \end{align*}
    where for $a\in \cA$ and $L\succeq 0$
    \begin{align*}
        \bar I(a,L)=\int_{\mS^{d-1}}\left(\left\langle x,\frac{L^{-1/2}\phi(a)}{\|L^{-1/2}\phi(a)\|}\right\rangle\right)^2dx.
    \end{align*}
\end{lemma}

\begin{proof}
    The proof is similar to that of \cref{lem:TS-to-max-uncertainty}. We replace square terms a few times when needed compared to the proof of \cref{lem:TS-to-max-uncertainty}.
    We start by rewriting the right hand side of the inequality
    \begin{align*}
        &\mE\left[\dot\mu(\phi(A_t)^\top\bar\theta_t)\|\phi(A_t)^\top\tilde\theta_t\|^{2}_{L_{t}^{-1}}\Big\vert \cF_{t-1}\right]\\
        =\;&\mE\left[\left(\max_{x\in \mS^{d-1}}\left\langle L_t^{-1/2}x,\sqrt{\dot\mu(\phi(A_t)^\top\bar\theta_t)}\phi(A_t)\right\rangle\right)^2 \Big\vert \cF_{t-1}\right]\\
        \ge\;& \mE\left[\left|\left\langle\frac{\tilde \theta_t}{\|L_t^{1/2}\tilde\theta_t\|_2},\sqrt{\dot\mu(\phi(A_t)^\top\bar\theta_t)}\phi(A_t)\right\rangle\right|^2\Big\vert\cF_{t-1}\right]\\
        \ge\;& \mE\left[\left|\left\langle\frac{\tilde \theta_t}{\|L_t^{1/2}\tilde\theta_t\|_2},\sqrt{\dot\mu(\phi(A_t^\MU)^\top\bar\theta_t)}\phi(A_t^\MU)\right\rangle\right|^2\Big\vert\cF_{t-1}\right],
    \end{align*}
    where in the last line we used the definition of $A_t$ and that
    \begin{align}
        \argmax_{a\in \cA}\sqrt{\dot\mu(\phi(a)^\top\bar\theta_t)}\left|\left\langle\phi(a),\tilde\theta_t\right\rangle\right| = \argmax_{a\in \cA}\dot\mu(\phi(a)^\top\bar\theta_t)^2\left(\left\langle\phi(a),\tilde\theta_t\right\rangle\right)^2\label{eq:argmax-2nd-order}
    \end{align}
    By definition of $\omega_t$, the same reasoning as \cref{eq:argmax-2nd-order},
    \begin{align*}
        &\mE\left[\left|\left\langle\frac{\tilde \theta_t}{\|L_t^{1/2}\tilde\theta_t\|_2},\sqrt{\dot\mu(\phi(A_t^\MU)^\top\bar\theta_t)}\phi(A_t^\MU)\right\rangle\right|^2\Big\vert\cF_{t-1}\right] & \\
        \ge\;& \mE\left[\left|\left\langle\frac{\tilde \theta_t}{\|L_t^{1/2}\tilde\theta_t\|_2},\sqrt{\dot\mu(\phi(A_t^\MU)^\top\omega_t^\MU)}\phi(A_t^\MU)\right\rangle\right|^2\Big\vert \cF_{t-1}\right]
    \end{align*}
    Since $\tilde\theta_t\sim \cN(0,L_t^{-1})$, we can rewrite it as $\tilde\theta_t = L_t^{-1/2}Y_T(\delta)$ for $Y_T(\delta)\sim \cN(0,I)$. Then plug it back in the above expression,
    \begin{align*}
        &\mE\left[\left|\left\langle\frac{\tilde \theta_t}{\|L_t^{1/2}\tilde\theta_t\|_2},\sqrt{\dot\mu(\phi(A_t^\MU)^\top\omega_t^\MU)}\phi(A_t^\MU)\right\rangle\right|^2\Big\vert \cF_{t-1}\right]\\
        = \;&\mE\left[\left|\left\langle\frac{Y_T(\delta)}{\|Y_T(\delta)\|_2},L_t^{-1/2}\sqrt{\dot\mu(\phi(A_t^\MU)^\top\omega_t^{\MU})}\phi(A_t^\MU)\right\rangle\right|^2\cF_{t-1}\right]\\
        = \;&\dot\mu(\phi(A_t^\MU)^\top\omega_t^{\MU})\|\phi(A_t^\MU)\|^2_{L_t^{-1}}\mE\left[\left|\left\langle\frac{Y_T(\delta)}{\|Y_T(\delta)\|_2},\frac{L_t^{-1/2}\phi(A_t^\MU)}{\|\phi(A_t^\MU)\|_{L_t^{-1}}}\right\rangle\right|^2\cF_{t-1}\right]\\
        =\;& U(A_t^\MU,\omega_t^\MU,L_t)^2\cdot \bar I( A_t^\MU,L_t),
    \end{align*}
    where the second line follows because $L_t^{-1/2}$ is symmetric (as $L_t$ is positive definite); the third line follows by $A_t^\MU,\omega_t^{\MU}$ are $\cF_{t-1}$-measurable and \cref{prop:measurable-cond-exp}; the fourth line follows by definition of $\bar I( A_t^{\MU},L_t)$.
\end{proof}
\begin{corollary}\label{coro:TS-to-max-uncertainty-2nd-order}
    For all $t\ge 1$, it holds almost surely that\todos{Fix}
    \begin{align*}
        U(A_t^\MU,\omega_t^\MU,L_t)^2 \le d\mE\left[U(A_t,\bar\theta_t,L_t)^2\Big\vert\cF_{t-1}\right].
    \end{align*}
\end{corollary}
\begin{proof}
    The proof follows by dividing both sides of the inequality showed in \cref{lem:TS-to-max-uncertainty-2nd-order} by $\bar I( A_t^\MU,L_t)$ and then showing that $\bar I_t(A_t^\MU,L_t)$ is exactly $1/d$. The latter follows using \cref{prop:gaussian-to-uniform} and \cref{prop:uniform-sphere-inner-product-squared}. 
\end{proof}
\subsection{Regret Analysis \THATS}
Now we start the regret analysis. 
\paragraph{Step 1: Regret Decomposition}
Let $\tilde a = \argmax_{a\in \cA}\phi(a)\tilde\theta_{T+1}^\Log$. The beginning of the analysis is similar to that of \cref{alg:max-uncertainty-logSR-GOAT}.
We do the same regret decomposition:
\begin{align}
    D_{\Log}(a_*,\tilde \theta_{T+1}^\Log)&=|\mu(\phi(a_*)^\top \tilde\theta_{T+1}^{\Log})-\mu(\phi(a)^\top \theta_*)|\notag\\
    &\le \underbrace{\dot\mu(\phi(a_*)^\top \theta_*)|\phi(a_*)^\top (\theta_*-\tilde\theta_{T+1}^{\Log})|}_{R_1(a_*)} + \underbrace{\ddot\mu(\xi_{a_*})|\phi(a_*)^\top (\theta_*-\tilde\theta_{T+1}^{\Log})|^2}_{R_2(a_*)},\label{eq:regret_decomp_a_star_THATS}
\end{align}
where $\xi_{a_*}$ in the last line is some point in between $\phi(a_*)^\top \theta_*$ and $\phi(a_*)^\top \tilde \theta_{T+1}^\Log$.
Similarly for $D_{\Log}( \tilde a,\tilde \theta_{T+1}^\Log)$, we have that 
\begin{align}
    D_{\Log}(\tilde a,\tilde \theta_{T+1}^\Log)\le \underbrace{\dot\mu(\phi(\tilde a)^\top \theta_*)|\phi( \tilde a)^\top (\theta_*-\tilde\theta_{T+1}^{\Log})|}_{R_1( \tilde a)} + \underbrace{\ddot\mu(\xi_{ \tilde a})|\phi( \tilde a)^\top (\theta_*-\tilde\theta_{T+1}^{\Log})|^2}_{R_2( \tilde a)},\label{eq:regret_decomp_a_hat_THATS}
\end{align}
\paragraph{Step 2.1: Bounding $R_1(a_*)$.}
For $R_1(a_*)$, we copy the analysis from that of \cref{alg:max-uncertainty-logSR-GOAT}, as what we did in \cref{eq:R_1_max_uncer}. Recall $\tilde\theta_{T+1}^\Log\in \cV_{T+1}$.
\begin{align*}
    R_1(a_*)&\le \max_{a\in \cA}\dot\mu\left(\phi(a)^\top\theta_*\right)\|\phi(a)\|_{H_{T+1}^{-1}(\theta_*)}\|\theta_*-\tilde\theta_{T+1}^\Log\|_{H_{T+1}(\theta_*)}\\
    &\le 2\gamma_{T+1}(\delta)\max_{a\in \cA}{\dot\mu\left(\phi(a)^\top\theta_*\right)}\|\phi(a)\|_{H_{T+1}^{-1}(\theta_*)}\notag\\
    &\le2\gamma_{T+1}(\delta)\max_{a\in \cA}{\dot\mu\left(\phi(a)^\top\theta_*\right)}\|\phi(a)\|_{L_{T+1}^{-1}}\notag\\
    &\le2\gamma_{T+1}(\delta)\sqrt{\dot\mu\left(\phi(a_*)^\top\theta_*\right)}\max_{a\in \cA, \theta\in \cV_{T+1}}\sqrt{\dot\mu\left(\phi(a)^\top\theta\right)}\|\phi(a)\|_{L_{T+1}^{-1}}\notag\\
    &=2\gamma_{T+1}(\delta)\sqrt{1/\kappa_*}\max_{a\in \cA,\theta\in \cV_{T+1}}U(a,\theta,L_{T+1}),
\end{align*}
where in the second line we use \cref{lem:TS-conf-width}; in the third line we use $L_{T+1}\preceq H_{T+1}(\theta_*)$; in the fourth line we use $\theta_*\in \cV_{T+1}$.
Now we use \cref{lem:dec-uncertainties-GOAT}, it follows that for all $1\le t \le T$, note that $\cV_{t+1}\subseteq \cV_{t}$ and $L_t\preceq L_{t+1}$ hence satisfying the conditions of \cref{lem:dec-uncertainties-GOAT}. Using the same argument as \cref{eq:bound_last_by_avg_1,eq:bound_last_by_avg_2},
\begin{align}
    \max_{a\in \cA,\theta\in \cV_{T+1}}U(a,\theta,L_{T+1})\le \frac{1}{T}\sum_{t=1}^T\max_{a\in \cA,\theta\in \cV_{t}}U(a,\theta,L_{t}).\label{eq:bound_last_by_average}
\end{align}
Recall that 
\begin{align*}
    U(a,\theta,L)=\sqrt{\dot\mu(\phi(a)^\top\theta)}\|\phi(a)\|_{L^{-1}}.
\end{align*}
By \cref{eq:bound_last_by_average}, with probability $1-2\delta$ (from \cref{lem:TS-conf-width}), we have that
\begin{align*}
    &\max_{a,\theta\in \cV_{T+1}}U(a,\theta,L_{T+1})\\
    \le\; &\frac{1}{T}\sum_{t=1}^T \max_{a,\theta\in \cV_{t}}U(a,\theta,L_t) \\
    =\;&\frac{1}{T}\sum_{t=1}^T U(A_t^\MU,\theta_t^\MU,L_t) \tag{c.f. \cref{eq:theta-t-MU-and-At-MU}}\\
    \le\;& \textcolor{red}{\frac{1}{T}\sum_{t=1}^T  U(A_t^\MU,\omega_t^\MU,L_t)+ \frac{1}{4}\|\phi(A_t^\MU)\|_{L_t^{-1}}|\phi(A_t^\MU)^\top (\theta_t^\MU - \omega_t^\MU)| }\\ 
    \le\;& \frac{1}{T}\sum_{t=1}^T  U(A_t^\MU,\omega_t^\MU,L_t) + \frac{1}{4}\|\phi(A_t^\MU)\|^2_{L_t^{-1}} \|\theta_t^\MU - \omega_t^\MU\|_{L_t} \\
    \le  & \frac{1}{T}\sum_{t=1}^T  U(A_t^\MU,\omega_t^\MU,L_t) + \frac{1}{4}\|\phi(A_t^\MU)\|^2_{L_t^{-1}} \|\theta_t^\MU - \omega_t^\MU\|_{H_t(\theta_*)}\\
\end{align*}
where the first line is the key difference from the proof of \cref{alg:max-uncertainty-logSR-GOAT}; in the second line we use \cref{eq:bound_last_by_average}; in the third line we plug in the definition of $A_t^\MU$ and $\theta_t^\MU$; in the fourth line we use Taylor expansion and upper bound $\ddot\mu(\cdot)$ by $1/4$ (\cref{eq:self-concordant,eq:logistic-lipschitz}); in the fifth line we used Cauchy-Schwarz inequality; in the sixth line we used the fact that $L_t\preceq H_t(\theta_*)$.
Now we give a way to bound $\|\theta_t^\MU - \omega_t^\MU\|_{H_t(\theta_*)}$. By triangle inequality, and $\theta_t^\MU,\theta_*\in \cE_t(\delta,\bar\theta_t)$,  
\begin{align}
    \|\theta_t^\MU - \omega_t^\MU\|_{H_t(\theta_*)} &\le\|\theta_t^\MU - \theta_*\|_{H_t(\theta_*)} + \|\theta_* - \omega_t^\MU\|_{H_t(\theta_*)}\le 4\gamma_{T+1}(\delta).\label{eq:TS-triangle-conf-width}
\end{align}
Plug in the above bound, we have
\begin{align*}
    & \frac{1}{T}\sum_{t=1}^T  U(A_t^\MU,\omega_t^\MU,L_t) + \frac{1}{4}\|\phi(A_t^\MU)\|^2_{L_t^{-1}} \|\theta_t^\MU - \omega_t^\MU\|_{H_t(\theta_*)} \\
    \le\;  & \frac{1}{T}\sum_{t=1}^T  U(A_t^\MU,\omega_t^\MU,L_t) + \frac{1}{4}\cdot 4\gamma_{T+1}(\delta)\|\phi(A_t^\MU)\|^2_{L_t^{-1}}  \\
    \le\; & \underbrace{\frac{1}{T}\sum_{t=1}^T  U(A_t^\MU,\omega_t^\MU,L_t) }_{(i)} + \underbrace{\gamma_{T+1}(\delta)\frac{1}{T}\sum_{t=1}^T \|\phi(A_t^\MU)\|^2_{L_t^{-1}}}_{(ii)},
\end{align*}
in the second line we used the fact that both $\theta_t^\MU$ and $\omega_t^\MU$ are in $\cV_t$ and \cref{lem:TS-conf-width}. 

Now it remains to bound $(i)$ and $(ii)$.
For $(i)$, we use our result on \THATS \cref{coro:TS-to-max-uncertainty} to get
\begin{align*}
    (i)\le \frac{1}{T}\sum_{t=1}^T \sqrt{d}\mE\left[U(A_t,\bar\theta_t,L_t)\Big\vert \cF_{t-1}\right],
\end{align*}
where the expectation is w.r.t the randomness of $\tilde\theta_t$ in each round.
Now everything starts to follow almost exactly the same to the part of analysis in \cref{alg:max-uncertainty-logSR-GOAT} again. Specifically, \cref{eq:Taylor_EPL}. 
Before we can invoke Elliptical Potential Lemma (\cref{lem:epl}), there is still one more step to be done, that is, bounding conditional expectations with realizations. To be more specific, note that for a stochastic process $\{Y_t\}_{t\ge 1}$ adapted to a filtration $\{\bF_t\}_{t\ge 1}$, even if $\sum_{t=1}^T Y_t\le c_T$ for some constants $\{c_t\}_{t\ge 1}$ with probability $1$, it does not necessarily follow that $\mE[\sum_{t=1}^T Y_t\vert \bF_t]\le c_T$ with probability 1\footnote{A counterexample is given at \url{https://math.stackexchange.com/questions/4834315/sum-of-conditional-expectations-of-a-bounded-stochastic-process}.} .
The good news is a similar inequality holds with some blow up of $\{c_t\}_{t\ge 1}$, which is shown in \cref{cor:Reversed-Bernstein-inequality-for-martingales}.
From \cref{cor:Reversed-Bernstein-inequality-for-martingales},
and putting back the definition of $U(A_t,\bar\theta_t,L_t )$
\begin{align*}
    (i)
    \le \gamma_{T+1}(\delta)\sqrt{d}\frac{1}{T}\left(16N_T(\delta)+\sum_{t=1}^T{ \sqrt{\dot\mu\left(\phi(A_t)^\top\bar\theta_t\right)}}\|\phi(A_t)\|_{L_{t}^{-1}}\right),\\
\end{align*}
where $N_T(\delta)=\log(\log(2T/\delta))$ and 
Using the analysis idea in \cref{eq:Taylor_EPL}, we have that
\begin{align}
    \sum_{t=1}^T \sqrt{\dot\mu(\phi_t^\top\bar\theta_t)}\|\phi_t\|_{L_t^{-1}}&\le \sqrt{T}\sqrt{dY_T(\delta)} \notag\\
    &+ 2T^{1/4}\sqrt{\gamma_{T+1}(\delta)}\kappa^{3/2}\left(dY_T(\delta)\right)^{3/4},\label{eq:Taylor_EPL_THATS}
\end{align}
Note that there is a constant difference in the second term compared to \cref{eq:Taylor_EPL} due to our relaxation of the admissible set ($\cW_t\to\cV_t$).
\begin{align*}
    (i)&\le \frac{d\gamma_{T+1}(\delta)\sqrt{Y_T(\delta)}}{\sqrt{T}} + \frac{2\sqrt{\gamma_{T+1}^3(\delta)}\kappa^{3/2}\left(d^{5/3}Y_T(\delta)\right)^{3/4}}{T^{3/4}}\\
    &+\frac{16\sqrt{d}\gamma_{T+1}(\delta)N_T(\delta)}{T}
\end{align*}
For $(ii)$, 
\begin{align*}
    (ii)&=\gamma_{T+1}(\delta)\frac{1}{T}\sum_{t=1}^T\|\phi(A_t^\MU)\|^2_{L_t^{-1}} \\
    &\le \gamma_{T+1}(\delta)\frac{\kappa}{T}\sum_{t=1}^T\underbrace{\dot\mu(\phi(A_t^\MU)^\top \omega_t^{\MU})\|\phi(A_t^\MU)\|^2_{L_t^{-1}}}_{U(A_t^\MU,\omega_t^\MU,L_t)^2} \\
    &\le \gamma_{T+1}(\delta)\frac{d\kappa}{T}\sum_{t=1}^T\mE\left[\underbrace{\dot\mu(\phi(A_t)^\top \bar\theta_t)\|\phi(A_t)\|^2_{L_t^{-1}}}_{U(A_t,\bar\theta_t,L_t)^2}\Big\vert \cF_t\right] 
\end{align*}
where in the second line we used $\max_{a,\theta\in \cV_t}\dot\mu(\phi(a)^\top\theta)\cdot \kappa\ge 1$; in the third line we used \cref{coro:TS-to-max-uncertainty-2nd-order}; in the last line we used $L_t^{-1}\preceq \kappa V_t^{-1}$ and $\dot\mu(\cdot)\le 1/4$ (\cref{eq:logistic-lipschitz}). 

By \cref{cor:Reversed-Bernstein-inequality-for-martingales}, we upper bound the conditional expectations with the realizations
\begin{align*}
    (ii)&\le \gamma_{T+1}(\delta)\frac{d\kappa^2}{4T}\left(16N_T(\delta)+{\sum_{t=1}^T\|\phi(A_t)\|^2_{V_t^{-1}}}\right)\\
    &\le \gamma_{T+1}(\delta)\frac{d\kappa^2}{2T}\left(16N_T(\delta)+dY_T(\delta)\right),
\end{align*}
where in the last line we use the elliptical potential lemma (\cref{lem:epl}).
Putting the bounds together on $(i)$ and $(ii)$, we have that
\begin{align*}
    R_1(a_*)&\le \frac{d\gamma_{T+1}(\delta)\sqrt{Y_T(\delta)}}{\sqrt{T{\kappa_*}}} + \frac{2\sqrt{\gamma_{T+1}^5(\delta)}\kappa^{3/2}\left(d^{5/3}Y_T(\delta)\right)^{3/4}}{T^{3/4}\sqrt{\kappa_*}}\\
    &+ \gamma_{T+1}(\delta)\frac{d\kappa^2}{2T\sqrt{\kappa_*}}\left(16N_T(\delta)+dY_T(\delta)\right)
\end{align*}
\paragraph{Step 2.2: Bounding $R_1(\tilde a)$.}
We now bound $R_1(\tilde a)$. Similar as before,
\begin{align*}
    R_1( \tilde a)&=\dot\mu(\phi( \tilde a)^\top\theta_*)\left|\phi( \tilde a)^\top (\theta_*-\theta_{T+1}^\Log)\right|\\
    &\le \dot\mu(\phi( \tilde a)^\top\theta_*) \|\phi( \tilde a)\|_{H_{T+1}^{-1}(\theta_*)}\|\theta_*-\theta_{T+1}^\Log\|_{H_{T+1}(\theta_*)}\\
    &\le \sqrt{\dot\mu(\phi( \tilde a)^\top\theta_*)}\gamma_{T+1}(\delta)\max_{a\in \cA, \theta\in \cV_{T+1}}\dot\mu(\phi(a)^\top\theta)\|\phi(a)\|_{L_{T+1}}\\
    &=\sqrt{\dot\mu(\phi( \tilde a)^\top\theta_*)}\gamma_{T+1}(\delta)\frac{1}{T}\sum_{t=1}^T U(A_t^\MU,\theta_t^\MU, L_t).
\end{align*}
Following the same steps as in $R_1(a_*)$ we have that
\begin{align*}
    R_1( \tilde a) &\le \sqrt{\dot\mu(\phi(\tilde a)^\top\theta_*)}\bigg[\frac{d\gamma_{T+1}(\delta)\sqrt{Y_T(\delta)}}{\sqrt{T{}}} + \frac{2\sqrt{\gamma_{T+1}^5(\delta)}\kappa^{3/2}\left(d^{5/3}Y_T(\delta)\right)^{3/4}}{T^{3/4}}+ \frac{d\kappa^2\gamma_{T+1}^2(\delta)\left(16N_T(\delta)+dY_T(\delta)\right)}{2T}\bigg]\\
    &\le \sqrt{\dot\mu(\phi(a_*)^\top\theta_*)+\SR_\Log(\tilde a)}\bigg[\frac{d\gamma_{T+1}(\delta)\sqrt{Y_T(\delta)}}{\sqrt{T{}}} + \frac{2\sqrt{\gamma_{T+1}^5(\delta)}\kappa^{3/2}\left(d^{5/3}Y_T(\delta)\right)^{3/4}}{T^{3/4}}\\
    &+\frac{d\kappa^2\gamma_{T+1}^2(\delta)\left(16N_T(\delta)+dY_T(\delta)\right)}{2T}\bigg]
    \tag{\cref{eq:get_the_regret_equation}}
\end{align*}
\paragraph{Step 3.1: Bounding $R_2(a_*)$.}
We now deal with $R_2(a_*)$.
\begin{align}
    R_2(a_*)=&\;\ddot\mu(\xi_{a_*})|\phi(a_*)^\top (\theta_*-\tilde\theta_{T+1}^{\Log})|^2\notag\\
    \le &\;\frac{1}{4}\|\phi(a_*)\|^2_{H_{T+1}^{-1}(\theta_*)}\|\theta_*-\tilde\theta_{T+1}^{\Log}\|^2_{H_{T+1}(\theta_*)}\notag\\
    \le &\; \frac{1}{4}\gamma_{T+1}^2(\delta)\max_a\|\phi(a)\|^2_{L_{T+1}^{-1}}\notag\\
    \le &\;\kappa\gamma_{T+1}^2(\delta)\max_{a,\theta\in \cV_{T+1}}\dot\mu(\phi(a)^\top\theta)\|\phi(a)\|^2_{L_{T+1}^{-1}}\label{eq:R_2_a_star}
\end{align}
where in the second line we used $|\ddot\mu|\le \dot\mu\le 1/4$ (\cref{eq:self-concordant,eq:logistic-lipschitz}) and Cauchy-Schwarz; in the third line we used \cref{lem:TS-conf-width} and that $\tilde\theta_{T+1}\in \cE_{T+1}(\delta)$; in the fourth line we used the definition of $\cV_{T+1}$. We now use \cref{lem:dec-uncertainties-order-2GOAT}. To be more specific, we let $K=L_t$, $K'=L_{T+1}$ and $\cY'=\cV_{T+1}$, $\cY=\cV_t$ for $1\le t\le T$. Since $\cV_{T+1}=\cV_{T}\cap \cE_{T+1}\subset \cV_t$ and $L_{T+1}=L_t+\sum_{i=t}^{T} \dot\mu(\phi_i^\top\theta_i')\phi_i\phi_i^\top$ where $\phi_i=\phi(A_i)$, the conditions of \cref{lem:dec-uncertainties-order-2GOAT} are satisfied. Then running the same argument as \cref{eq:bound_last_by_avg_1,eq:bound_last_by_avg_2},
\begin{align*}
    &\kappa^2\gamma_{T+1}^2(\delta)\max_{a,\theta\in \cV_{T+1}}\dot\mu(\phi(a)^\top\theta)^2\|\phi(a)\|^2_{L_{T+1}^{-1}}\\
    \le &\;\kappa^2\gamma_{T+1}^2(\delta)\frac{1}{T}\sum_{t=1}^T \max_{a,\theta\in \cV_{t}}\dot\mu(\phi(a)^\top\theta)^2\|\phi(a)\|^2_{L_{t}^{-1}}\\
    =  &\; \kappa^2\gamma_{T+1}^2(\delta)\frac{1}{T}\sum_{t=1}^T \dot\mu(\phi(A_t^\MU)^\top\theta_t^\MU)^2\|\phi(A_t^\MU)\|^2_{L_{t}^{-1}}
\end{align*}
where in the third line we used definition of $A_t^\MU$ and $\theta_t^\MU$ (c.f. \cref{eq:theta-t-MU-and-At-MU}).
We then do Taylor expansion on $\dot\mu(\phi(A_t^\MU)^\top\theta_t^\MU)$ at $\dot\mu(\phi(A_t^\MU)^\top\omega_t^\MU)$. For a $\xi_t:=\xi( A_t^\MU,\theta_t^\MU,\omega^\MU)$ that is between $\phi(A_t^\MU)^\top\omega_t^\MU$ and $\phi(A_t^\MU)^\top\theta_t^\MU$
\begin{align*}
    \dot\mu(\phi(A_t^\MU)^\top\theta_t^\MU)^2 &= \left(\dot\mu(\phi(A_t^\MU)^\top\omega_t^\MU)+\ddot\mu(\xi_t)\cdot \left|\phi(A_t^\MU)^\top (\theta_t^\MU-\omega_t^{\MU})\right|\right)^2\\
    &\le \left(\dot\mu(\phi(A_t^\MU)^\top\omega_t^\MU)+\frac{1}{4}\cdot \left|\phi(A_t^\MU)^\top (\theta_t^\MU-\omega_t^{\MU})\right|\right)^2\\
    &\le 2\dot\mu(\phi(A_t^\MU)^\top\omega_t^\MU)^2+2\cdot \frac{1}{16} \left|\phi(A_t^\MU)^\top (\theta_t^\MU-\omega_t^{\MU})\right|^2,
\end{align*}
where in the second line we used that $|\ddot\mu|\le \dot\mu\le \frac{1}{4}$ (\cref{eq:self-concordant,eq:logistic-lipschitz}) and the two terms in the first line are non-negative; in the third line we used $(a+b)^2\le 2a^2+2b^2$. Hence,
\begin{align*}
    &\; \kappa^2\gamma_{T+1}^2(\delta)\frac{1}{T}\sum_{t=1}^T \dot\mu(\phi(A_t^\MU)^\top\theta_t^\MU)^2\|\phi(A_t^\MU)\|^2_{L_{t}^{-1}}\\
    \le& \; \underbrace{2\kappa^2\gamma_{T+1}^2(\delta)\frac{1}{T}\sum_{t=1}^T \sqrt{\dot\mu(\phi(A_t^\MU)^\top\omega_t^{\MU})}^2\|\phi(A_t^\MU)\|^2_{L_{t}^{-1}} }_{(iii)}\\
    + &\;\underbrace{\kappa^2\gamma_{T+1}^2(\delta)\frac{1}{T}\sum_{t=1}^T \frac{1}{2}\left|\phi(A_t^\MU)^\top (\theta_t^\MU-\omega_t^{\MU})\right|^2\|\phi(A_t^\MU)\|^2_{L_t^{-1}}}_{(iv)},
\end{align*}
where in the last line we upper bound $1/8$ by $1/2$.
For $(iii)$,
\begin{align}
    (iii)&\le 2d\kappa^2\gamma_{T+1}^2(\delta)\frac{1}{T}\sum_{t=1}^T \mE\left[ \dot\mu(\phi(A_t)^\top \bar \theta_t)^2\|\phi(A_t)\|^2_{L_{t}^{-1}}\Big\vert \cF_t\right]\notag\\
    &\le 2d\kappa^2\gamma_{T+1}^2(\delta)\frac{1}{T}\left(16N_T(\delta)+\sum_{t=1}^T \dot\mu(\phi(A_t)^\top \bar \theta_t)^2\|\phi(A_t)\|^2_{L_{t}^{-1}}\right)\notag\\
    &\le \frac{1}{8}d\kappa^3\gamma_{T+1}^2(\delta)\frac{1}{T}\sum_{t=1}^T \|\phi(A_t)\|^2_{V_t^{-1}} + \frac{32}{T}d\kappa^2\gamma_{T+1}^2(\delta)N_T(\delta)\notag\\
    &\le \frac{1}{8}d\kappa^3\gamma_{T+1}^2(\delta)\frac{1}{T}\sum_{t=1}^T \|\phi(A_t)\|^2_{V_{t}^{-1}}+ \frac{32}{T}d\kappa^2\gamma_{T+1}^2(\delta)N_T(\delta)\label{eq:reverseBern-epl-1}\\ 
    &\le \frac{1}{T}d^2\kappa^3\gamma_{T+1}^2(\delta)Y_T(\delta)+ \frac{32}{T}d\kappa^2\gamma_{T+1}^2(\delta)N_T(\delta)\label{eq:reverseBern-epl-3},
\end{align}
where in the first line we used that $\min_{a\in \cA,\theta\in \mB_d(S)}\dot\mu(\phi(a)^\top\theta)\cdot \kappa\ge 1$; in the second line we used $L_t^{-1}\preceq \kappa V_t^{-1}$; in the last line we used elliptical potential lemma (\cref{lem:epl}) to bound the second term.
For $(iv)$, we apply Cauchy-Schwarz inequality and \cref{lem:TS-conf-width} to get
\begin{align*}
    (iv)&\le \kappa^2\gamma_{T+1}^2(\delta)\frac{1}{T}\sum_{t=1}^T \frac{1}{2}\|\phi(A_t^\MU)\|^2_{L_t^{-1}}\|\theta_t^\MU - \omega_t^\MU\|^2_{L_t}\cdot \|\phi(A_t^\MU)\|^2_{L_t^{-1}}\\
    &\le \kappa^2\gamma_{T+1}^2(\delta)\frac{1}{T}\sum_{t=1}^T \frac{1}{2}\|\phi(A_t^\MU)\|^4_{L_t^{-1}}\|\theta_t^\MU - \omega_t^\MU\|^2_{H_t(\theta_*)}\\
    &\le 8\kappa^2\gamma^4_{T+1}(\delta)\frac{1}{T}\sum_{t=1}^T \|\phi(A_t^\MU)\|^4_{L_t^{-1}}.
\end{align*} 
Then we use the fact that $\min_{a\in \cA,\theta\in \mB_d(S)}\dot\mu(\phi(a)^\top\theta)\cdot\kappa\ge 1$,
\begin{align*}
    &\;8\kappa^2\gamma^4_{T+1}(\delta)\frac{1}{T}\sum_{t=1}^T \|\phi(A_t^\MU)\|^4_{L_t^{-1}}\\
    \le &\; 8\kappa^3\gamma^4_{T+1}(\delta)\frac{1}{T}\sum_{t=1}^T \dot\mu(\phi(A_t^\MU)^\top\omega_t^{\MU})\|\phi(A_t^\MU)\|^4_{L_t^{-1}}\\
    \le &\;\frac{8d\kappa^3}{\lambda_1^2}\gamma^4_{T+1}(\delta)\frac{1}{T}\sum_{t=1}^T \mE\left[\dot\mu(\phi(A_t)^\top\bar\theta_t)\|\phi(A_t)\|^2_{L_t^{-1}}\Big\vert \cF_t\right]\\
    \le &\;\frac{8d\kappa^4}{8\lambda_1^2}\gamma^4_{T+1}(\delta)\frac{1}{T}\left(16N_T(\delta)+\sum_{t=1}^T \|\phi(A_t)\|^2_{V_t^{-1}}\right)\\
    \le &\;\frac{d\kappa^4}{\lambda_1^2}\gamma^4_{T+1}(\delta) \frac{1}{T}\left(dY_T(\delta)+16N_T(\delta)\right),
\end{align*}
where in the third line we used \cref{coro:TS-to-max-uncertainty-2nd-order}; in the fourth line we used $L_t^{-1}\preceq \kappa V_t^{-1}$ and in the last line we used similar argument as (iii) (\cref{eq:reverseBern-epl-1,eq:reverseBern-epl-3}).
Putting all the bounds on $R_2(a_*)$ together, we have that
\begin{align*}
    R_2(a_*) &\le \frac{1}{T}d^2\kappa^3\gamma_{T+1}^2(\delta)Y_T(\delta)+ \frac{32}{T}d\kappa^2\gamma_{T+1}^2(\delta)N_T(\delta) \\
    &+ \frac{d\kappa^4}{\lambda_1^2}\gamma^4_{T+1}(\delta) \frac{1}{T}\left(dY_T(\delta)+16N_T(\delta)\right)
\end{align*}
\paragraph{Step 3.2: Bounding $R_2(\tilde a)$.}
Similarly, for $R_2(\tilde a)$, we can repeat the analysis:
\begin{align*}
    R_2(\tilde a)=&\ddot\mu(\xi_{ \tilde a})|\phi( \tilde a)^\top (\theta_*-\theta_{T+1}^{\Log})|^2\\
    \le &\;\frac{1}{4}\|\phi(\tilde a)\|^2_{H_{T+1}^{-1}(\theta_*)}\|\theta_*-\tilde\theta_{T+1}^{\Log}\|^2_{H_{T+1}(\theta_*)}\\
    \le &\; \frac{1}{4}\gamma_{T+1}^2(\delta)\max_a\|\phi(a)\|^2_{L_{T+1}^{-1}}\\
    \le &\;\kappa\gamma_{T+1}^2(\delta)\max_{a,\theta\in \cV_{T+1}}\dot\mu(\phi(a)^\top\theta)\|\phi(a)\|^2_{L_{T+1}^{-1}}
\end{align*}
Now this is exactly the same as what we had in \cref{eq:R_2_a_star}
So we have that 
\begin{align*}
    R_2(\tilde a) &\le \frac{1}{T}d^2\kappa^3\gamma_{T+1}^2(\delta)Y_T(\delta)+ \frac{32}{T}d\kappa^2\gamma_{T+1}^2(\delta)N_T(\delta) \\
    &+ \frac{d\kappa^4}{\lambda_1^2}\gamma^4_{T+1}(\delta) \frac{1}{T}\left(dY_T(\delta)+16N_T(\delta)\right)
\end{align*}
\paragraph{Step 4: Bounding the whole regret}
For the simple regret, putting our bounds for $R_1(a_*)$ ,$R_1(\tilde a)$, $R_2(a_*)$, $R_2(\tilde a)$ together,
\begin{align*}
    \SR_\Log(\tilde a)&\le R_1(a_*)+R_1(\tilde a)+R_2(a_*)+R_2(\tilde a)\\
    &\le \frac{2d\gamma_{T+1}(\delta)\sqrt{Y_T(\delta)}}{\sqrt{T{\kappa_*}}} + \frac{4\sqrt{\gamma_{T+1}^5(\delta)}\kappa^{3/2}\left(d^{5/3}Y_T(\delta)\right)^{3/4}}{T^{3/4}\sqrt{\kappa_*}}+ \frac{2d\gamma_{T+1}(\delta)\kappa^2}{2T\sqrt{\kappa_*}}\left(16N_T(\delta)+dY_T(\delta)\right)\\
    &+\sqrt{\SR_\Log(\tilde a)}\Bigg[\frac{d\gamma_{T+1}(\delta)\sqrt{Y_T(\delta)}}{\sqrt{T{}}}+ \gamma_{T+1}^2(\delta)\frac{d\kappa^2}{2T}\Big(N_T(\delta)+dY_T(\delta)\Big)\\
    &+ \frac{2\sqrt{\gamma_{T+1}^5(\delta)}\kappa^{3/2}\left(d^{5/3}Y_T(\delta)\right)^{3/4}}{T^{3/4}}\Bigg]\\
    &+\frac{2}{T}d^2\kappa^3\gamma_{T+1}^2(\delta)Y_T(\delta)+ \frac{64}{T}d\kappa^2\gamma_{T+1}^2(\delta)N_T(\delta) + \frac{2d\kappa^4}{T\lambda_1^2}\gamma^4_{T+1}(\delta) \left(dY_T(\delta)+16N_T(\delta)\right)
\end{align*}

We now use a fact that if for $b,c>0$,
\begin{align*}
    x^2\le bx + c,
\end{align*}
then it follows that 
\begin{align*}
    x\le b+\sqrt{c}.
\end{align*}
Take $\sqrt{\SR^\Log(\tilde a)}$ as $x$, we have that 
\begin{align*}
    \sqrt{\SR^\Log(\tilde a)} &\le \Bigg[\frac{d\gamma_{T+1}(\delta)\sqrt{Y_T(\delta)}}{\sqrt{T{}}}+ \gamma_{T+1}^2(\delta)\frac{d\kappa^2}{2T}\Big(N_T(\delta)+dY_T(\delta)\Big)\\
    &+ \frac{2\sqrt{\gamma_{T+1}^5(\delta)}\kappa^{3/2}\left(d^{5/3}Y_T(\delta)\right)^{3/4}}{T^{3/4}}\Bigg]\\
    &+\Bigg[\frac{2d\gamma_{T+1}(\delta)\sqrt{Y_T(\delta)}}{\sqrt{T{\kappa_*}}} + \frac{4\sqrt{\gamma_{T+1}^5(\delta)}\kappa^{3/2}\left(d^{5/3}Y_T(\delta)\right)^{3/4}}{T^{3/4}\sqrt{\kappa_*}}\\
    &+2\gamma_{T+1}(\delta)\frac{d\kappa^2}{2T\sqrt{\kappa_*}}\left(16N_T(\delta)+dY_T(\delta)\right)\\
    &+\frac{2}{T}d^2\kappa^3\gamma_{T+1}^2(\delta)Y_T(\delta)+ \frac{64}{T}d\kappa^2\gamma_{T+1}^2(\delta)N_T(\delta) \\
    &+ \frac{2d\kappa^4}{\lambda_1^2}\gamma^4_{T+1}(\delta) \frac{1}{T}\left(dY_T(\delta)+16N_T(\delta)\right)
    \Bigg]^{\frac{1}{2}}
\end{align*}
Square both sides and use the fact that $(a+b)^2\le 2a^2 + 2b^2$, $(a+b+c)^2\le 3a^2+3b^2+3c^2$:
\begin{align*}
    \SR^\Log(\tilde a) &\le \frac{4d\gamma_{T+1}(\delta)\sqrt{Y_T(\delta)}}{\sqrt{T{\kappa_*}}} \\
    &+ \frac{8\sqrt{\gamma_{T+1}^5(\delta)}\kappa^{3/2}\left(d^{5/3}Y_T(\delta)\right)^{3/4}}{T^{3/4}\sqrt{\kappa_*}}+\frac{4d\gamma_{T+1}(\delta)\kappa^2}{2T\sqrt{\kappa_*}}\left(16N_T(\delta)+2dY_T(\delta)\right)\\
    &+\frac{1}{T}\bigg(4d^2\kappa^3\gamma_{T+1}^2(\delta)Y_T(\delta)+ 128d\kappa^2\gamma_{T+1}^2(\delta)N_T(\delta) \bigg)\\
    &+ \frac{4d\kappa^4}{T\lambda_1^2}\gamma^4_{T+1}(\delta) \left(dY_T(\delta)+16N_T(\delta)\right)\\
    &+\frac{3d^2\gamma_{T+1}^2(\delta){Y_T(\delta)}}{{T}}+ \frac{12{\gamma_{T+1}^5(\delta)}\kappa^{3}\left(d^{5/3}Y_T(\delta)\right)^{3/2}}{T^{3/2}}\\
    &+ \frac{3d^2\gamma_{T+1}^4(\delta)\kappa^4}{4T^2}\Big(N_T(\delta)+dY_T(\delta)\Big)^2
\end{align*}


    
 

\section{Lower bound}
For mathematical coherence, we rewrite the setting in a more formal and measure theoretical way. The setting keeps unchanged.
\paragraph{Histories and recommendation policies.}
Let $(\cA,\cG)$ be a measurable action space and fix a horizon $T\in\mathbb N$.
Since rewards are Bernoulli, define for $t\in\{0,\dots,T\}$
\[
\HH_t := (\cA\times\{0,1\})^t,
\qquad
\cF_t := (\cG\otimes 2^{\{0,1\}})^{\otimes t},
\]
with the convention $\HH_0=\{\varnothing\}$.
A $T$-round recommendation policy is a sequence
\[
\pi=(\pi_1,\dots,\pi_T,\pi_{T+1}),
\]
where, for each $t\in[T]$, $\pi_t$ is a probability kernel from
$(\HH_{t-1},\cF_{t-1})$ to $(\cA,\cG)$, and $\pi_{T+1}$ is a probability
kernel from $(\HH_T,\cF_T)$ to $(\cA,\cG)$.
We write $\Pi_T$ for the set of all such policies.

For $\theta\in\Theta$, let
\[
P_\theta(\cdot\mid a)=\Ber\!\bigl(\mu(a^\top\theta)\bigr),
\qquad a\in\cA.
\]
The interaction between $\theta$ and $\pi$ induces a law $\PP_{\theta,\pi}$
on $(\HH_T,\cF_T)$. If for $\bH_t=(A_1,X_1,\dotsc,A_T,X_T)$, $\hat A_T\sim \pi_{T+1}(a\vert \bH_t)$ denotes the
final recommendation, then the logistic simple regret is
\begin{align*}
    \SR_T^{\Log}(\theta,\pi)
=
\sum_{a\in \cA}\pi_{T+1}(a\vert \bH_T)\left(
    \mu(a_*(\theta)^\top\theta)
    -
    \mu(a^\top\theta)
\right),
\end{align*}
where
\[
a_*(\theta)\in\arg\max_{a\in\cA} a^\top\theta.
\]


\begin{definition}[Shifted saturated hypercube class]
\label{def:ssh-class}
Fix $d\ge 2$, $\kappa_0\ge 4$, and $T\ge \kappa_0(d-1)^2$.
Let $m=m(\kappa_0)\ge 0$ satisfy
\[
\mu'(m)=\frac{1}{\kappa_0},
\]
and let
\[
\varepsilon
:=
\frac{d-1}{8e^3}\sqrt{\frac{\kappa_0}{T}}.
\]
Define
\[
\cA_{d,T,\kappa_0} 
:=
\left\{
a_u:=\frac{1}{\sqrt 2}\Bigl(1,\frac{u}{\sqrt{d-1}}\Bigr)
:\;
u\in\{-1,1\}^{d-1}
\right\}\subset\mathbb R^d
\]
and
\[
\Theta_{d,T,\kappa_0} 
:=
\left\{
\theta_v:=\sqrt 2\Bigl(m-\varepsilon,\frac{\varepsilon}{\sqrt{d-1}}\,v\Bigr)
:\;
v\in\{-1,1\}^{d-1}
\right\}\subset\mathbb R^d.
\]
\end{definition}

\begin{lemma}[Geometry of the shifted saturated hypercube]
\label{lem:ssh-geometry}
For the class in Definition~\ref{def:ssh-class}, the following hold:
\begin{enumerate}
\item $\cA_{d,T,\kappa_0}\subset \mathbb B_d(1)$;
\item $\Theta_{d,T,\kappa_0}\subset \mathbb B_d(S)$ with
\[
S:=\sqrt{2(m^2+2)};
\]
\item for every $\theta\in\Theta_{d,T,\kappa_0} $,
\[
\kappa_*(\theta)=\kappa_0.
\]
\end{enumerate}
\end{lemma}
\begin{proof}
For every $u\in\{-1,1\}^{d-1}$,
\[
\|a_u\|_2^2
=
\frac12\left(1+\frac{\|u\|_2^2}{d-1}\right)
=
\frac12(1+1)=1,
\]
so $\cA\subset \mathbb B_d(1)$.

Likewise, for every $v\in\{-1,1\}^{d-1}$,
\[
\|\theta_v\|_2^2
=
2\left((m-\varepsilon)^2+\varepsilon^2\right)
\le
2(m^2+2),
\]
because $\varepsilon\le 1$ under the condition $T\ge \kappa_0(d-1)^2$. Hence $\Theta_{\varepsilon,m}\subset \mathbb B_d(S)$ with $S=\sqrt{2(m^2+2)}$.

Now let $H(u,v)$ be the Hamming distance between $u$ and $v$. Since
\[
u^\top v=(d-1)-2H(u,v),
\]
we have
\begin{align*}
a_u^\top\theta_v
&=
\Bigl(1,\frac{u}{\sqrt{d-1}}\Bigr)
\cdot
\Bigl(m-\varepsilon,\frac{\varepsilon}{\sqrt{d-1}}v\Bigr)
\\
&=
m-\varepsilon+\frac{\varepsilon}{d-1}u^\top v
\\
&=
m-\frac{2\varepsilon}{d-1}H(u,v).
\end{align*}
Therefore
\begin{align}
a_v^\top\theta_v=m,
\qquad
a_u^\top\theta_v=m-\frac{2\varepsilon}{d-1}H(u,v)\le m,\label{eq:lb_linear_gap}
\end{align}
with equality if and only if $u=v$. Since $\mu$ is strictly increasing, it follows that $a_*(\theta_v)=a_v$. Moreover,
\[
\kappa_*(\theta_v)
=
\frac{1}{\mu'(a_*(\theta_v)^\top\theta_v)}
=
\frac{1}{\mu'(m)}
=
\kappa_0.
\]
\end{proof}
\begin{theorem}[Shifted saturated hypercube lower bound]
\label{thm:ssh-lower-bound}
Let $d\ge 2$, $\kappa_0\ge 4$, and $T\ge \kappa_0(d-1)^2$.
Let $(\cA_{d,T,\kappa_0} ,\Theta_{d,T,\kappa_0} )$
be the shifted saturated hypercube class from
Definition~\ref{def:ssh-class}.
Then there exists a universal constant $c>0$ such that
\[
\inf_{\pi\in\Pi_T}
\sup_{\theta\in\Theta_{d,T,\kappa_0} }
\EE_{\theta,\pi}\left[\SR_T^{\Log}(\theta,\pi)\right]
\ge
c\,\frac{d}{\sqrt{\kappa_0 T}}.
\]
More precisely, one may take
\[
c=\frac{1}{32e^5}.
\]
\end{theorem}
\begin{proof}
Fix a $T\ge \kappa_0(d-1)^2$ and a T-round recommendation policy $\pi$.
To reduce clutter, we abbreviate $\PP_{\theta_v,\pi}, \EE_{\theta_v,\pi}$ as $\PP_v,\EE_v$. Let $\bar\PP_{v}:=\PP_{v}\otimes \pi_{T+1}$ and $\bar \EE_{v}$ be the expectation under $\bar \PP_{v}$ for all $v\in \{\pm 1\}^d$. Let $\cA:=\cA_{d,T,\kappa_0}$, $\Theta:=\Theta_{d,T,\kappa_0}$.

\paragraph{Step 1: Lower bound the regret by the Hamming error.}
Conditioned on $\bH_T$, let
\(
\hat U_T\sim \pi(\cdot \vert \bH_T),
\)
and the corresponding sampled final arm
\[
a_{\hat U_T}:=\phi({\hat U_T})\in \cA.
\]
In words, $\hat A_T$ is obtained by sampling one arm from the
recommendation distribution reported by the original policy.

We first rewrite the distribution-valued regret as the ordinary arm-valued
simple regret of the sampled recommendation $\hat A_T$.
By the tower property,
\begin{align}
\bar s_T^\pi(\theta_v):=\mE[\SR(\theta_v,\pi)]
&=
\EE_v\!\left[
\sum_{u\in\{-1,1\}^{d-1}}
\pi_T(a_u)
\Bigl(
\mu(a_v^\top\theta_v)-\mu(a_u^\top\theta_v)
\Bigr)
\right]
\\
&=
\mathbb E_v\!\left[
\mathbb E\!\left[
\mu(a_v^\top\theta_v)-\mu(a_{\hat U_T}^\top\theta_v)
\,\middle|\, \bH_T
\right]
\right]
\\
&=
\bar{\mathbb E}_v\!\left[
\mu(a_v^\top\theta_v)-\mu(a_{\hat U_T}^\top\theta_v)
\right].
\label{eq:distribution-regret-equals-sampled-arm-regret}
\end{align}
Using \cref{eq:lb_linear_gap},
\[
a_v^\top\theta_v-a_{\hat U_T}^\top\theta_v
=
m-\left(m-\frac{2\varepsilon}{d-1}H(\hat U_T,v)\right)
=
\frac{2\varepsilon}{d-1}H( \hat U_T,v).
\]
Hence
\[
\bar s_T^\pi(\theta_v)
=
\bar{\mathbb{E}}_v\!\left[
\mu(m)-\mu\!\left(m-\frac{2\varepsilon}{d-1}H( \hat U_T,v)\right)
\right].
\]
Fix $h\in\{0,\dots,d-1\}$. By the mean-value theorem, for some
\[
\xi\in\left[m-\frac{2\varepsilon h}{d-1},m\right]\subset[m-2\varepsilon,m],
\]
we have
\[
\mu(m)-\mu\!\left(m-\frac{2\varepsilon h}{d-1}\right)
=
\mu'(\xi)\,\frac{2\varepsilon h}{d-1}.
\]
Applying \cref{lem:mu_dot_relate} with $z_1=\xi$ and $z_2=m$, and using $|\xi-m|\le 2\varepsilon$, yields
\[
\mu'(\xi)
\ge
e^{-2\varepsilon}\mu'(m)
=
\frac{e^{-2\varepsilon}}{\kappa_0}.
\]
Therefore,
\[
\mu(m)-\mu\!\left(m-\frac{2\varepsilon h}{d-1}\right)
\ge
\frac{2\varepsilon e^{-2\varepsilon}}{\kappa_0(d-1)}\,h.
\]
Taking $h=H( \hat U_T,v)$ and then expectation gives
\begin{equation}
\label{eq:regret-hamming}
\bar s_T^\pi(\theta_v)
\ge
\frac{2\varepsilon e^{-2\varepsilon}}{\kappa_0(d-1)}
\bar{\mathbb{E}}_v\bigl[H( \hat U_T,v)\bigr].
\end{equation}

\paragraph{Step 2: Pairwise KL bound for neighboring sign vectors.}
For each $i\in\{1,\dots,d-1\}$, let $v^{(i)}$ denote the sign vector obtained from $v$ by flipping coordinate $i$. Define
\[
\alpha
:=
\sup_{v\in\{-1,1\}^{d-1}}
\KL(\bar {\mathbb{P}}_v,\bar {\mathbb{P}}_{v^{(i)}}).
\]
Note that 
\begin{align}
D_{\mathrm{KL}}(\bar\PP_{v},\bar\PP_{{v^{(i)}}})
&=
D_{\mathrm{KL}}(\PP_v\otimes\pi_{T+1}\|\PP_{v^{(i)}}\otimes\pi_{T+1})
\\
&=
D_{\mathrm{KL}}(\PP_v\|\PP_{v^{(i)}})
+
\EE_v\!\left[
D_{\mathrm{KL}}\bigl(\pi_{T+1}(\cdot\vert \bH_T)\|\pi_{T+1}(\cdot\vert \bH_T)\bigr)
\right]
\\
&=
D_{\mathrm{KL}}(\PP_v\|\PP_{v^{(i)}}).
\label{eq:ssh-terminal-kernel-adds-no-kl}
\end{align}
Hence $\alpha = D_{\mathrm{KL}}(\PP_v\|\PP_{v^{(i)}})$.
We now upper bound $\alpha$.

Fix $v$ and $i$. For any arm $a_u\in \cA$,
\begin{align*}
a_u^\top\theta_v-a_u^\top\theta_{v^{(i)}}
&=
\frac{\varepsilon}{d-1}\bigl(u^\top v-u^\top v^{(i)}\bigr)
\\
&=
\frac{2\varepsilon}{d-1}u_iv_i.
\end{align*}
Hence
\[
\bigl|a_u^\top\theta_v-a_u^\top\theta_{v^{(i)}}\bigr|
=
\frac{2\varepsilon}{d-1}.
\]
We would like to bound the KL-divergence between $\mathbb{P}_v$ and $\mathbb{P}_{v^{(i)}}$ where we need the following tools.
\begin{proposition}[Chain rule of KL, exercise 14.12 of \cite{Lattimore_Szepesvari_2020}]\label{prop:kl_chain_rule}
    Let $P$ and $Q$ be measures on 
    $\left(\mathbb{R}^n, \mathfrak{B}(\mathbb{R}^n)\right)$, and for 
    $t \in [n]$, let $X_t(x) = x_t$ be the coordinate project from 
    $\mathbb{R}^n \to \mathbb{R}$. Then let $P_t$ and $Q_t$ be regular 
    versions of $X_t$ given $X_1,\ldots,X_{t-1}$ under $P$ and $Q$, 
    respectively. Show that
    \[
    \KL(P,Q)
    =
    \sum_{t=1}^n 
    \mathbb{E}_P
    \left[
    \KL\left(
    P_t(\cdot \mid X_1,\ldots,X_{t-1}),
    Q_t(\cdot \mid X_1,\ldots,X_{t-1})
    \right)
    \right].
    \]
\end{proposition}
\begin{proposition}
    For $0<p,q< 1$, the KL divergence between $\Ber(p)$ and $Ber(q)$ satisfies
    \begin{align*}
        \KL(\Ber(p),\Ber(q)) = p\log\left(\frac{p}{q}\right) + (1-p)\log\left(\frac{1-p}{1-q}\right)
    \end{align*}
\end{proposition}
\begin{lemma}\label{lem:kl_logistic}
    For $x,y\in \RR$, it follows that 
    \begin{align*}
        \KL(\Ber(\mu(x)),\Ber(\mu(y))) \le \frac{(\mu(x)-\mu(y))^2}{\dot\mu(y)}
    \end{align*}
\end{lemma}
\begin{proof}
    Let \(p:=\mu(x)\) and \(q:=\mu(y)\). Then
\[
D_{\mathrm{KL}}(\mathrm{Ber}(p),\mathrm{Ber}(q))
=
p\log\frac{p}{q}+(1-p)\log\frac{1-p}{1-q}.
\]
Rewrite this as
\[
D_{\mathrm{KL}}(\mathrm{Ber}(p),\mathrm{Ber}(q))
=
p\log\left(1+\frac{p-q}{q}\right)
+
(1-p)\log\left(1-\frac{p-q}{1-q}\right).
\]
Using \(\log(1+t)\le t\) for all \(t>-1\), we obtain
\[
\log\left(1+\frac{p-q}{q}\right)\le \frac{p-q}{q},
\qquad
\log\left(1-\frac{p-q}{1-q}\right)\le -\frac{p-q}{1-q}.
\]
Therefore
\begin{align*}
    D_{\mathrm{KL}}(\mathrm{Ber}(p),\mathrm{Ber}(q))
&\le
p\frac{p-q}{q}
-
(1-p)\frac{p-q}{1-q}\\
&=
(p-q)\left(\frac{p}{q}-\frac{1-p}{1-q}\right)\\
&=
(p-q)\frac{p(1-q)-q(1-p)}{q(1-q)}\\
&=
\frac{(p-q)^2}{q(1-q)}.
\end{align*}
Substituting back \(p=\mu(x)\) and \(q=\mu(y)\) gives
\[
D_{\mathrm{KL}}(\mathrm{Ber}(\mu(x)),\mathrm{Ber}(\mu(y)))
\le
\frac{(\mu(x)-\mu(y))^2}{\mu(y)(1-\mu(y))}.
\]
Since for sigmoid
\[
\dot{\mu}(y)=\mu(y)(1-\mu(y)),
\]
we conclude that
\[
D_{\mathrm{KL}}(\mathrm{Ber}(\mu(x)),\mathrm{Ber}(\mu(y)))
\le
\frac{(\mu(x)-\mu(y))^2}{\dot{\mu}(y)}.
\]
\end{proof}

\begin{corollary}\label{eq:rel-ent-decomp}
    It follows that 
    \begin{align*}
        \KL(\mathbb{P}_{\theta,\pi},\mathbb{P}_{\theta',\pi})
        &\le
        \mathbb{E}_\theta\!\left[
        \sum_{t=1}^T
        \frac{\bigl(\mu(\phi(A_t)^\top\theta)-\mu(\phi(A_t)^\top\theta')\bigr)^2}{\mu'(\phi(A_t)^\top\theta')}
        \right],
    \end{align*}
\end{corollary}
\begin{proof}
    Putting \cref{lem:kl_logistic} and \cref{prop:kl_chain_rule} together proves the claim.
\end{proof}
Now apply \cref{eq:rel-ent-decomp} to $\mathbb{P}_v$ and $\mathbb{P}_{v^{(i)}}$:
\[
\KL(\mathbb{P}_v,\mathbb{P}_{v^{(i)}})
\le
\mathbb{E}_v\!\left[
\sum_{t=1}^T
\frac{\bigl(\mu(I_t)-\mu(J_t)\bigr)^2}{\mu'(J_t)}
\right],
\]
where
\[
I_t:=\phi(A_t)^\top\theta_v,
\qquad
J_t:=\phi(A_t)^\top\theta_{v^{(i)}}.
\]
By Step~1, both $I_t$ and $J_t$ lie in the interval $[m-2\varepsilon,m]$. By the mean-value theorem, for some $\xi_t$ between $I_t$ and $J _t$,
\[
\mu(I_t)-\mu(J_t)=\mu'(\xi_t)(I_t-J_t).
\]
Therefore
\[
\frac{\bigl(\mu(I_t)-\mu(J_t)\bigr)^2}{\mu'(J_t)}
=
\frac{\mu'(\xi_t)^2}{\mu'(J_t)}(I_t-J_t)^2.
\]
Using \cref{lem:mu_dot_relate} on the interval $[m-2\varepsilon,m]$, we obtain
\[
\mu'(\xi_t)
\le
e^{2\varepsilon}\mu'(m)
=
\frac{e^{2\varepsilon}}{\kappa_0},
\qquad
\mu'(J_t)
\ge
e^{-2\varepsilon}\mu'(m)
=
\frac{e^{-2\varepsilon}}{\kappa_0}.
\]
Hence
\[
\frac{\mu'(\xi_t)^2}{\mu'(J_t)}
\le
\frac{(e^{2\varepsilon}/\kappa_0)^2}{e^{-2\varepsilon}/\kappa_0}
=
\frac{e^{6\varepsilon}}{\kappa_0}.
\]
Since $|I_t-J_t|=2\varepsilon/(d-1)$,
\[
\frac{\bigl(\mu(I_t)-\mu(J_t)\bigr)^2}{\mu'(J_t)}
\le
\frac{e^{6\varepsilon}}{\kappa_0}\left(\frac{2\varepsilon}{d-1}\right)^2
=
\frac{4e^{6\varepsilon}\varepsilon^2}{\kappa_0(d-1)^2}.
\]
Summing over $t$ gives
\begin{equation}
\label{eq:pairwise-kl}
\KL(\mathbb{P}_v,\mathbb{P}_{v^{(i)}})
\le
\frac{4e^{6\varepsilon}T\varepsilon^2}{\kappa_0(d-1)^2}.
\end{equation}
Thus
\[
\alpha
\le
\frac{4e^{6\varepsilon}T\varepsilon^2}{\kappa_0(d-1)^2}.
\]

\paragraph{Step 3: Lower bound the average Hamming error.}
For each coordinate $i\in\{1,\dots,d-1\}$, define
\[
R_i
:=
\frac{1}{2^{d-1}}
\sum_{v\in\{-1,1\}^{d-1}}
\bar{\mathbb{P}}_v\bigl(\hat U_{T,i}\neq v_i\bigr).
\]
Then
\[
\frac{1}{2^{d-1}}
\sum_{v\in\{-1,1\}^{d-1}}
\bar{\mathbb{E}}_v\bigl[H( \hat U_T,v)\bigr]
=
\sum_{i=1}^{d-1} R_i.
\]

Fix $i$. Pair each $v$ with $v_i=1$ to its neighbor $v^{(i)}$. Then
\[
R_i
=
\frac{1}{2^{d-1}}
\sum_{v:\,v_i=1}
\left[
\bar{\mathbb{P}}_v(\hat U_{T,i}=-1)
+
\bar{\mathbb{P}}_{v^{(i)}}(\hat U_{T,i}=1)
\right].
\]
For any two distributions $P,Q$ and any event $A$,
\[
P(A)+Q(A^c)\ge 1-\operatorname{TV}(P,Q),
\]
so here
\[
R_i
\ge
\frac{1}{2^{d-1}}
\sum_{v:\,v_i=1}
\left[
1-\operatorname{TV}(\bar{\mathbb{P}}_v,\bar{\mathbb{P}}_{v^{(i)}})
\right].
\]
Pinsker's inequality gives
\[
\operatorname{TV}(P,Q)
\le
\sqrt{\frac12\KL(P,Q)}.
\]
Since there are $2^{d-2}$ vectors with $v_i=1$, we obtain
\[
R_i
\ge
\frac12-\sqrt{\frac{\alpha}{8}}.
\]
Summing over $i$ yields
\begin{equation}
\label{eq:avg-hamming}
\frac{1}{2^{d-1}}
\sum_{v\in\{-1,1\}^{d-1}}
\bar{\mathbb{E}}_v\bigl[H( \hat U_T,v)\bigr]
\ge
(d-1)\left(\frac12-\sqrt{\frac{\alpha}{8}}\right).
\end{equation}

\paragraph{Step 4: Choose $\varepsilon$ and conclude.}
By definition,
\[
\varepsilon
=
\frac{d-1}{8e^3}\sqrt{\frac{\kappa_0}{T}}.
\]
Since $T\ge \kappa_0(d-1)^2$, we have $\varepsilon\le 1/(8e^3)<1$. Therefore $e^{-2\varepsilon}\ge e^{-2}$ and $e^{6\varepsilon}\le e^6$. Plugging the chosen value of $\varepsilon$ into \cref{eq:pairwise-kl} gives
\[
\alpha
\le
\frac{4e^6T}{\kappa_0(d-1)^2}
\cdot
\frac{(d-1)^2\kappa_0}{64e^6T}
=
\frac{1}{16}.
\]
Hence \cref{eq:avg-hamming} implies
\[
\frac{1}{2^{d-1}}
\sum_{v\in\{-1,1\}^{d-1}}
\bar{\mathbb{E}}_v\bigl[H( \hat U_T,v)\bigr]
\ge
(d-1)\left(\frac12-\frac{1}{\sqrt{128}}\right)
\ge
\frac{d-1}{4}.
\]
Averaging \cref{eq:regret-hamming} over $v$ therefore gives
\begin{align*}
\frac{1}{2^{d-1}}\sum_{v\in\{-1,1\}^{d-1}} \bar s_T^\pi(\theta_v,\pi)
&\ge
\frac{2\varepsilon e^{-2\varepsilon}}{\kappa_0(d-1)}
\cdot
\frac{d-1}{4}
\\
&\ge
\frac{2\varepsilon e^{-2}}{\kappa_0(d-1)}
\cdot
\frac{d-1}{4}
\\
&=
\frac{e^{-2}}{2}\cdot \frac{\varepsilon}{\kappa_0}
\\
&=
\frac{e^{-2}}{2}\cdot
\frac{1}{8e^3}
\frac{d-1}{\sqrt{\kappa_0T}}
\\
&=
\frac{1}{16e^5}\,\frac{d-1}{\sqrt{\kappa_0T}}.
\end{align*}

Finally, since the supremum dominates the average,
\[
\inf_{\pi}\sup_{\theta_*\in \Theta} \EE[\SR^\Log(\theta_*,\pi)]
\ge
\frac{1}{16e^5}\,\frac{d-1}{\sqrt{\kappa_0T}}.
\]
For $d\ge 2$, we have $d-1\ge d/2$, and therefore
\[
\inf_{\pi}\sup_{\theta_*\in \Theta} \EE[\SR^\Log(\theta_*,\pi)]
\ge
\frac{1}{32e^5}\,\frac{d}{\sqrt{\kappa_0T}}.
\]
This completes the proof.
\end{proof}


\section{Technical Lemmas}\label{appendix:technical_lemmas}
In \cref{thm:Bernstein-inequality-for-martingales,thm:Reversed-Bernstein-inequality-for-martingales}, we provided a corrected version of Thm. 3 in \cite{zanette2021design}.
The reason why it's flawed is: the optimization over $\lambda$ in eq 133 and 134 depends on a random variable, therefore one cannot "choose" without the knowledge of the random quantity. We rectify the situation by forming a geometric cover over possible values for $\lambda$, which solved the issue but introduced a second logarithmic term. 
\begin{theorem}[Bernstein's inequality for Martingales]
    \label{thm:Bernstein-inequality-for-martingales}
    Consider the stochastic process $\{X_t\}$ adapted to the
    filteration $\{\cF_t\}$. Assume $X_t\leq 1$ almost surely,
    and $\mE[X_t| \cF_{t-1}] = 0$.
    Then
    \begin{align}
        \label{eq:Bernstein-Martingales-p1}
        \forall \lambda \in (0,1],
        \qquad
        P\left(
            \sum_{t=1}^T X_t
            \leq \lambda \sum_{t=1}^T \mE[X_t^2|\cF_{t-1}]
            + \frac{1}{\lambda}
            \log{\frac{1}{\delta}}
            \right)
            \geq 1 - \delta,
    \end{align}
    which implies
    \begin{align}
        \label{eq:Bernstein-Martingales-p2}
        P\left(
            \sum_{t=1}^T X_t
            \leq
            3 \sqrt{
            \left(
                \sum_{t=1}^T \mE[X_t^2|\cF_{t-1}]
            \right)
            \log({\frac{\lg(\sqrt{T})}{\delta}})
            }
            + 2 \log({\frac{\lg(\sqrt{T})}{\delta}})
        \right)
        \geq 1 - \delta.
    \end{align}
\end{theorem}
\begin{proof}
    Define the random variable $M_t$ as
    \begin{align}
        M_t = M_{t-1} \exp(
                \lambda X_t
                - \lambda^2 \mE[X_t^2|\cF_{t-1}]
            ),
    \end{align}
    where in particular $M_0 = 1$,
    and $\mE[\cdot|\cF_0]=\mE[\cdot]$
    so $M_t$ is $\cF_t$-measurable.
    Recall the inequalities $e^x \leq 1 + x + x^2$
    for $x \leq 1$ and $1 + x \leq e^x$:
    \begin{align}
        \mE[M_t| \cF_{t-1}] &= M_{t-1} \mE\left[
            \exp(
                \lambda X_t
                - \lambda^2 \mE[X_t^2|\cF_{t-1}]
            ) | \cF_{t-1}
        \right]\\
        &\leq M_{t-1} \mE\left[
                1 + \lambda X_t
                + \lambda^2 X_t^2
            | \cF_{t-1}
        \right]
        \exp(-\lambda^2 \mE[X_t^2|\cF_{t-1}])
        \\
        &\leq
        M_{t-1}
        \left(
            1
            + \lambda \mE[X_t|\cF_{t-1}]
            + \lambda^2 
            \mE\left[
                X_t^2
            | \cF_{t-1}
        \right]
        \right)
        \exp(-\lambda^2 \mE[X_t^2|\cF_{t-1}])
        \\
        &\leq M_{t-1}
        \exp(\lambda^2 \mE[X_t^2|\cF_{t-1}])
        \exp(-\lambda^2 \mE[X_t^2|\cF_{t-1}])
        \\
        &= M_{t-1}.
    \end{align}
Thus, $\{M_t\}$ is a supermartingale adapted to $\{\cF_t\}$.
In particular $\mE[M_t| \cF_{t-1}] \leq M_0 = 1$.
Then by the Markov inequality:
\begin{align}
    \label{eq:Bernstein-markov-1}
    P \left(
        \underbrace{
            \lambda \sum_{t=1}^{T} X_t
            - \lambda^2
            \sum_{t=1}^{T}
            \mE[X_t^2|\cF_{t-1}]
        }_{\log(M_t)}
    > \log \frac{1}{\delta}
    \right)
    = P \left(
        M_t > \frac{1}{\delta}
    \right)
    \leq \frac{
        \mE\left[\mE[M_t|\cF_{t-1}]\right]
    }{\frac{1}{\delta}}
    \leq \delta,
\end{align}
which proves \cref{eq:Bernstein-Martingales-p1}.

Next, to prove \cref{eq:Bernstein-Martingales-p2}
define the sequence
$N(l) := \{\lambda_i = l 2^i\}_{i=0}^{\floor{\lg(1/l)}}
\cup \{1\}$
for a value $l \leq 1$ chosen later, and
\begin{align}
\hat{\lambda}=
\sqrt{\frac{\log(\frac{\floor{\lg(1/l)} + 1}{\delta})}
{\sum_{t=1}^{T}\mE[X_t^2|\cF_{t-1}]}}.
\end{align}
Also, by using \cref{eq:Bernstein-markov-1} and
a union bound over the $\floor{\lg(1/l)} + 1$ points
in $N(l)$ we get:
\begin{align}
        \label{eq:union-bound-1}
        P \left(
        \forall \lambda \in N(l): \quad
        \lambda \sum_{t=1}^{T} X_t
        - \lambda^2
        \sum_{t=1}^{T}
        \mE[X_t^2|\cF_{t-1}]
        \leq \log \frac{\floor{\lg(\frac{1}{l})} + 1}{\delta}
        \right)
        \geq 1 - \delta.
\end{align}
Firstly, if $1 \leq \hat{\lambda}$, which also means
$\sum_{t=t}^{T}\mE[X_t^2|\cF_{t-1}]
\leq \log(\frac{\floor{\lg(1/l)} + 1}{\delta})$,
for value $\lambda =1 \in N(l)$ in \cref{eq:union-bound-1} we get:
\begin{align}
    \label{eq:Bernstein-case1}
    \sum_{t=1}^{T} X_t
    &\leq
    \sum_{t=1}^{T}
    \mE[X_t^2|\cF_{t-1}] + \log(\frac{\floor{\lg(1/l)} + 1}{\delta})
    \leq
    2\log(\frac{\floor{\lg(1/l)} + 1}{\delta}).
\end{align}
Secondly, for $\hat{\lambda}<1$,
according $N(l)$'s construction one of the two cases below holds:
\begin{align}
    \label{eq:tilde-lambda}
    \hat{\lambda} < l \qquad \text{or} \qquad
    \exists \tilde{\lambda} \in N(l)\,\,\,\text{st.}\,
    \tilde{\lambda} \leq \hat{\lambda}
    \leq 2\tilde{\lambda}.
\end{align}
For $\hat{\lambda} \geq l$
by \cref{eq:union-bound-1} and $\tilde{\lambda}$
defined in \cref{eq:tilde-lambda}
 we have
\begin{align}
    \sum_{t=1}^T X_t
    &\leq
    \tilde{\lambda}
    \sum_{t=1}^T \mE[X_t^2|\cF_{t-1}]
    + \frac{1}{\tilde{\lambda}}
    \log(\frac{\floor{\lg(1/l)} + 1}{\delta})\\
    &\leq
    \hat{\lambda}
    \sum_{t=1}^T \mE[X_t^2|\cF_{t-1}]
    + \frac{2}{\hat{\lambda}}
    \log(\frac{\floor{\lg(1/l)} + 1}{\delta})\\
    &\leq
    3 \sqrt{\left(\sum_{t=1}^{T} \mE[X_t^2|\cF_{t-1}]\right)
    \log(\frac{\floor{\lg(1/l)} + 1}{\delta})
    }
\end{align}
For $\hat{\lambda} < l$,
which means $
\log(\frac{\floor{\lg(1/l)} + 1}{\delta})
< l^2\sum_{t=1}^{T} \mE[X_t^2|\cF_{t-1}]
$
, by \cref{eq:union-bound-1} we have
\begin{align}
    \sum_{t=1}^T X_t
    &\leq
    l
    \sum_{t=1}^T \mE[X_t^2|\cF_{t-1}]
    + \frac{1}{l}
    \log(\frac{\floor{\lg(1/l)} + 1}{\delta})\\
    \sum_{t=1}^T X_t
    &\leq
    2l
    \sum_{t=1}^T \mE[X_t^2|\cF_{t-1}]
\end{align}
Finally, by setting $l = \frac{1}{\sqrt{T}}$,
the fact that $\sum_{t=1}^{T} \mE[X_t^2|\cF_{t-1}] \leq T$, and
summing up with RHS of \cref{eq:Bernstein-case1} to cover
both cases we get:
\begin{align}
    P\left(
    \sum_{t=1}^{T} X_t \leq
    3\sqrt{\left(\sum_{t=1}^{T} \mE[X_t^2|\cF_{t-1}]\right)
    \log(\frac{\lg(\sqrt{T})}{\delta}}) +
    2 \log(\frac{\lg(\sqrt{T})}{\delta})
    \right) \geq 1 - \delta,
\end{align}
which proves the second part of the thesis.
\end{proof}
\begin{theorem}[Reversed Bernstein's inequality for Martingales]
    \label{thm:Reversed-Bernstein-inequality-for-martingales}
    Let $\{X_t\}$ be a stochastic process
    adapted to the filteration $\{\cF_t\}$.
    Assuming $0 \leq X_t \leq 1$ almost surely,
    then it holds that:
    \begin{align}
        P\left(
            \sum_{t=1}^{T} \mE[X_t|\cF_{t-1}]
            \leq \frac{1}{4}
            \left(
                c_1 +
                \sqrt{
                    c_1^2
                    + 4
                    \left(
                        \sum_{t=1}^{T} X_t + c_2
                    \right)
                }
            \right)^2
        \right) &\geq 1 - \delta, \label{eq:reversed-inequality-theorem}\\
        c_1 = 3 \sqrt{\log \frac{\lg(\sqrt{T})}{\delta}},
        c_2 = 2 \log \frac{\lg(\sqrt{T})}{\delta}.&
    \end{align}
\end{theorem}
\begin{proof}
Consider the random noise
\begin{align}
    \xi_t := \mE[X_t|\cF_{t-1}] - X_t,
\end{align}
which allows us to write
\begin{align}
    \label{eq:Reversed-p1}
    \sum_{i=1}^{t} \mE[X_i|\cF_{i-1}] =
    \sum_{i=1}^{t} (\xi_i + X_i).
\end{align}
Then the \cref{thm:Bernstein-inequality-for-martingales}
ensures the following statement:
\begin{align}
    \label{eq:Reversed-p2}
    P \left(
        \sum_{t=1}^{T} \xi_t
        \leq
        c_1
        \sqrt{
            \sum_{t=1}^{T}
            \mE[\xi_t^2|\cF_{t-1}]
        }
        + c_2
    \right) \geq 1 - \delta.
\end{align}
Notice that since $0 \leq X_t \leq 1$ almost surely,
we have
\begin{align}
    \mE[\xi_t^2|\cF_{t-1}] &= \mE[
        \left(
            X_t - \mE[X_t|\cF_{t-1}]
        \right)^2
        |\cF_{t-1}]\\
        &=\mE[X_t^2|\cF_{t-1}]
        - \mE[X_t|\cF_{t-1}]^2\\
        &\leq \mE[X_t^2|\cF_{t-1}]\\
        &\leq \mE[X_t|\cF_{t-1}].
\end{align}
Plugging back into \cref{eq:Reversed-p2}
and using \cref{eq:Reversed-p1} gives
\begin{align}
    P \left(
        \sum_{t=1}^{T} \xi_t
        = \sum_{t=1}^{T} (\mE[X_t|\cF_{t-1}] - X_t)
        \leq c_1
        \sqrt{
            \sum_{t=1}^{T}
            \mE[X_t|\cF_{t-1}]
        }
        + c_2
    \right) \geq 1 - \delta
\end{align}
or equivalently
\begin{align}
    P \left(
        \sum_{t=1}^{T} \mE[X_t|\cF_{t-1}]
        \leq
        \sum_{t=1}^{T} X_t
        + c_1 \sqrt{
            \sum_{t=1}^{T}
            \mE[X_t|\cF_{t-1}]
        }
        + c_2
    \right) \geq 1 - \delta.
\end{align}
Solving for $\sum_{t=1}^{T} \mE[X_t|\cF_{t-1}]$
gives under such event
\begin{align}
    \sum_{t=1}^{T} \mE[X_t|\cF_{t-1}]
            \leq \frac{1}{4}
            \left(
                c_1 +
                \sqrt{
                    c_1^2
                    + 4
                    \left(
                        \sum_{t=1}^{T} X_t + c_2
                    \right)
                }
            \right)^2.
\end{align}
\end{proof}
The following is a corollary to the above theorem, which is
easier to use.
\begin{corollary}
    \label{cor:Reversed-Bernstein-inequality-for-martingales}
    Under the same assumptions of
    \cref{thm:Reversed-Bernstein-inequality-for-martingales},
    with probability at least $1 - \delta$ we have
    \begin{align}
        \sum_{t=1}^{T} \mE[X_t|\cF_{t-1}]
        \leq \left(
        4\sqrt{\log(\frac{2\log T}{\delta})}
        + \sqrt{\sum_{t=1}^{T} X_t}
        \right)^2.
    \end{align}
\end{corollary}
\begin{proof}
    Starting from the right hand side of the inequality
    in \cref{eq:reversed-inequality-theorem}, using the fact that $
    \sqrt{a + b} \leq \sqrt{a} + \sqrt{b}$ for $a,b \geq 0
    $ we have:
    \begin{align}
        \frac{1}{4}
        \left(
            c_1 +
            \sqrt{
                c_1^2
                + 4
                \left(
                    \sum_{t=1}^{T} X_t + c_2
                \right)
            }
        \right)^2
        &\leq
        \frac{1}{4}
        \left(
        6\sqrt{\log\left(\frac{\lg \sqrt T}{\delta}\right)}
        + 2 \sqrt{\sum_{t=1}^{T} X_t}
        + 2 \sqrt{c_2}
        \right)^2\\
        &\leq
        \frac{1}{4}
        \left(
        6\sqrt{\log\left(\frac{\lg \sqrt T}{\delta}\right)}
        + 2 \sqrt{\sum_{t=1}^{T} X_t}
        + 2 \sqrt{2 \log\left(\frac{\lg \sqrt T}{\delta}\right)}
        \right)^2\\
        &\leq
        \left(
        4\sqrt{\log\left(\frac{\lg \sqrt T}{\delta}\right)}
        + \sqrt{\sum_{t=1}^{T} X_t}
        \right)^2\\
        &\le \left(
            4\sqrt{\log\left(\frac{\frac{1}{2}\log T}{\log 2\cdot \delta}\right)}
            + \sqrt{\sum_{t=1}^{T} X_t}
            \right)^2\\
        &\le \left(
            4\sqrt{\log\left(\frac{2\log T}{\delta}\right)}
            + \sqrt{\sum_{t=1}^{T} X_t}
            \right)^2\\
    \end{align}
    
	\label{sec:eplsec} 
    \begin{lemma}[Elliptical potential lemma]\label{lem:epl}
        Fix $\lambda, A > 0$. Let $\{a_t\}_{t=1}^\infty$ be a sequence in $AB^d_2$ and let $V_0 = \lambda I$. Define $V_{t+1} = V_t + a_{t+1} a_{t+1}^\top$ for each $t \in \NN$. Then, for all $n \in \NN^+$,
        \begin{equation*}
            \sum_{t=1}^n \norm{a_t}_{V_{t-1}^{-1}}^2 \leq 2 d\max\left\{1,\frac{A^2}{\lambda}\right\} \log \left( 1 + \frac{n A^2}{d\lambda}\right) \,. 
        \end{equation*}
    \end{lemma}
    \begin{proof}
    See, e.g.,  Lemma 19.4 of \cite{Lattimore_Szepesvari_2020}.
    \end{proof}
\end{proof}

\begin{proposition}\label{prop:measurable-cond-exp}
    Let $X$ be a bounded random variable, $Y$ be an integrable random variable and $\cG$ be a $\sigma$-algebra such that $X$ is $\cG$-measurable. Then it holds almost surely that
    \begin{align*}
        \mE[XY\vert \cG]=X\mE[Y\vert \cG].
    \end{align*}
\end{proposition}
\begin{proof}
    See e.g. section XI.3.(h) of \cite{doob2012measure}.
\end{proof}
\begin{proposition}\label{prop:gaussian-to-uniform}\todos{Add citations}
    If $X\sim \cN(0,I)$, then $\frac{X}{\|X\|_2}\sim \mathrm{Unif}(\mS^{d-1})$.
\end{proposition}
\begin{proposition}\label{prop:uniform-sphere-inner-product}
    If $X\sim \mathrm{Unif}(\mS^{d-1})$ and $u$ is any fixed unit vector, then it follows that
    \begin{align*}
        \mE[|\langle X,u\rangle|] \ge  \sqrt{\frac{2}{\pi d}}.
    \end{align*}
\end{proposition}
\begin{proposition}\label{prop:uniform-sphere-inner-product-squared}
    If $X\sim \mathrm{Unif}(\mS^{d-1})$ and $u$ is any fixed unit vector, then it follows that
    \begin{align*}
        \mE\left[(\langle X,u\rangle)^2\right] =\frac{1}{d}.
    \end{align*}
\end{proposition}

\section{Details and Results of experiments}\label{appendix:experiment}

\subsection{Curvature-Aware vs.\ Na\"ive Logistic Exploration}

The purpose of this experiment is to demonstrate that the trivial extension of $\LinTS$ to the logistic case leads to poor behavior relative to $\THATS$. The extension mentioned simply replaces the least-squares method with MLE. This method will essentially focus on growing the “magnitude” of $V_t$. Hence we design an arm set $\{-e_i\}_{i=1}^{d-1}\cup \{0.3\cdot e_d,-0.3\cdot e_d\}$ (the arm set does not change across different rounds) and set $\theta_*=[M,M,\dotsc,1]$. The optimal and second optimal arm are $\pm 0.3 \cdot e_d$ respectively ($\approx 0.57, 0.43$).
Similar to the linear case, we first run $100$ runs on $M=2$ (where the suboptimal means $\approx 0.12$) for $T=1500$ rounds and report the average simple regret and the error of estimating each component of $\theta_*$. Then we vary $M$ in $\{1+0.5\cdot m\}_{m=1}^{18}$ with $d=50$ to see how many rounds are needed to make the simple regret fall below $10^{-4}$. The reason why we choose $10^{-4}$ is that in practice, for example, online advertisement clicking, the clicking probability (recall for Bernoulli it is the mean of the distribution) is usually very small: $\sim 10^{-3}$ (\citet{faury2020improved}) and a gap of $10^{-4}$ is a relative error of $0.1$\;.
\paragraph{Observations and Interpretations} Recall that $\kappa\approx \exp(M)$ as $M$ gets large. 
The optimal arm (after applied an inner product with $\theta_*$) lies in the near-linear, central part of the sigmoid function, i.e., having less curvature, where $\dot\mu(\cdot)$ is close to its max, while the suboptimal arms lie in the flat part ($\dot\mu$ close to zero). 
The reward received for $\pm 0.3 \cdot e_d$ will be quite noisy, because the mean (after taking a dot product with $\theta_*$) is close to $1/2$, and not noisy for the other arms. A clever method should thus pull $\pm 0.3 \cdot e_d$ (who also happen to be the best two arms by design) more often than the others to get sufficiently good estimates for separating the arm values. $\THATS$ indeed does this, while the extended $\LinTS$ method will fail to take this into account due to its nature of growing the magnitude of $V_t$ mentioned before. 

We also plotted the error of estimating the last component of $\theta_*$ for  
$M=2$ and observed that, as expected, the error of $\THATS$ vanishes much faster than that of the adapted $\LinTS$, while the variance across the runs is also much smaller as demonstrated in \cref{fig:M2_last_estimate}. What's more, the estimation of other ``unimportant'' components of $\theta_*$ (the first $d-1$ components) by $\THATS$ does not converge at all while that of $\LinTS$ does. This also provides evidence that $\THATS$ is focusing its exploration on the important directions only, which is exactly what we want, while $\LinTS$ wastes its time on these ``unimportant'' components until it fully rules them out.

As $M$ is varied in $\{1+0.5\cdot m\}_{m=1}^{18}$, for $d=50$, we observe that for a fixed level of suboptimality across $100$ independent runs, the average number of rounds needed for $\THATS$ is much less than that of $\LinTS$ with significantly lower variance as is demonstrated in \cref{fig:logcase}.
\paragraph{Implementation Details} For $\THATS$, we set $\lambda_{\text{log}}=1$ and for $\LinTS$ we set $\lambda_{\text{lin}}=1$, and $S=\|\theta_*\|+1$ for both. In order to reduce computational complexity, we set $\bar\theta_t$ (the constrained MLE) to be the global MLE $\hat\theta_t$ and set $\theta_t'$ to be the projection of $\tilde\theta_t$ to the $\ell_2$-ball of radius $S$.

%
%
%
\begin{figure*}[t]
    \centering 
    \begin{subfigure}[b]{0.48\textwidth}
        \centering
        \includegraphics[width=\linewidth]{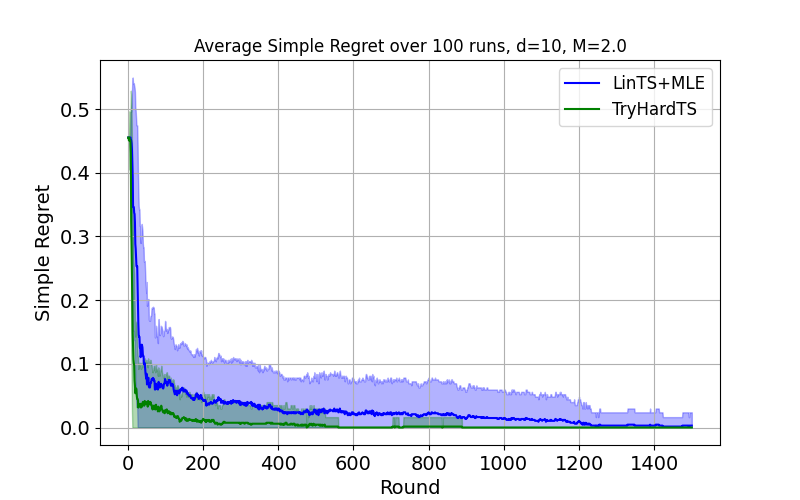}
    \end{subfigure}
    \hfill 
    \begin{subfigure}[b]{0.48\textwidth}
        \centering
        \includegraphics[width=\linewidth]{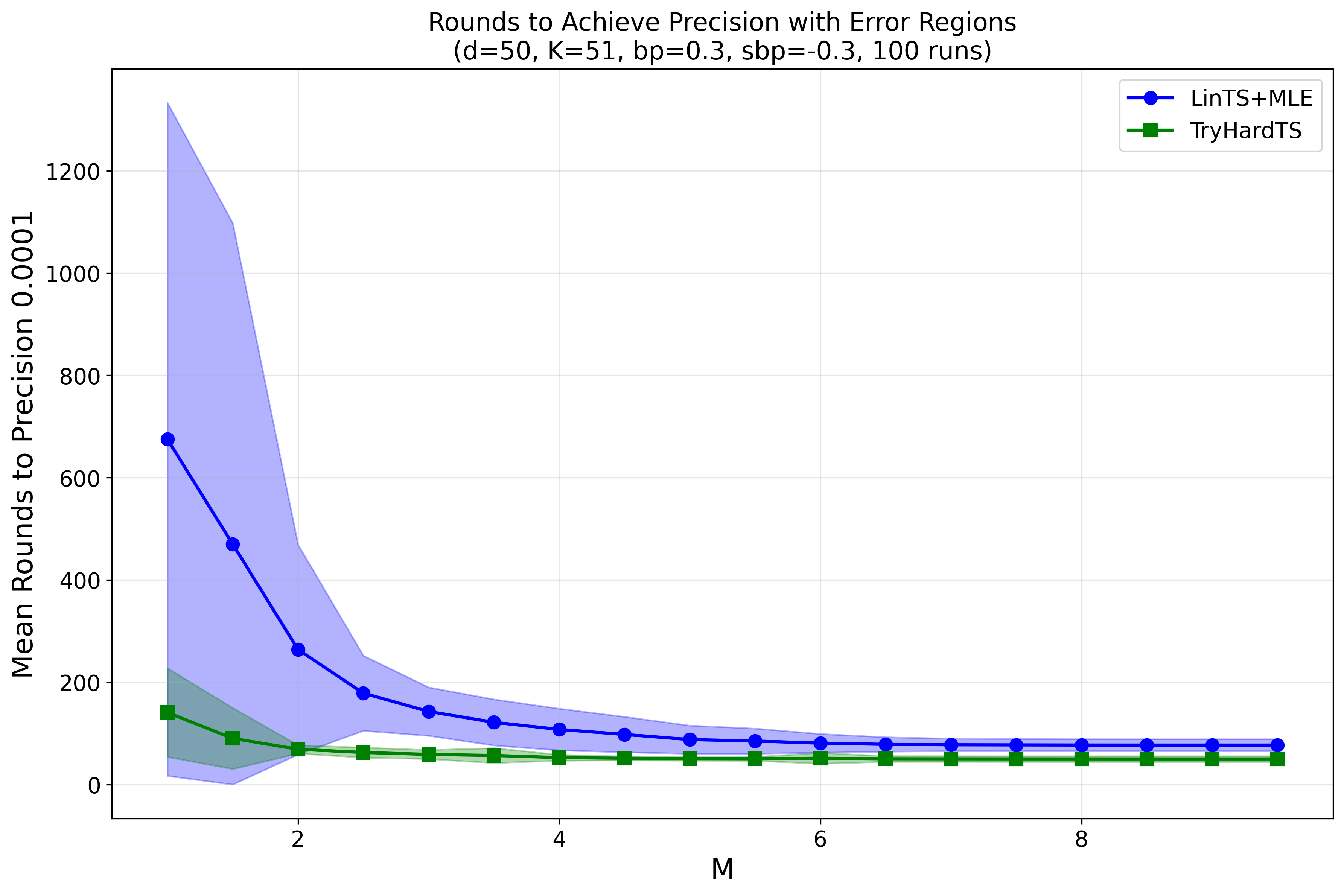}
    \end{subfigure}
    \caption{Left: Logistic case where $M=2$. Average simple regret vs. number of rounds for $d=10, T=1500, M=2$. Right: Number of rounds needed to make the simple regret fall below $10^{-4}$ vs. $M$ for $d=50$.}
    \label{fig:logcase}
\end{figure*}
\begin{figure*}[t]
    \centering 
    \begin{subfigure}[b]{0.48\textwidth}
        \centering
        \includegraphics[width=\linewidth]{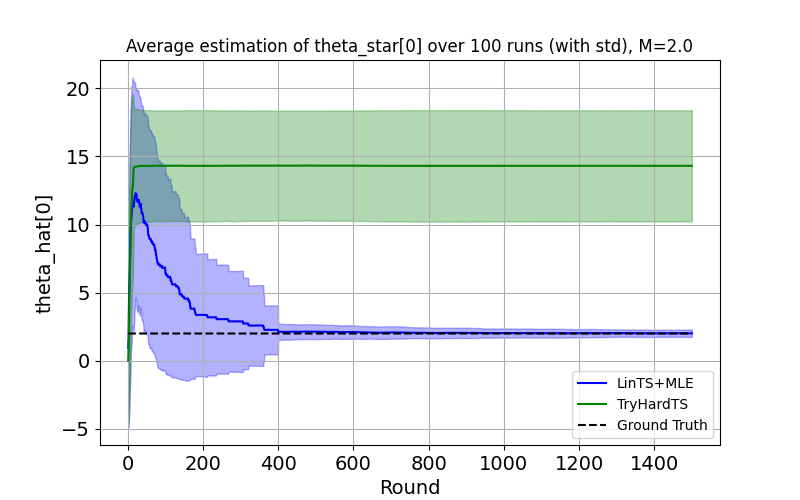}
    \end{subfigure}
    \hfill 
    \begin{subfigure}[b]{0.48\textwidth}
        \centering
        \includegraphics[width=\linewidth]{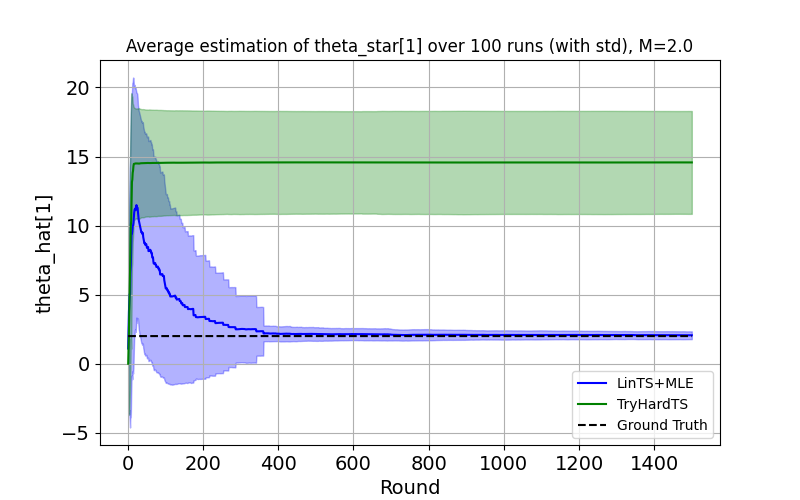}
    \end{subfigure}
    \caption{The estimation of "unimportant" components of $\theta_*$ for $M=2$. Left: $\theta_*[0]$, Right: $\theta_*[1]$. }
\end{figure*}
\begin{figure*}[t]
    \centering 
    \begin{subfigure}[b]{0.48\textwidth}
        \centering
        \includegraphics[width=\linewidth]{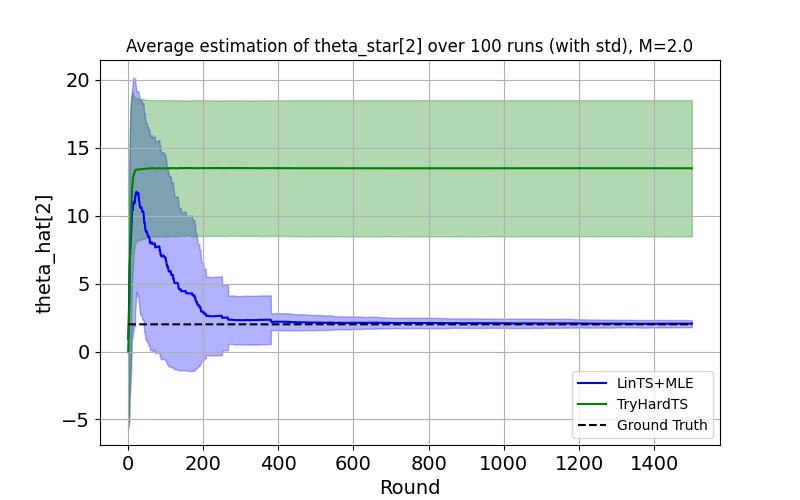}
    \end{subfigure}
    \hfill 
    \begin{subfigure}[b]{0.48\textwidth}
        \centering
        \includegraphics[width=\linewidth]{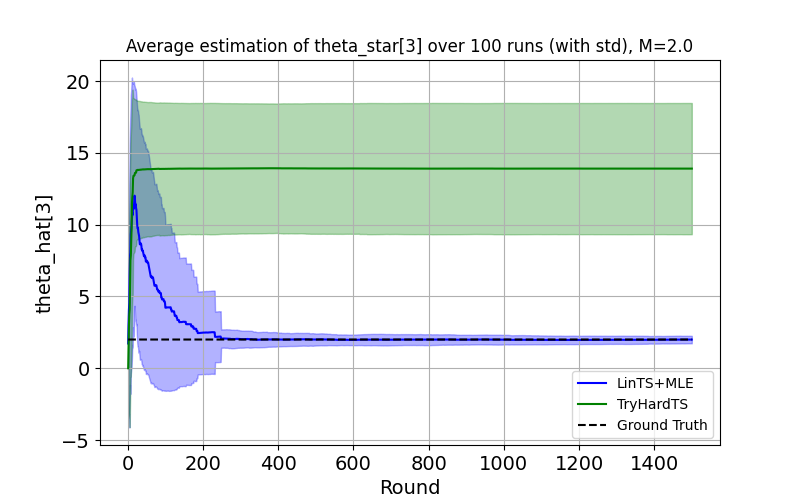}
    \end{subfigure}
    \caption{The estimation of "unimportant" components of $\theta_*$ for $M=2$. Left: $\theta_*[2]$, Right: $\theta_*[3]$. }
\end{figure*}
\begin{figure*}[t]
    \centering 
    \begin{subfigure}[b]{0.48\textwidth}
        \centering
        \includegraphics[width=\linewidth]{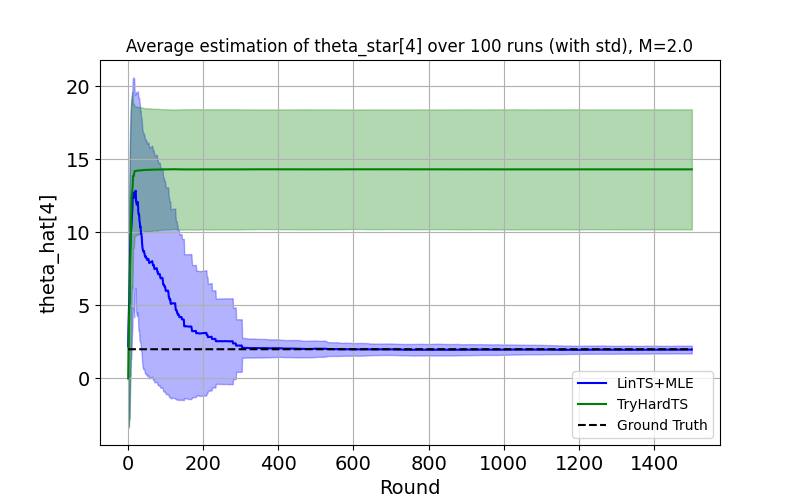}
    \end{subfigure}
    \hfill 
    \begin{subfigure}[b]{0.48\textwidth}
        \centering
        \includegraphics[width=\linewidth]{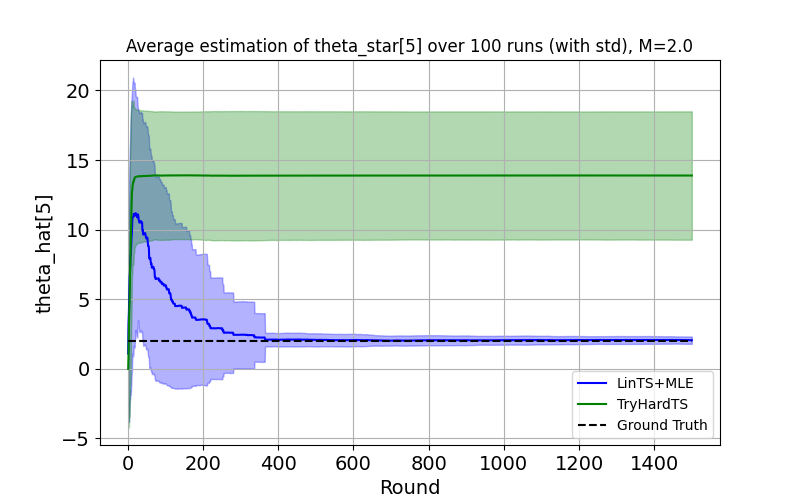}
    \end{subfigure}
    \caption{The estimation of "unimportant" components of $\theta_*$ for $M=2$. Left: $\theta_*[4]$, Right: $\theta_*[5]$. }
\end{figure*}
\begin{figure*}[t]
    \centering 
    \begin{subfigure}[b]{0.48\textwidth}
        \centering
        \includegraphics[width=\linewidth]{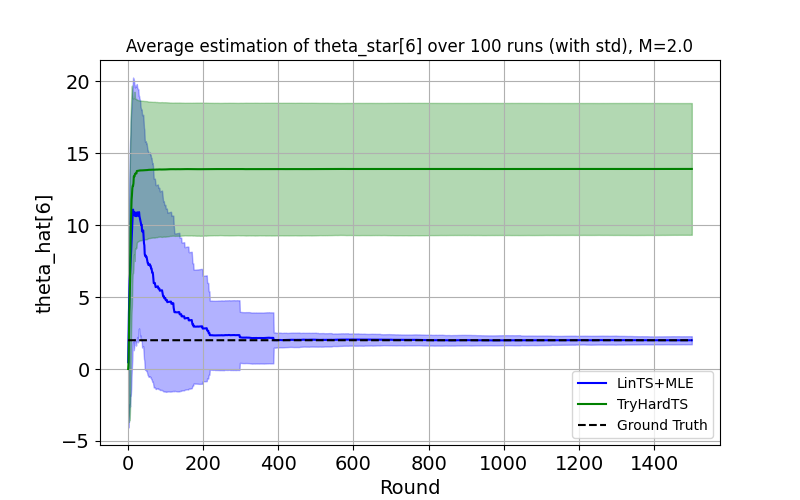}
    \end{subfigure}
    \hfill 
    \begin{subfigure}[b]{0.48\textwidth}
        \centering
        \includegraphics[width=\linewidth]{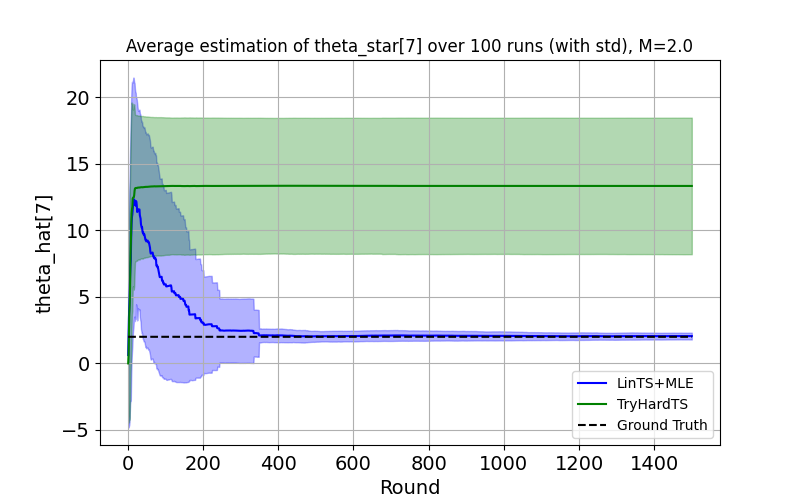}
    \end{subfigure}
    \caption{The estimation of "unimportant" components of $\theta_*$ for $M=2$. Left: $\theta_*[6]$, Right: $\theta_*[7]$. }
\end{figure*}
\begin{figure*}[t]
    \centering 
    \begin{subfigure}[b]{0.48\textwidth}
        \centering
        \includegraphics[width=\linewidth]{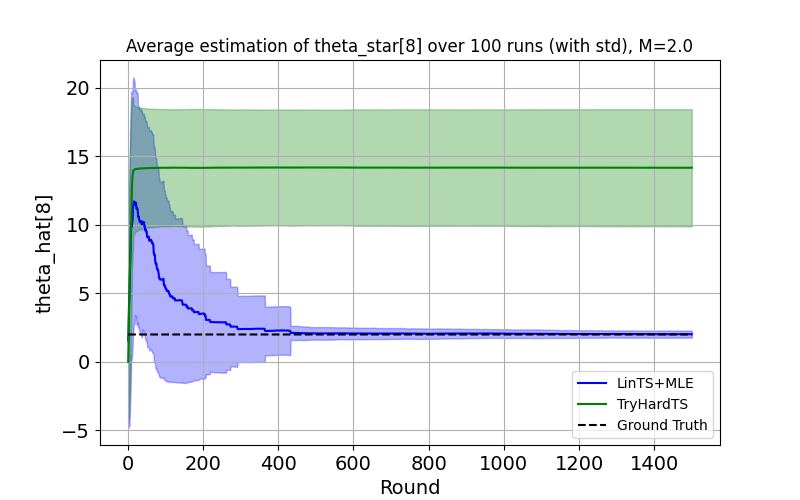}
    \end{subfigure}
    \hfill 
    \begin{subfigure}[b]{0.48\textwidth}
        \centering
        \includegraphics[width=\linewidth]{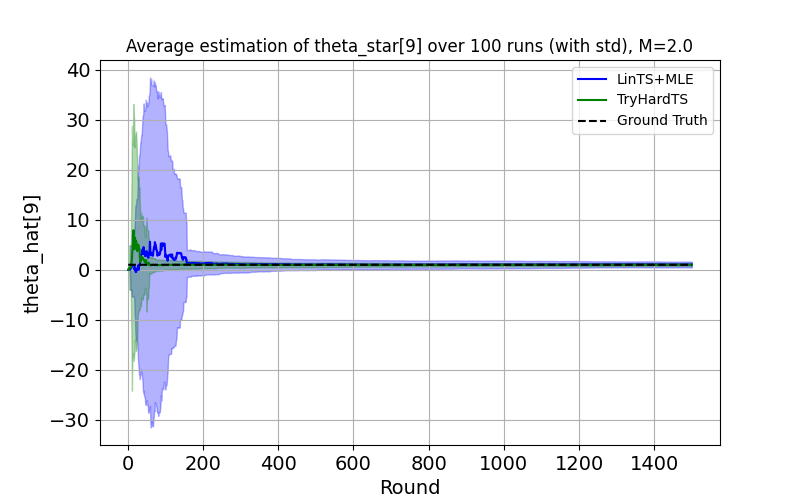}
        
    \end{subfigure}
    \caption{The estimation of components of $\theta_*$ for $M=2$. Left: $\theta_*[8]$, Right: $\theta_*[9]$ (the "important" component). }
    \label{fig:M2_last_estimate}
\end{figure*}
\newpage

\subsection{Easy-Geometry Experiment}\label{appendix:easy-geometry-experiment}
\paragraph{Algorithms and hypotheses}
We run $\MULog$, $\THATS$, cumulative Thompson sampling (with online-to-batch conversion), adaECOLog-OFU, adaECOLog-TS \citep{faury2022jointly}, and uniform exploration on both $\cA_{\mathrm{hard}}$ and $\cA_{\mathrm{easy}}$. Three hypotheses are tested. First, adding informative low-reward probe arms should make the instance easier for pure-exploration algorithms. Second, this improvement should come from sampling the probe arms, not merely from changing the recommendation set — we track the fraction of pulls allocated to $\{\pm e_2,\ldots,\pm e_d\}$. Third, the improvement should be stronger for curvature-aware methods than for cumulative-regret baselines. We also note that cumulative-regret baselines can appear worse on $\cA_{\mathrm{easy}}$ than on $\cA_{\mathrm{hard}}$: these methods face a no-win situation with probe arms — using them provides useful information early but penalizes the uniform online-to-batch average over all $T$ past actions, while avoiding them forgoes the information gain entirely; pure-exploration methods do not face this dilemma because they track probe-arm sampling separately from the final recommendation.

\paragraph{Instance construction}
We use a finite-arm version of the shifted saturated hypercube construction. Fix $d\ge 2$, a saturation level $m>0$, and a small perturbation parameter $\varepsilon>0$. For a sign vector $v\in\{-1,1\}^{d-1}$, define
\[
    \theta_v
    =
    \sqrt 2
    \left(
        m-\varepsilon,
        \frac{\varepsilon}{\sqrt{d-1}}v
    \right).
\]
The candidate recommendation arms are
\[
    a_u
    =
    \frac{1}{\sqrt 2}
    \left(
        1,
        \frac{u}{\sqrt{d-1}}
    \right),
    \qquad
    u\in\{-1,1\}^{d-1}.
\]
For the instance indexed by $v$, the unique optimal candidate arm is $a_v$. Moreover, for every candidate arm $a_u$, the logit $a_u^\top\theta_v$ lies near $m$. Thus, when $m$ is large, all candidate arms are in a saturated region of the sigmoid and their rewards carry relatively little information.

We compare two action sets. The hard action set is
\[
    \cA_{\mathrm{hard}}
    =
    \{a_u:u\in\{-1,1\}^{d-1}\}.
\]
The easy-probe action set adds statistically informative but low-reward probe arms:
\[
    \cA_{\mathrm{easy}}
    =
    \cA_{\mathrm{hard}}
    \cup
    \{\pm e_2,\ldots,\pm e_d\}.
\]
The probe arms have zero first coordinate, and hence their logits under $\theta_v$ are approximately
\[
    \pm \frac{\sqrt 2\varepsilon}{\sqrt{d-1}}v_i.
\]
These logits are close to the center of the sigmoid, so the corresponding Bernoulli rewards have high variance and provide useful information about the signs of $v$. At the same time, these arms have much smaller expected reward than the candidate arms when $m$ is large, so they are unattractive from the cumulative-regret perspective.

\paragraph{Measurements}
For each algorithm and action set, we measure the average simple regret as a function of the exploration budget $T$. We also record the number of rounds required to reach a fixed simple-regret threshold. To make the comparison across saturation levels meaningful, we report both the raw simple regret and the normalized simple regret divided by $\sigma_*=\dot\mu(a_v^\top\theta_v)$, the reward variance at the optimal arm.

In the easy-probe instance, we additionally report the fraction of samples allocated to probe arms,
\[
    \frac{1}{T}\sum_{t=1}^T \II\{A_t\in\{\pm e_2,\ldots,\pm e_d\}\}.
\]
This directly tests whether an algorithm is using the informative low-reward actions. We report this quantity separately for $\MULog$, $\THATS$, cumulative Thompson sampling, and uniform exploration. We also track the sign-recovery error
\[
    \frac{1}{d-1}
    \sum_{i=1}^{d-1}
    \II\{\hat v_i\neq v_i\},
\]
where $\hat v$ is the sign vector associated with the final recommended candidate arm. This diagnostic separates the statistical task of learning the correct direction from the final simple-regret value.

\paragraph{Experimental parameters}
We use a single environment rather than a sweep. We set $d=8$, $m=4$, $\varepsilon=1$, and $v=(1,\ldots,1)$. Fixing $v$ loses no generality for this construction, since all sign vectors are equivalent up to coordinate sign flips of the action set. It also removes unnecessary environment randomness and makes the hard/easy comparisons cleaner.

These choices give $|\cA_{\mathrm{hard}}|=2^{7}=128$ and $|\cA_{\mathrm{easy}}|=142$, so the finite-arm optimizations remain manageable and the probe arms are only a small fraction of the easy action set. Candidate arms have logits in $[m-2\varepsilon,m]=[2,4]$, while the optimal-arm variance is $\sigma_*=\dot\mu(4)\approx 1.77\cdot 10^{-2}$. With these values, a candidate arm with one sign mistake has logit $4-2/7$ and simple regret approximately $5.8\cdot 10^{-3}$, while a candidate arm with two sign mistakes has logit $4-4/7$ and simple regret approximately $1.36\cdot 10^{-2}$. Thus the proposed thresholds below test whether the final recommendation is nearly optimal, roughly allowing one sign mistake but not two, rather than requiring exact best-arm identification. Probe-arm logits have magnitude $\sqrt{2/7}\approx 0.535$, so they lie near the high-curvature part of the sigmoid. At the same time, their mean rewards are about $0.63$ and $0.37$, much smaller than the candidate-arm rewards, which is what should make them unattractive to cumulative-regret algorithms.

The raw simple-regret threshold is $10^{-2}$ and the normalized threshold is $0.5$ (giving $0.5\sigma_*\approx 8.8\cdot 10^{-3}$); both thresholds lie between the regret incurred by one and by two sign mistakes, so they test whether the final recommendation has at most one wrong sign rather than exact best-arm identification. We ran $100$ independent runs with $T_{\max}=10000$ rounds. All algorithms use $\lambda=1$ and $S=\|\theta_v\|+1$. Both $\THATS$ and $\MULog$ use the MLE-centered curvature proxy: $\bar\theta_t$ is set to the global regularized MLE $\hat\theta_t$, sampled parameters are projected to the $\ell_2$-ball of radius $S$, and $L_t$ is built from this MLE-centered approximation rather than from exact confidence-set intersections (so neither algorithm depends on a $\delta$ parameter in practice). Cumulative Thompson sampling samples $\tilde\theta_t$ from the Gaussian approximation $\mathcal{N}(\hat\theta_t, H_t^{-1}(\hat\theta_t))$, plays the arm with largest sampled logit, and returns a recommendation by sampling uniformly from the $T$ past actions and taking the one with highest logit under the final MLE. The adaECOLog-OFU and adaECOLog-TS baselines use the implementation of \citet{faury2022jointly} with $\delta=0.05$ and also return a recommendation by sampling uniformly from the $T$ past actions (online-to-batch conversion).
\begin{figure}[t]
    \centering

    \begin{subfigure}[t]{0.48\textwidth}
        \centering
        \includegraphics[width=\linewidth]{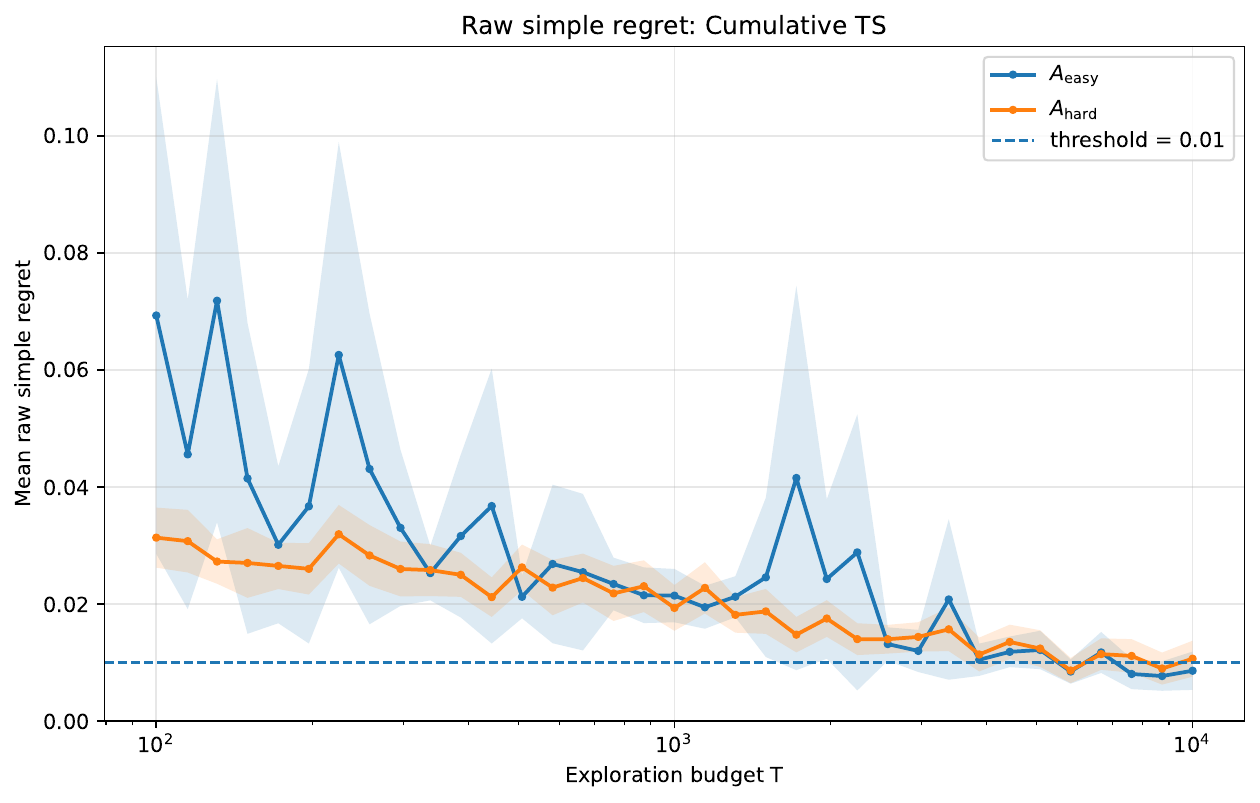}
        \caption{Raw regret}
        \label{fig:cum_ts_raw_regret}
    \end{subfigure}
    \hfill
    \begin{subfigure}[t]{0.48\textwidth}
        \centering
        \includegraphics[width=\linewidth]{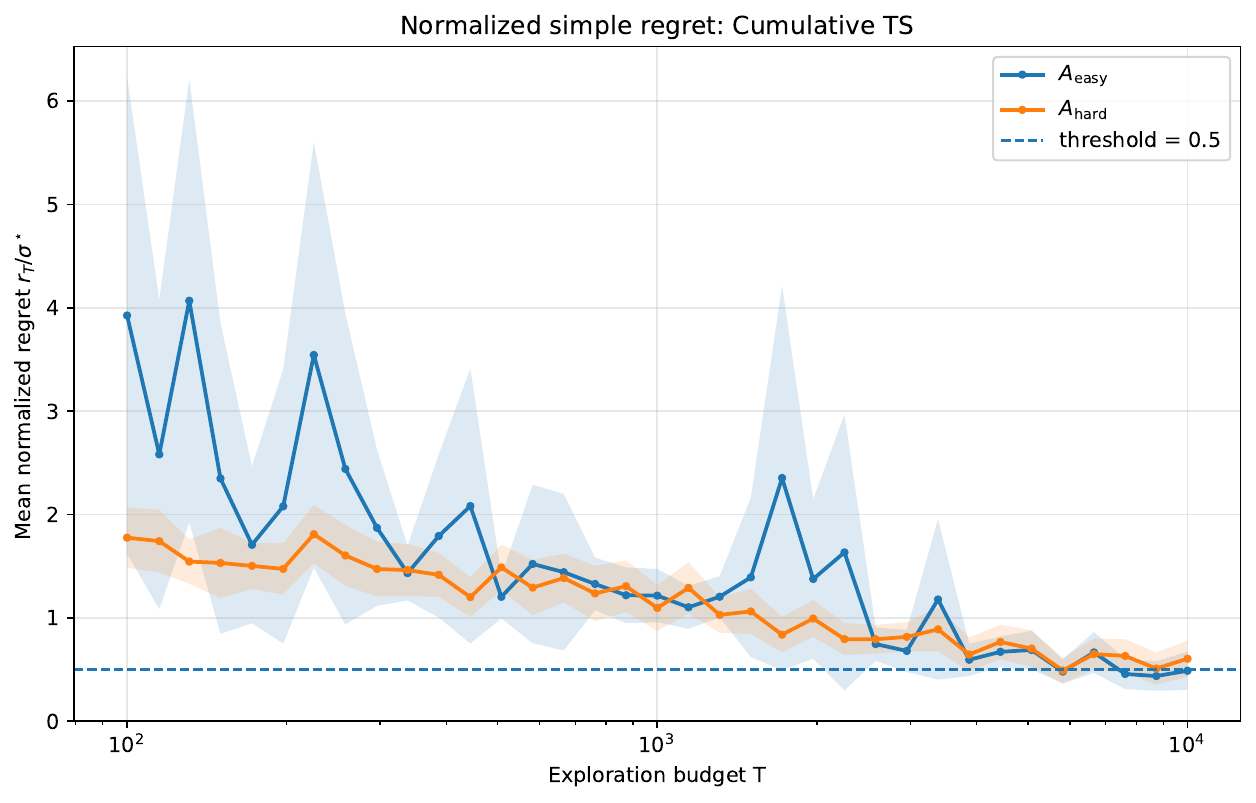}
        \caption{Normalized regret}
        \label{fig:cum_ts_normalized_regret}
    \end{subfigure}

    \vspace{0.75em}

    \begin{subfigure}[t]{0.48\textwidth}
        \centering
        \includegraphics[width=\linewidth]{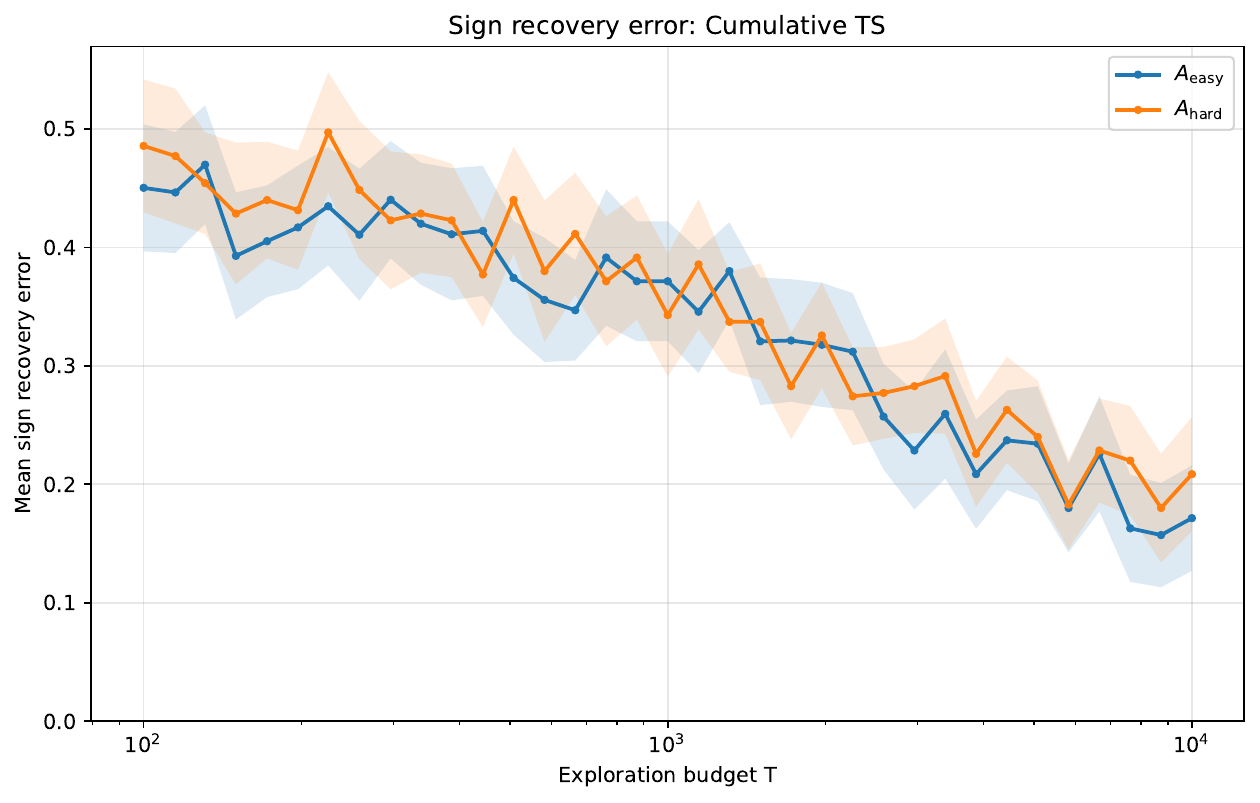}
        \caption{Sign recovery error}
        \label{fig:cum_ts_sign_recovery_error}
    \end{subfigure}
    \hfill
    \begin{subfigure}[t]{0.48\textwidth}
        \centering
        \includegraphics[width=\linewidth]{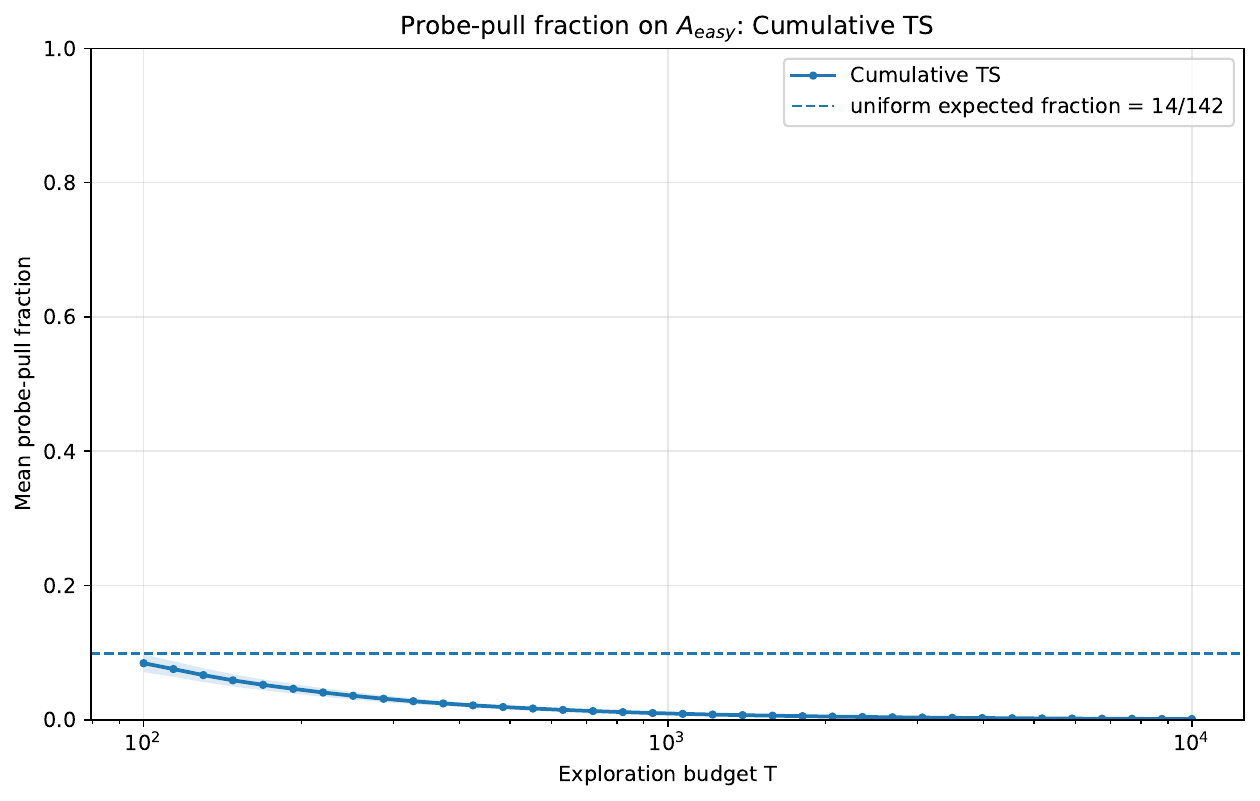}
        \caption{Probe fraction}
        \label{fig:cum_ts_probe_fraction}
    \end{subfigure}

    \vspace{0.75em}

    \begin{subfigure}[t]{0.48\textwidth}
        \centering
        \includegraphics[width=\linewidth]{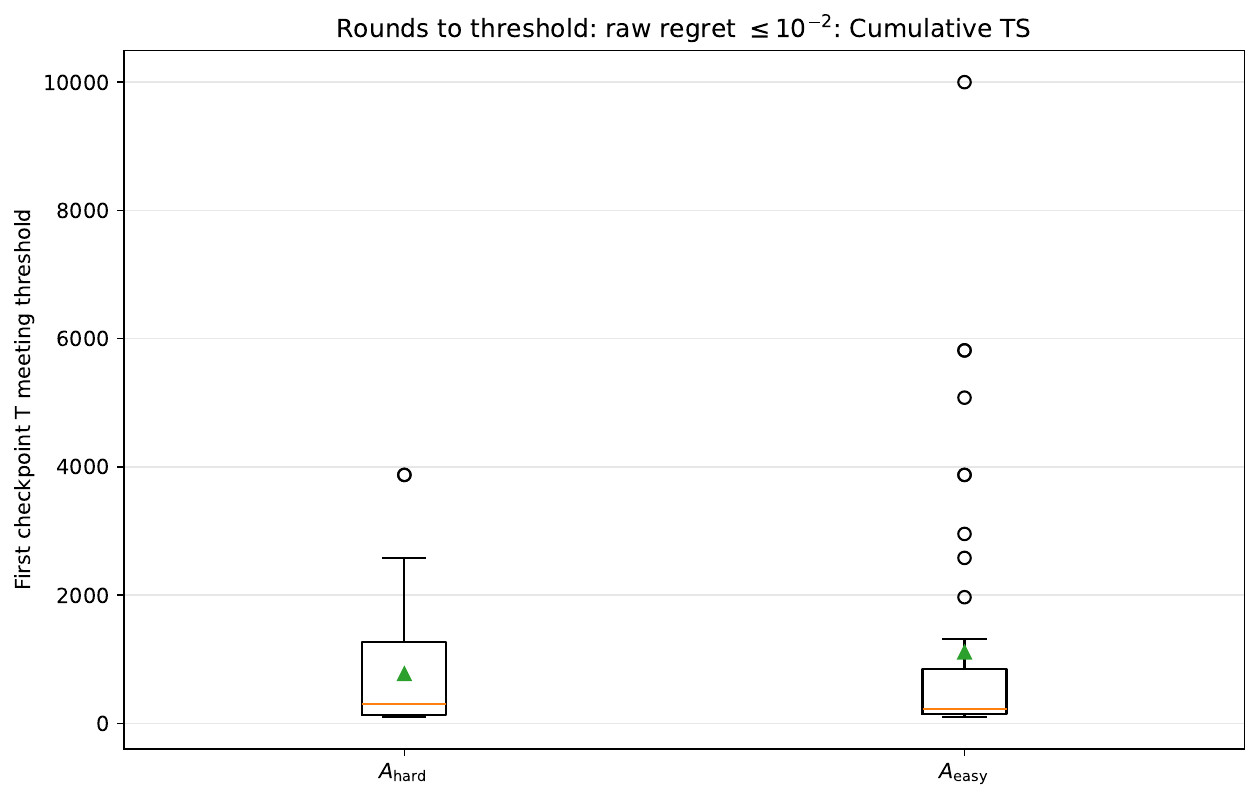}
        \caption{Rounds to raw threshold}
        \label{fig:cum_ts_rounds_to_raw_threshold}
    \end{subfigure}
    \hfill
    \begin{subfigure}[t]{0.48\textwidth}
        \centering
        \includegraphics[width=\linewidth]{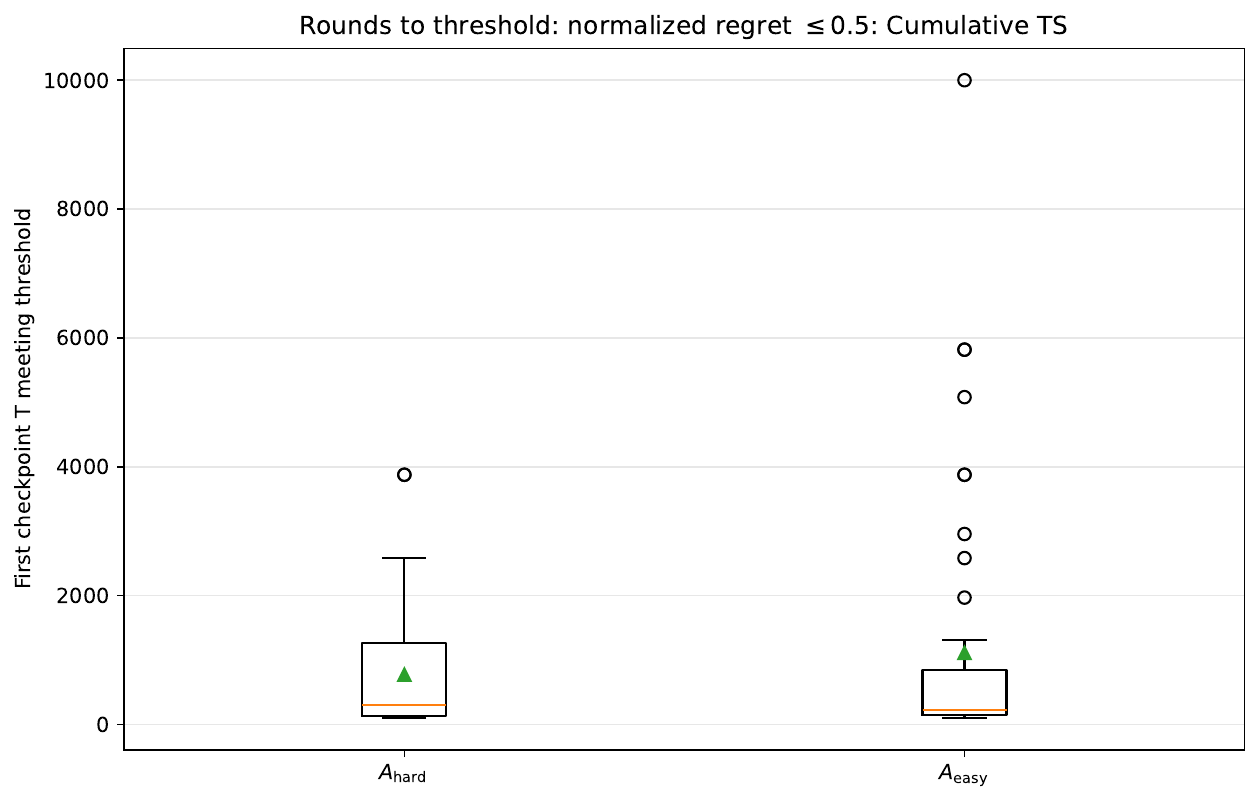}
        \caption{Rounds to normalized threshold}
        \label{fig:cum_ts_rounds_to_normalized_threshold}
    \end{subfigure}

    \caption{Metrics for cumulative-regret Thompson sampling.}
    \label{fig:cum-ts-metrics}
\end{figure}

\begin{figure}[t]
    \centering

    \begin{subfigure}[t]{0.48\textwidth}
        \centering
        \includegraphics[width=\linewidth]{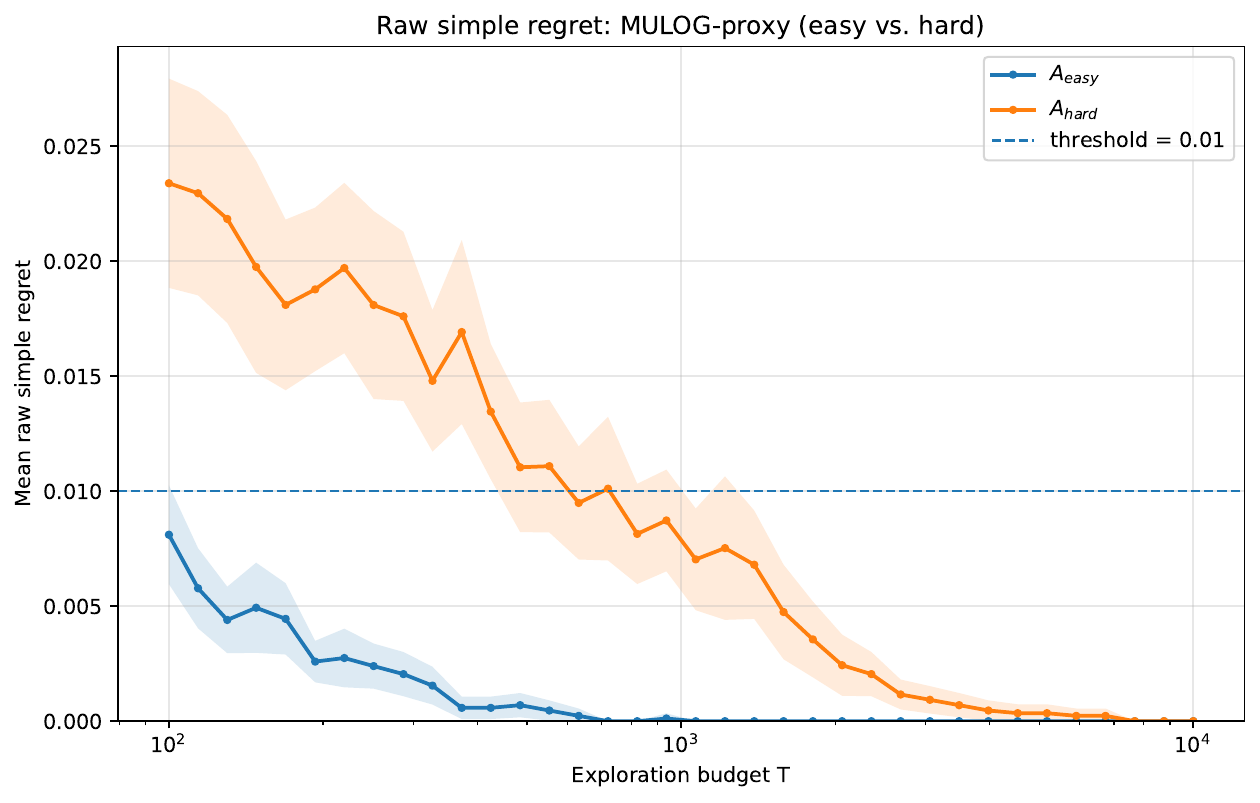}
        \caption{Raw regret}
        \label{fig:mulog_proxy_raw_regret}
    \end{subfigure}
    \hfill
    \begin{subfigure}[t]{0.48\textwidth}
        \centering
        \includegraphics[width=\linewidth]{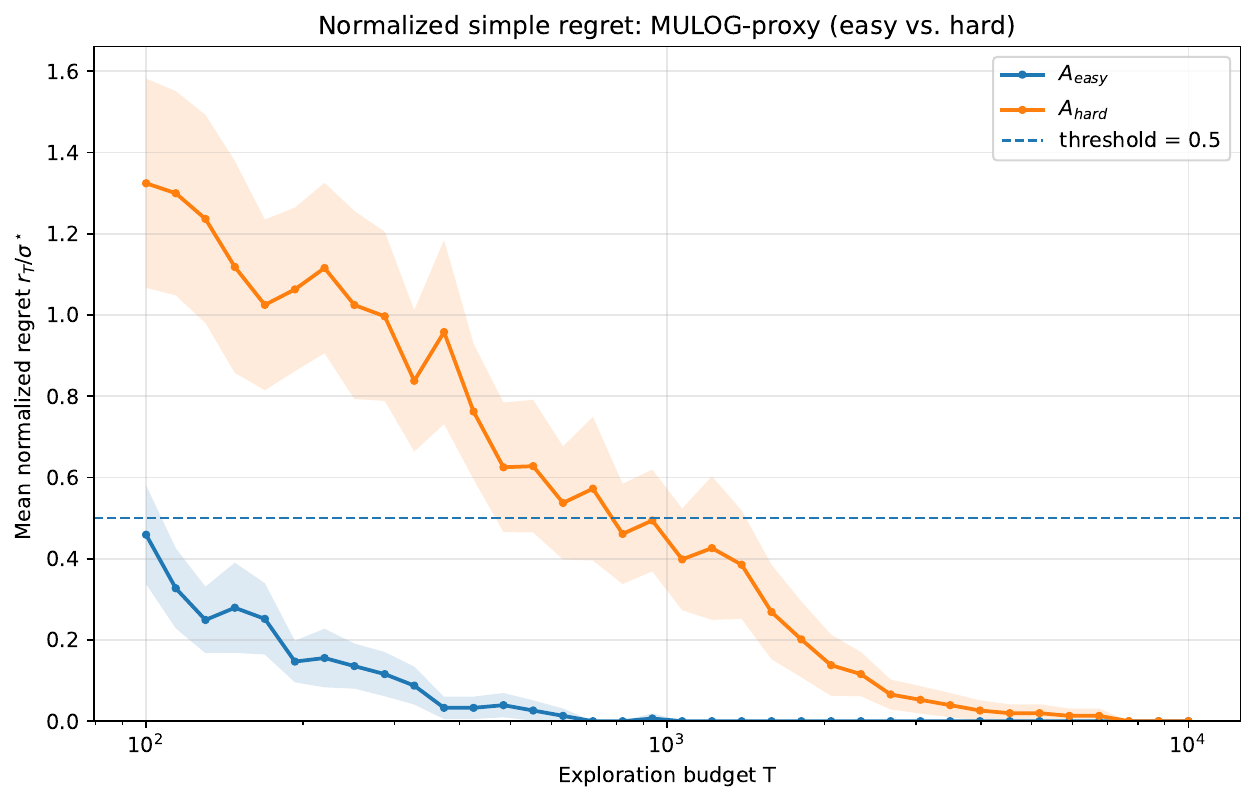}
        \caption{Normalized regret}
        \label{fig:mulog_proxy_normalized_regret}
    \end{subfigure}

    \vspace{0.75em}

    \begin{subfigure}[t]{0.48\textwidth}
        \centering
        \includegraphics[width=\linewidth]{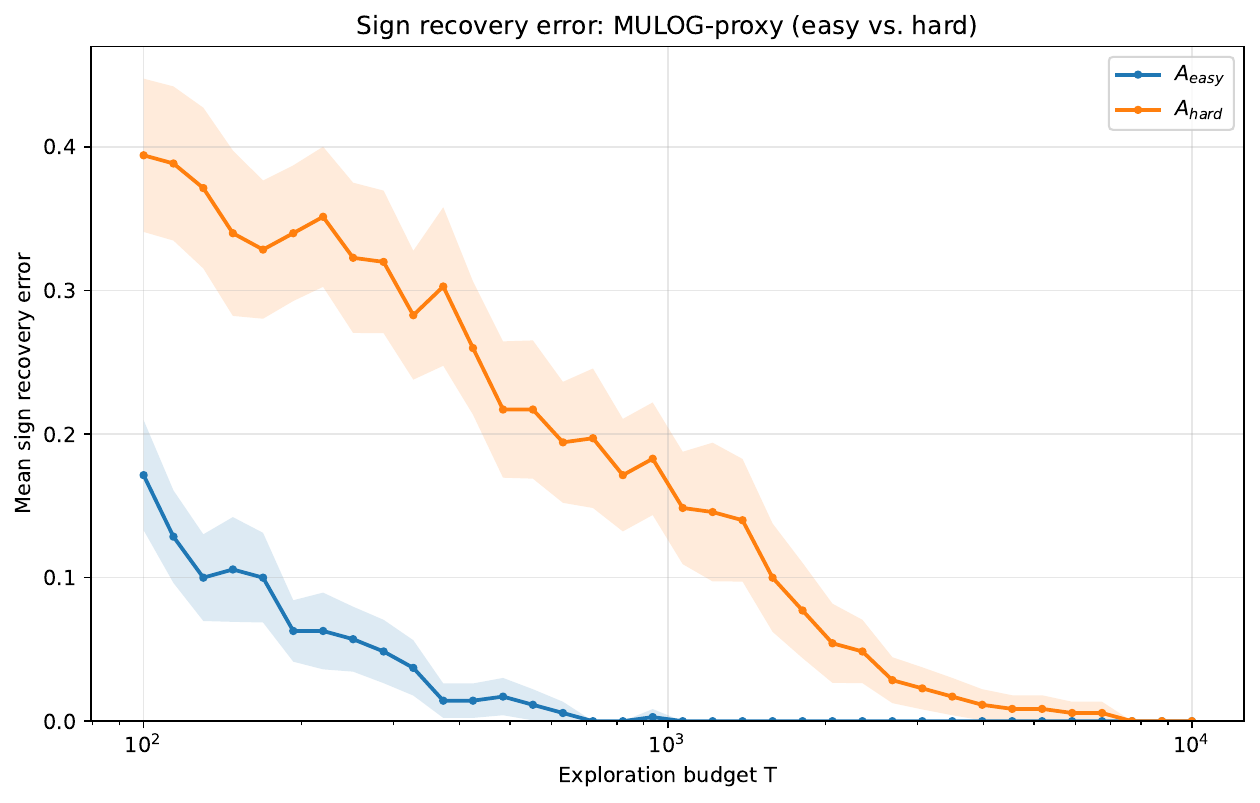}
        \caption{Sign recovery error}
        \label{fig:mulog_proxy_sign_recovery_error}
    \end{subfigure}
    \hfill
    \begin{subfigure}[t]{0.48\textwidth}
        \centering
        \includegraphics[width=\linewidth]{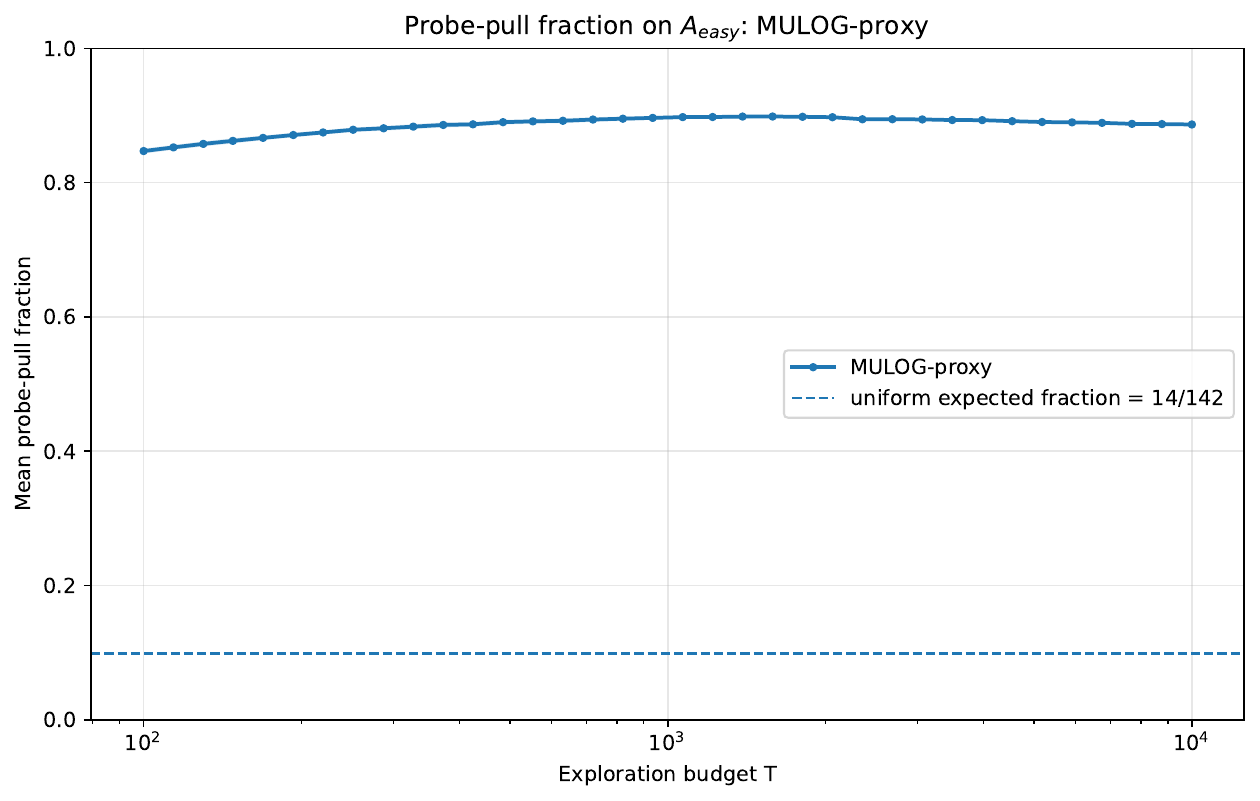}
        \caption{Probe fraction}
        \label{fig:mulog_proxy_probe_fraction}
    \end{subfigure}

    \vspace{0.75em}

    \begin{subfigure}[t]{0.48\textwidth}
        \centering
        \includegraphics[width=\linewidth]{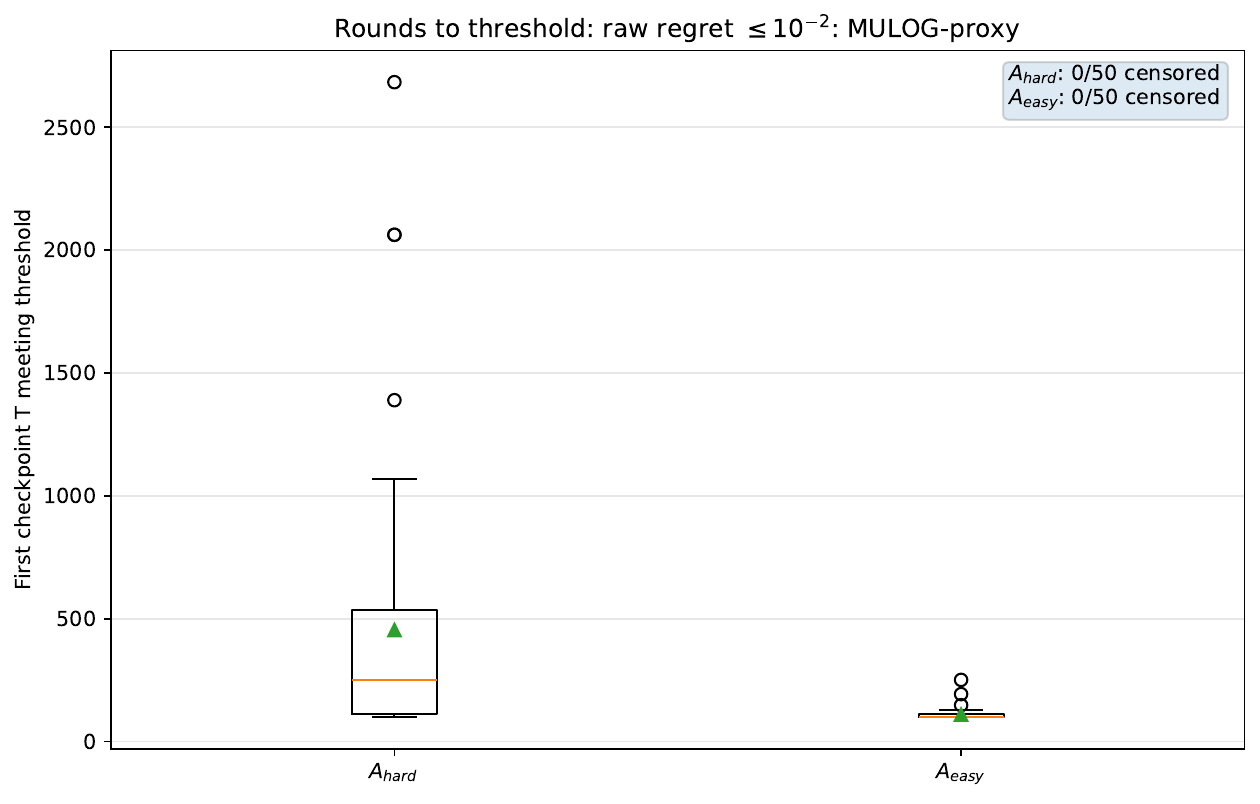}
        \caption{Rounds to raw threshold}
        \label{fig:mulog_proxy_rounds_to_raw_threshold}
    \end{subfigure}
    \hfill
    \begin{subfigure}[t]{0.48\textwidth}
        \centering
        \includegraphics[width=\linewidth]{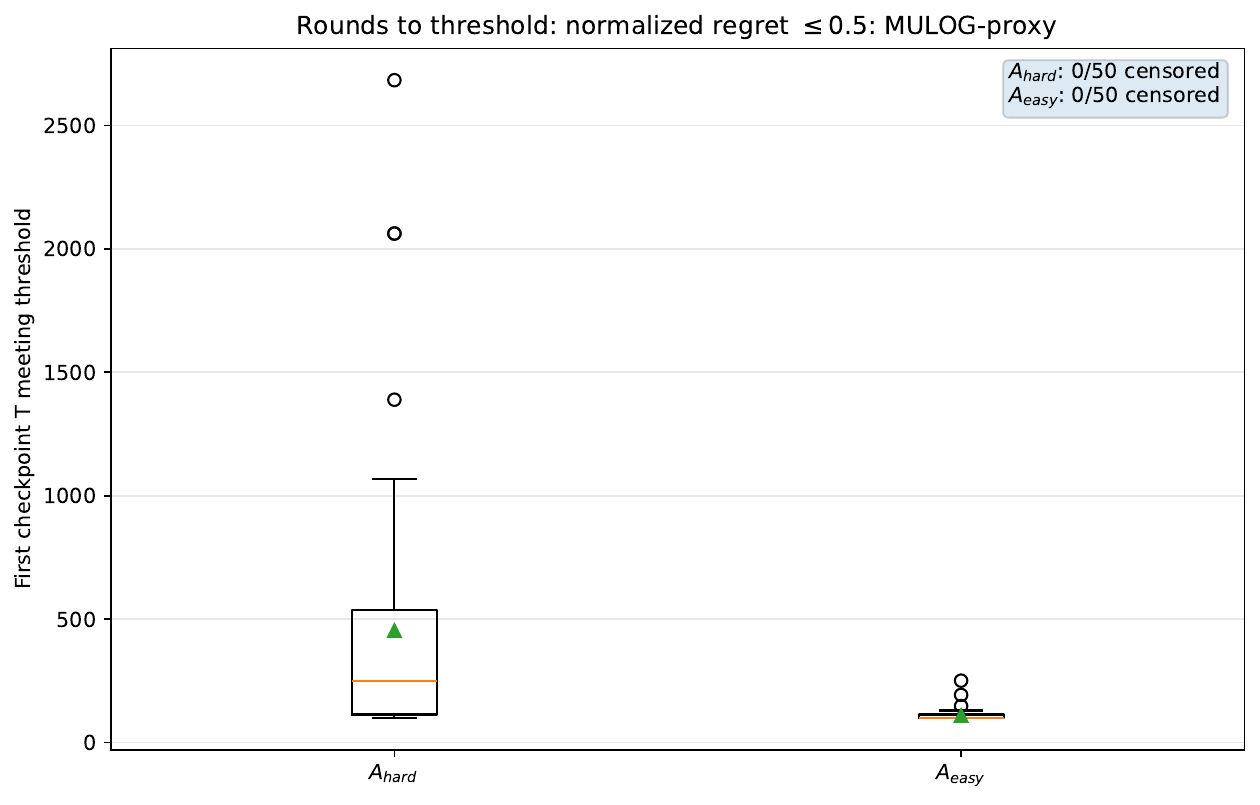}
        \caption{Rounds to normalized threshold}
        \label{fig:mulog_proxy_rounds_to_normalized_threshold}
    \end{subfigure}

    \caption{Metrics for the MULog proxy algorithm.}
    \label{fig:mulog-proxy-metrics}
\end{figure}

\begin{figure}[t]
    \centering

    \begin{subfigure}[t]{0.48\textwidth}
        \centering
        \includegraphics[width=\linewidth]{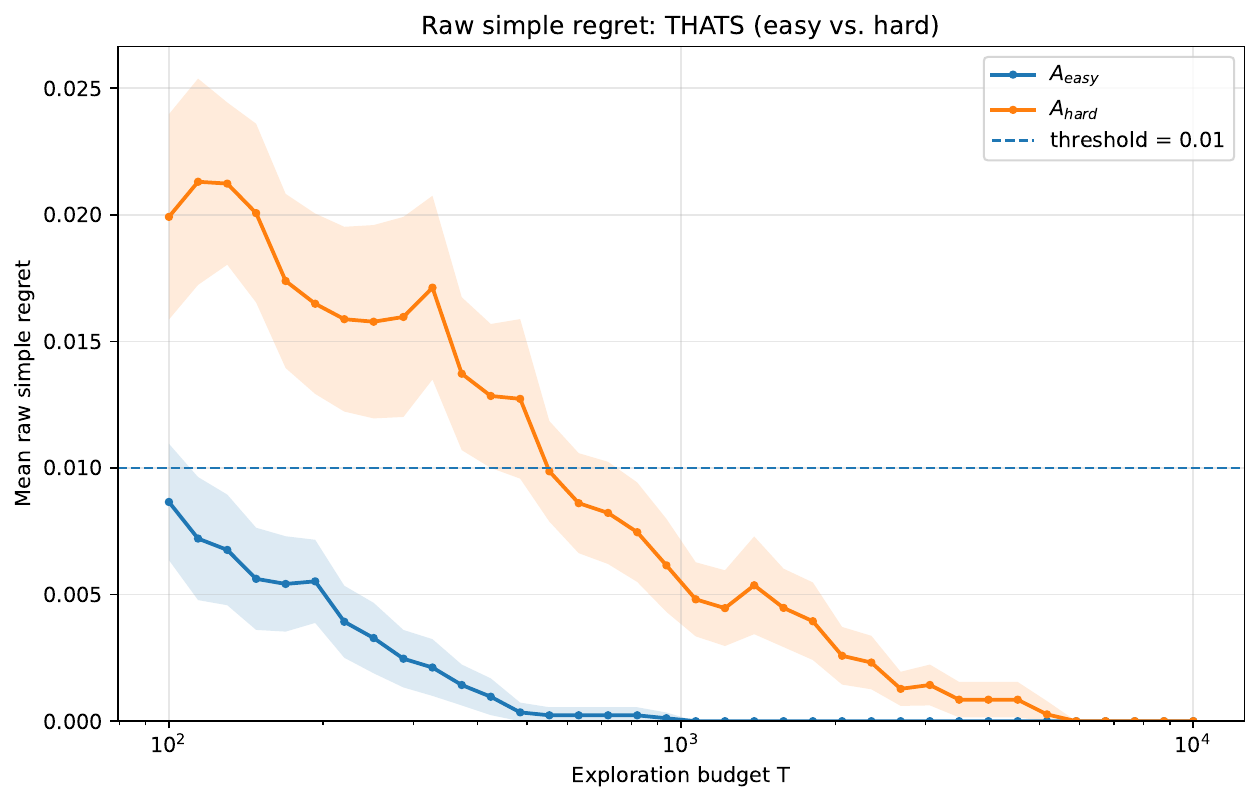}
        \caption{Raw regret}
        \label{fig:thats_raw_regret}
    \end{subfigure}
    \hfill
    \begin{subfigure}[t]{0.48\textwidth}
        \centering
        \includegraphics[width=\linewidth]{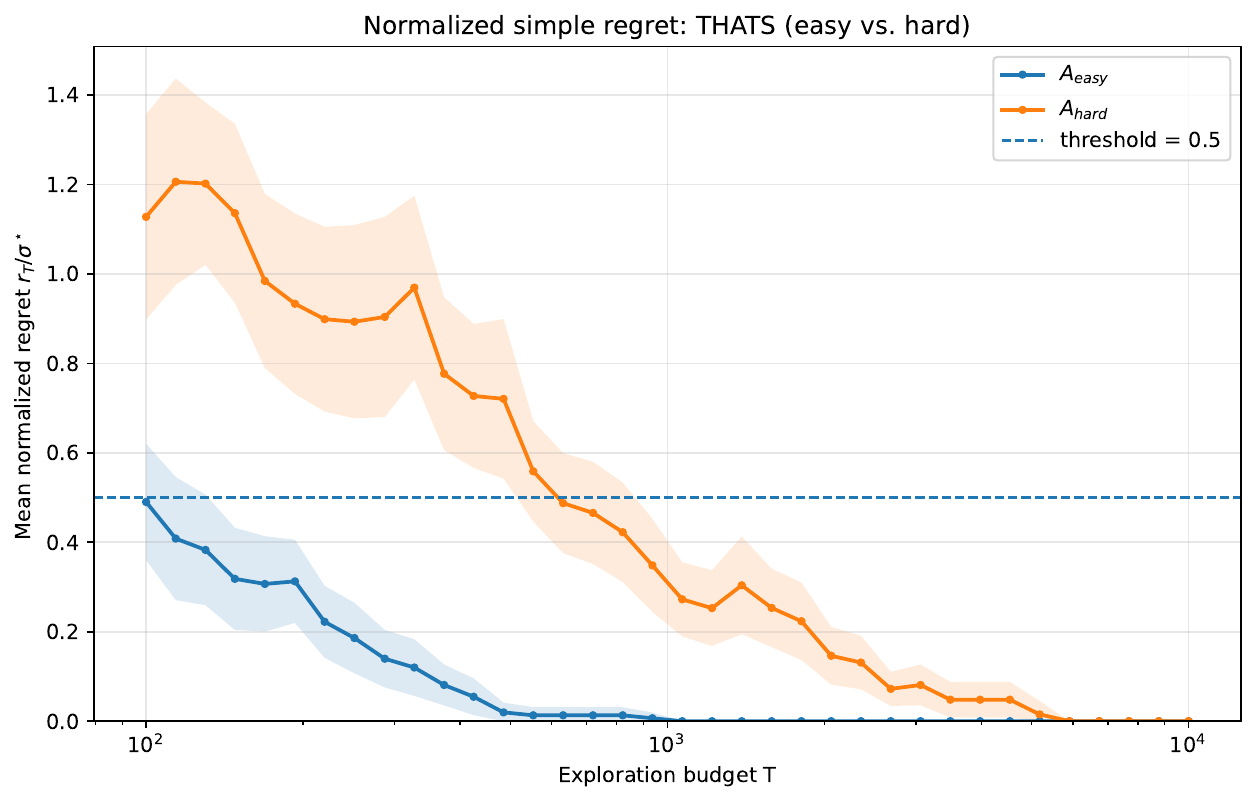}
        \caption{Normalized regret}
        \label{fig:thats_normalized_regret}
    \end{subfigure}

    \vspace{0.75em}

    \begin{subfigure}[t]{0.48\textwidth}
        \centering
        \includegraphics[width=\linewidth]{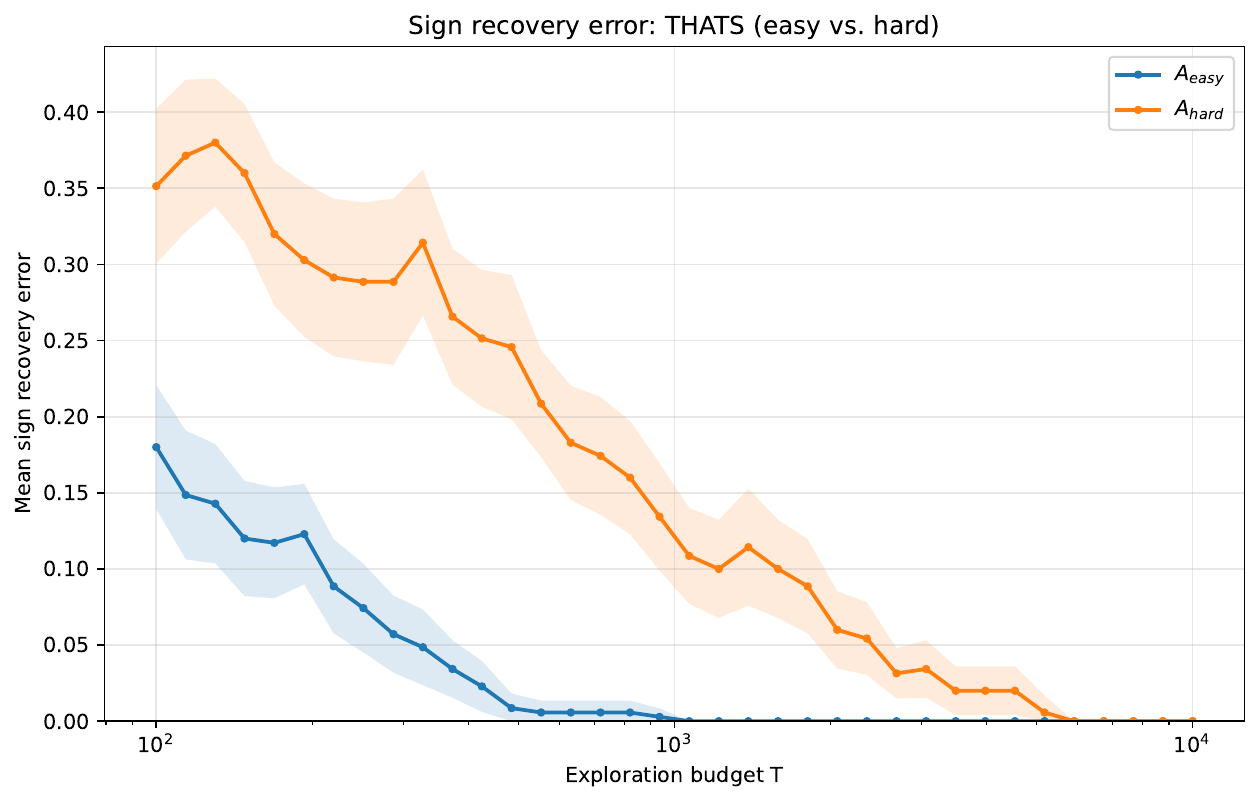}
        \caption{Sign recovery error}
        \label{fig:thats_sign_recovery_error}
    \end{subfigure}
    \hfill
    \begin{subfigure}[t]{0.48\textwidth}
        \centering
        \includegraphics[width=\linewidth]{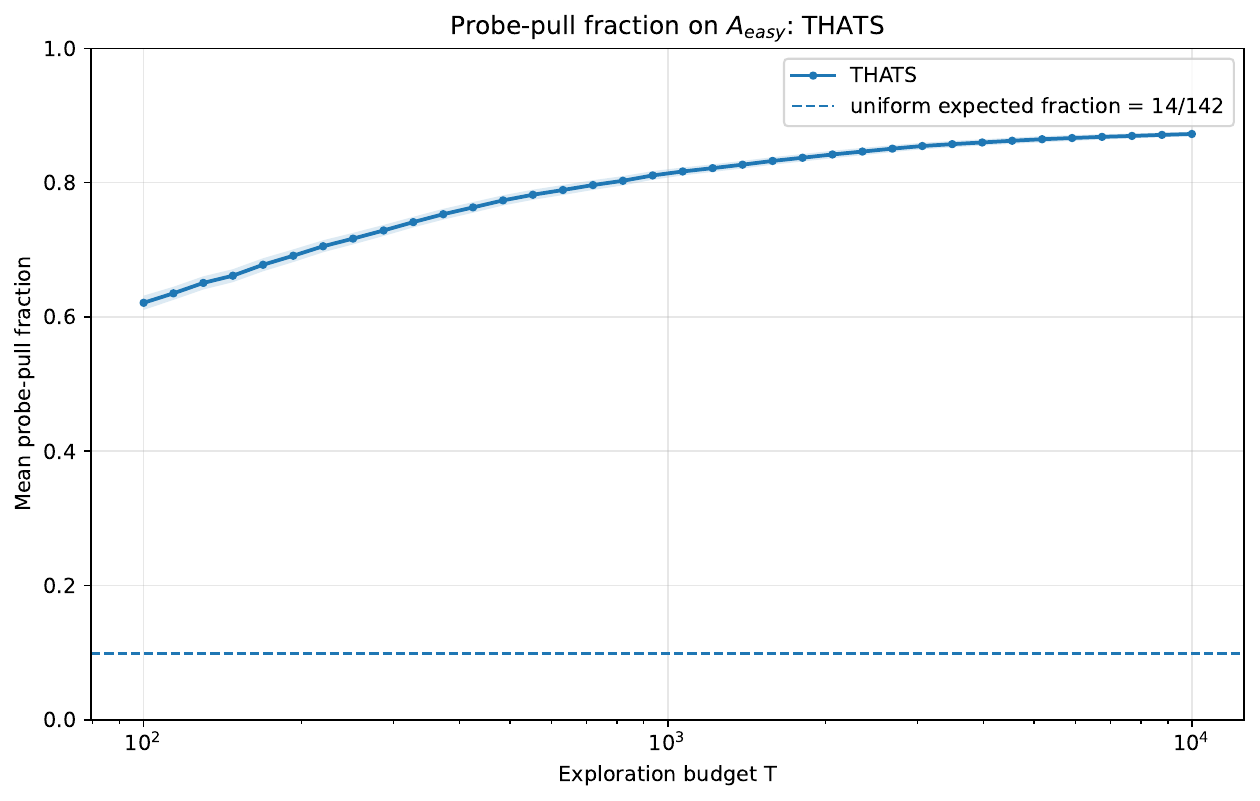}
        \caption{Probe fraction}
        \label{fig:thats_probe_fraction}
    \end{subfigure}

    \vspace{0.75em}

    \begin{subfigure}[t]{0.48\textwidth}
        \centering
        \includegraphics[width=\linewidth]{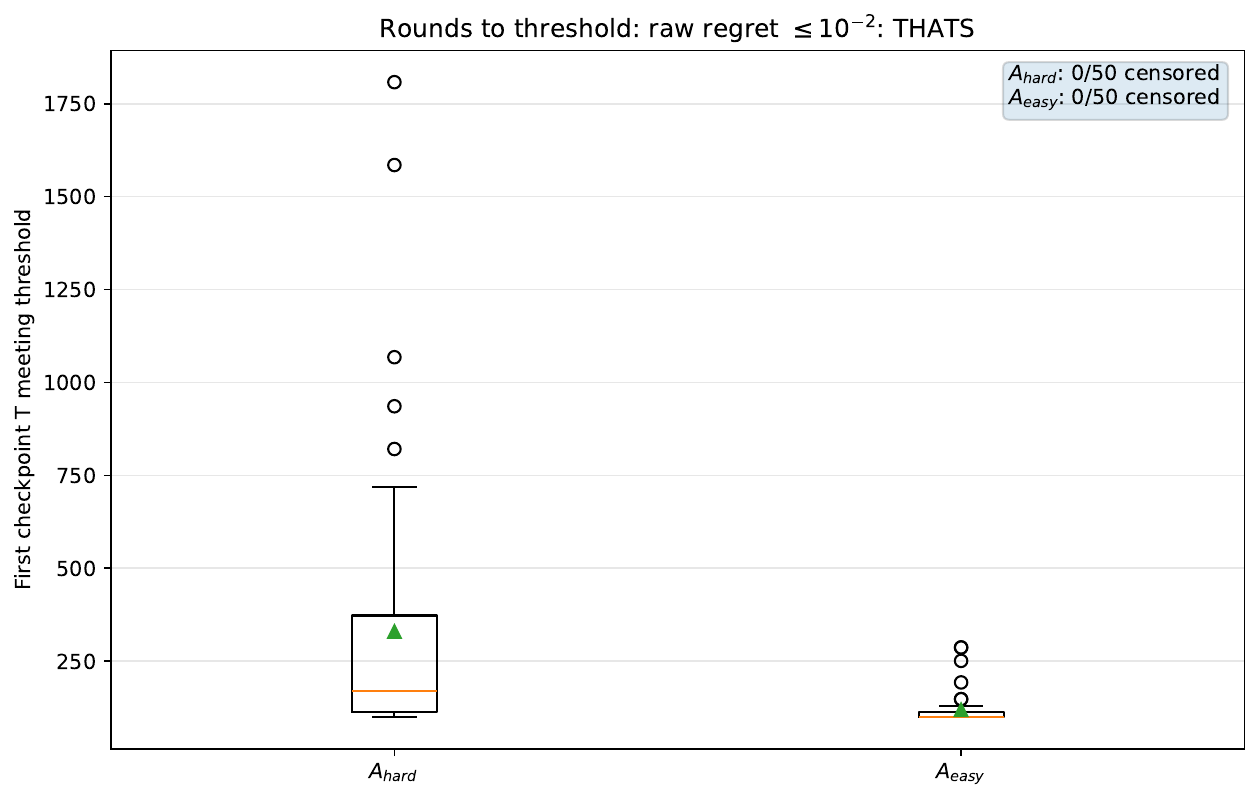}
        \caption{Rounds to raw threshold}
        \label{fig:thats_rounds_to_raw_threshold}
    \end{subfigure}
    \hfill
    \begin{subfigure}[t]{0.48\textwidth}
        \centering
        \includegraphics[width=\linewidth]{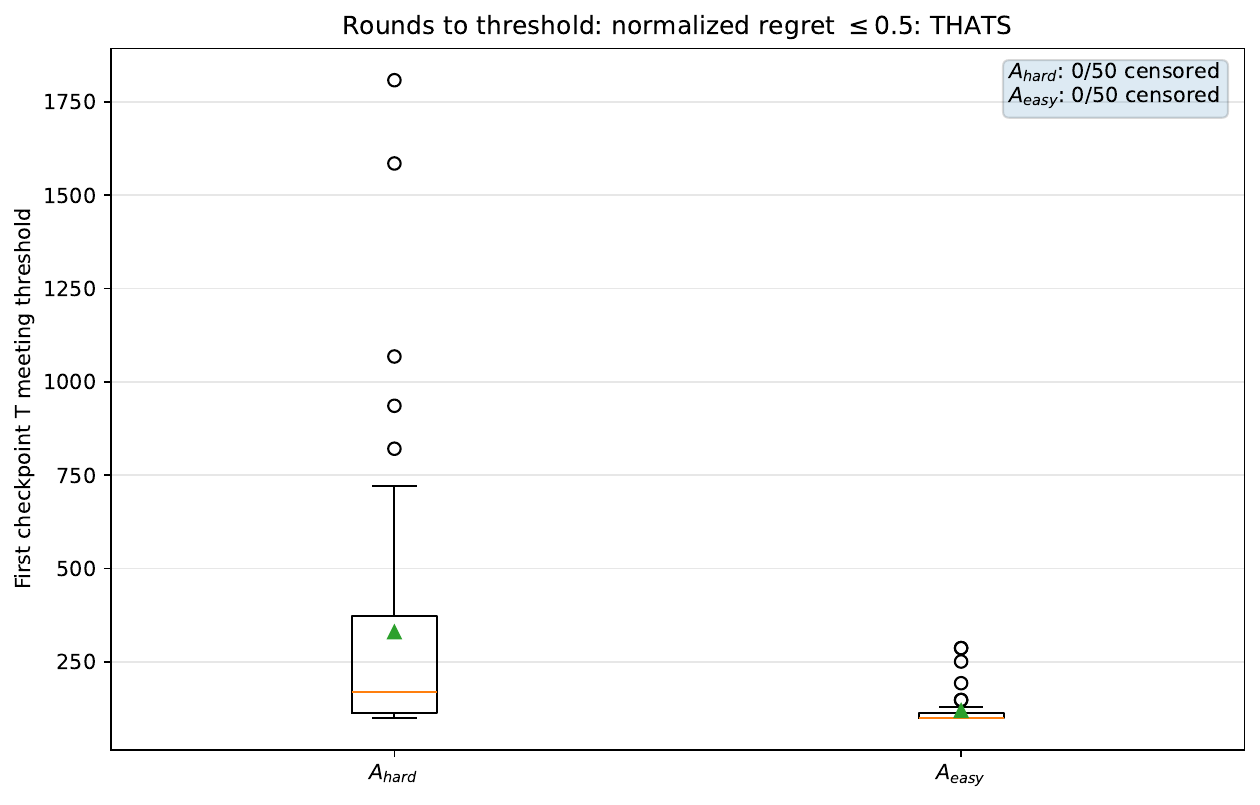}
        \caption{Rounds to normalized threshold}
        \label{fig:thats_rounds_to_normalized_threshold}
    \end{subfigure}

    \caption{Metrics for THATS.}
    \label{fig:thats-metrics}
\end{figure}

\begin{figure}[t]
    \centering

    \begin{subfigure}[t]{0.48\textwidth}
        \centering
        \includegraphics[width=\linewidth]{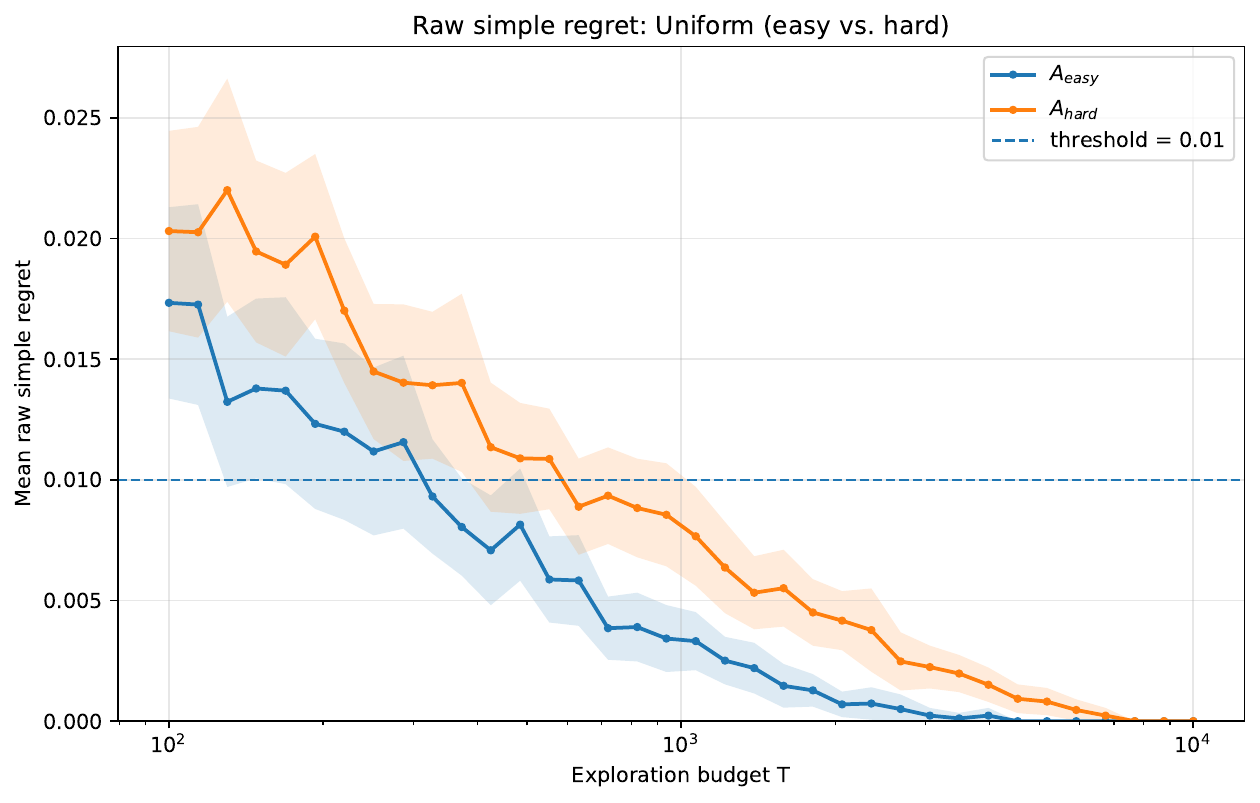}
        \caption{Raw regret}
        \label{fig:uniform_raw_regret}
    \end{subfigure}
    \hfill
    \begin{subfigure}[t]{0.48\textwidth}
        \centering
        \includegraphics[width=\linewidth]{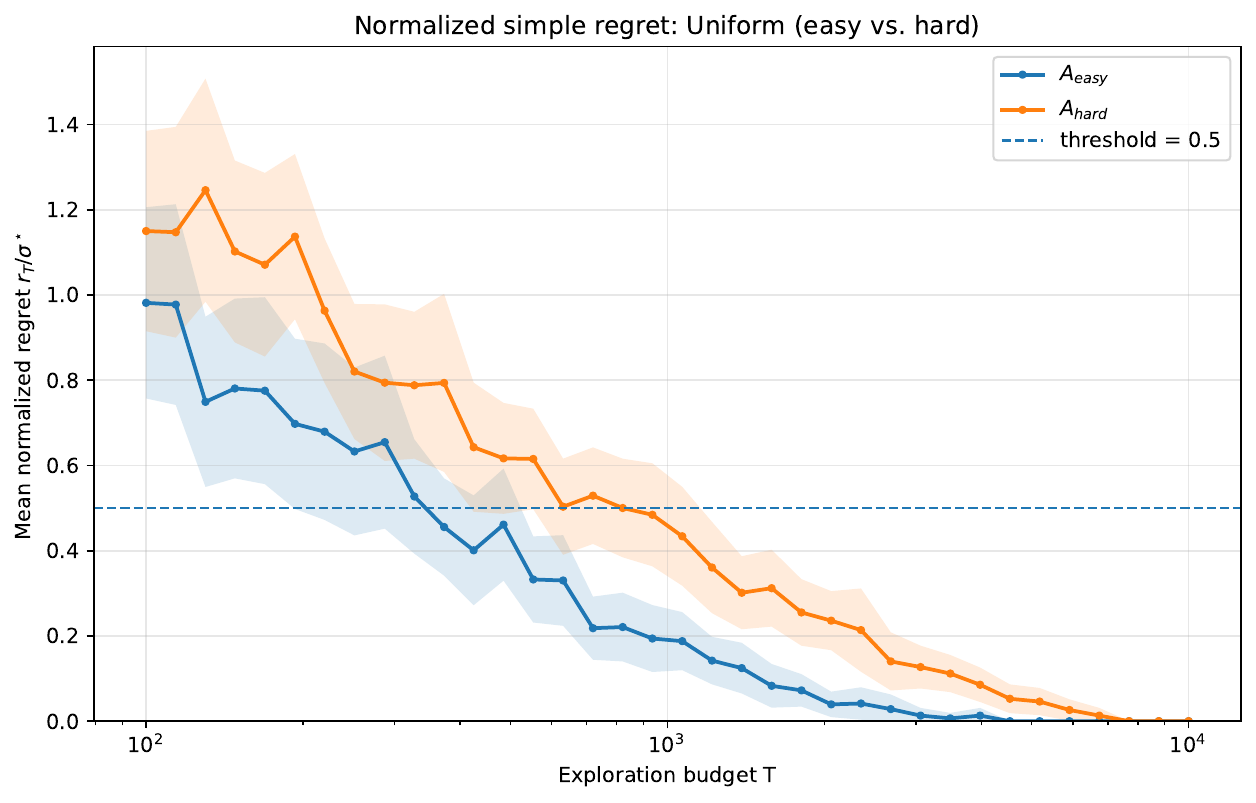}
        \caption{Normalized regret}
        \label{fig:uniform_normalized_regret}
    \end{subfigure}

    \vspace{0.75em}

    \begin{subfigure}[t]{0.48\textwidth}
        \centering
        \includegraphics[width=\linewidth]{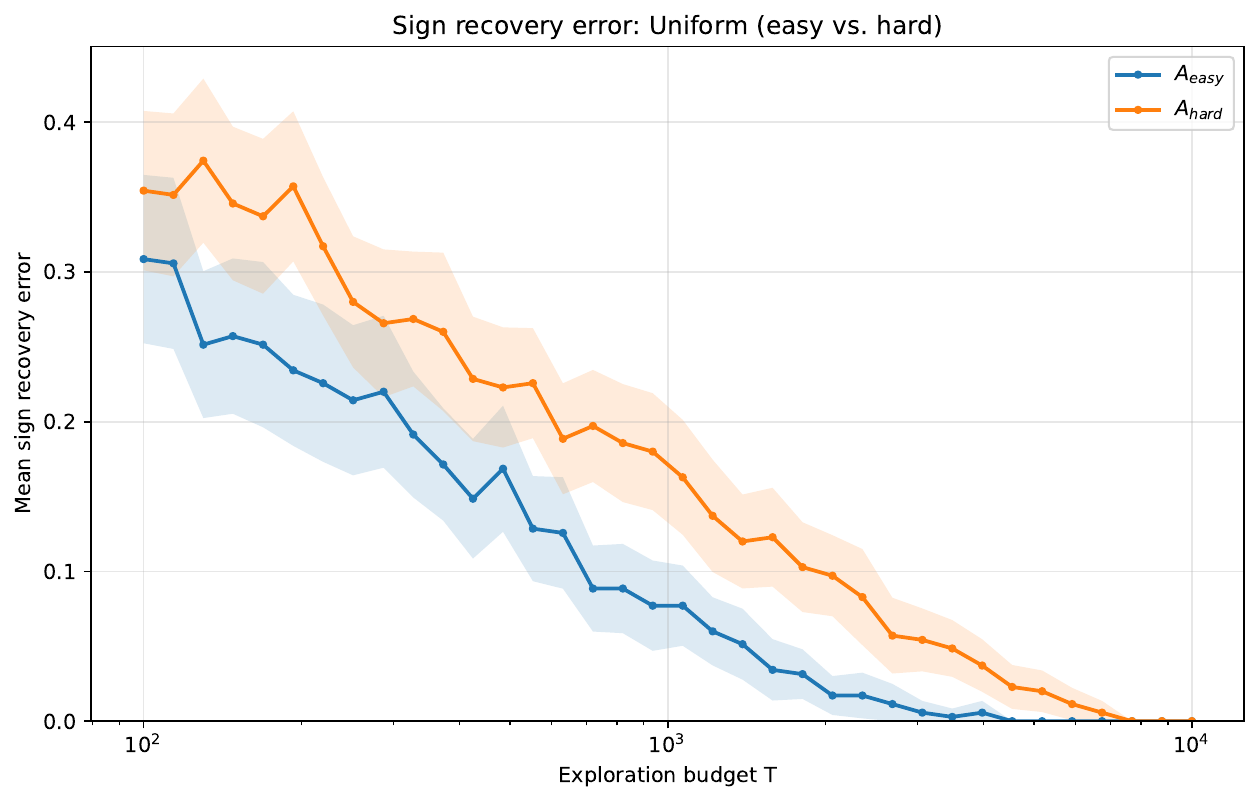}
        \caption{Sign recovery error}
        \label{fig:uniform_sign_recovery_error}
    \end{subfigure}
    \hfill
    \begin{subfigure}[t]{0.48\textwidth}
        \centering
        \includegraphics[width=\linewidth]{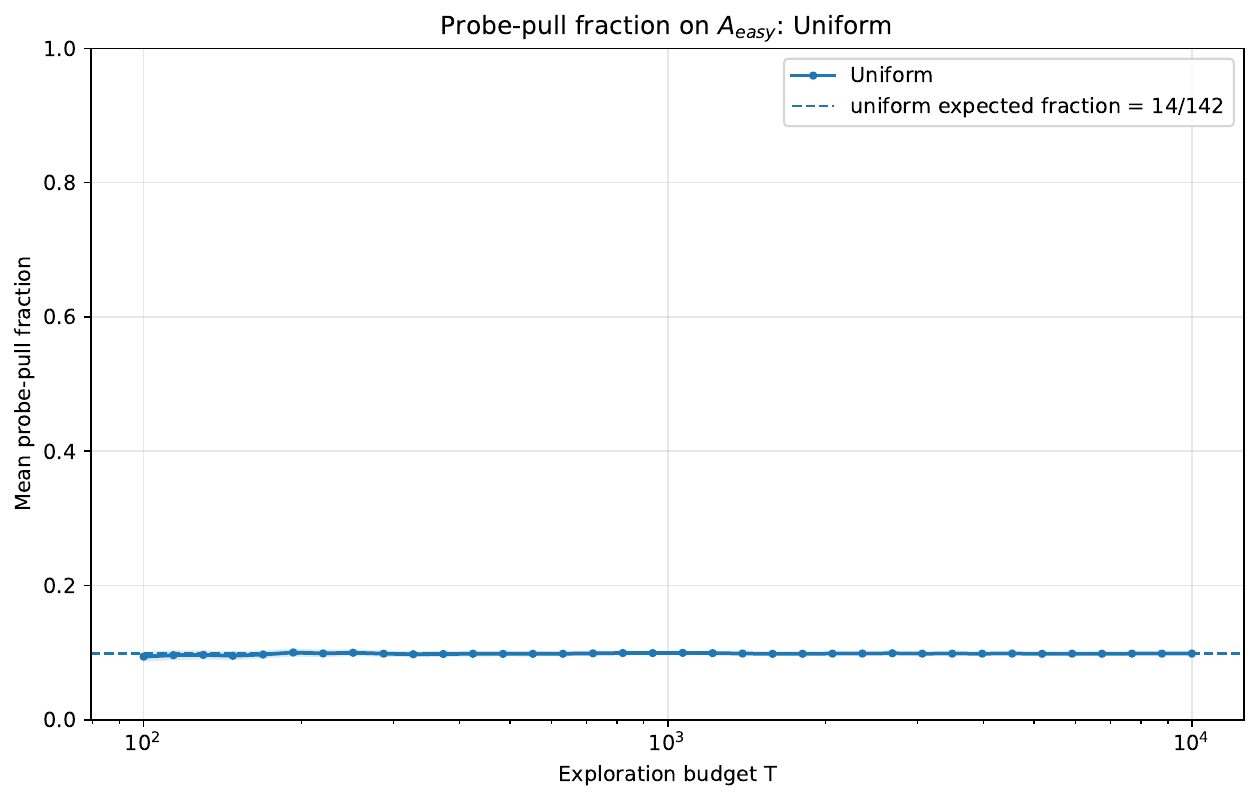}
        \caption{Probe fraction}
        \label{fig:uniform_probe_fraction}
    \end{subfigure}

    \vspace{0.75em}

    \begin{subfigure}[t]{0.48\textwidth}
        \centering
        \includegraphics[width=\linewidth]{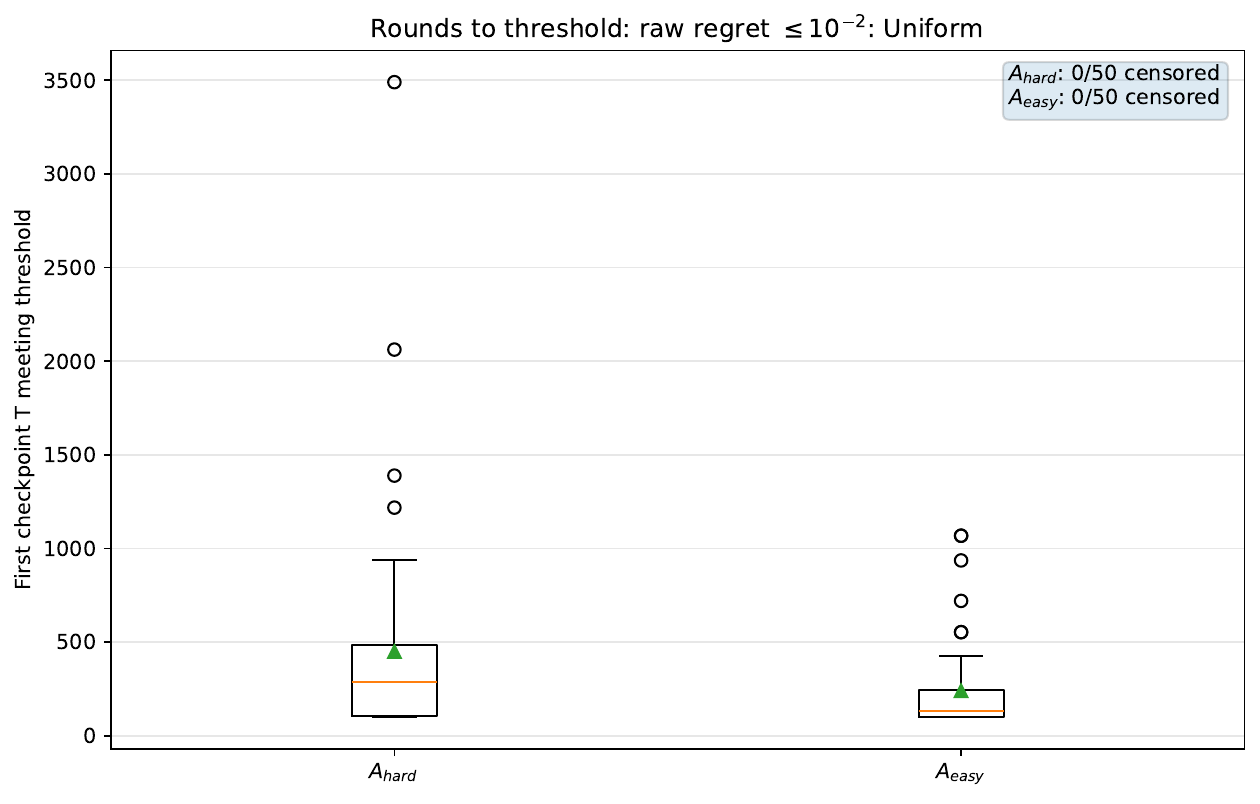}
        \caption{Rounds to raw threshold}
        \label{fig:uniform_rounds_to_raw_threshold}
    \end{subfigure}
    \hfill
    \begin{subfigure}[t]{0.48\textwidth}
        \centering
        \includegraphics[width=\linewidth]{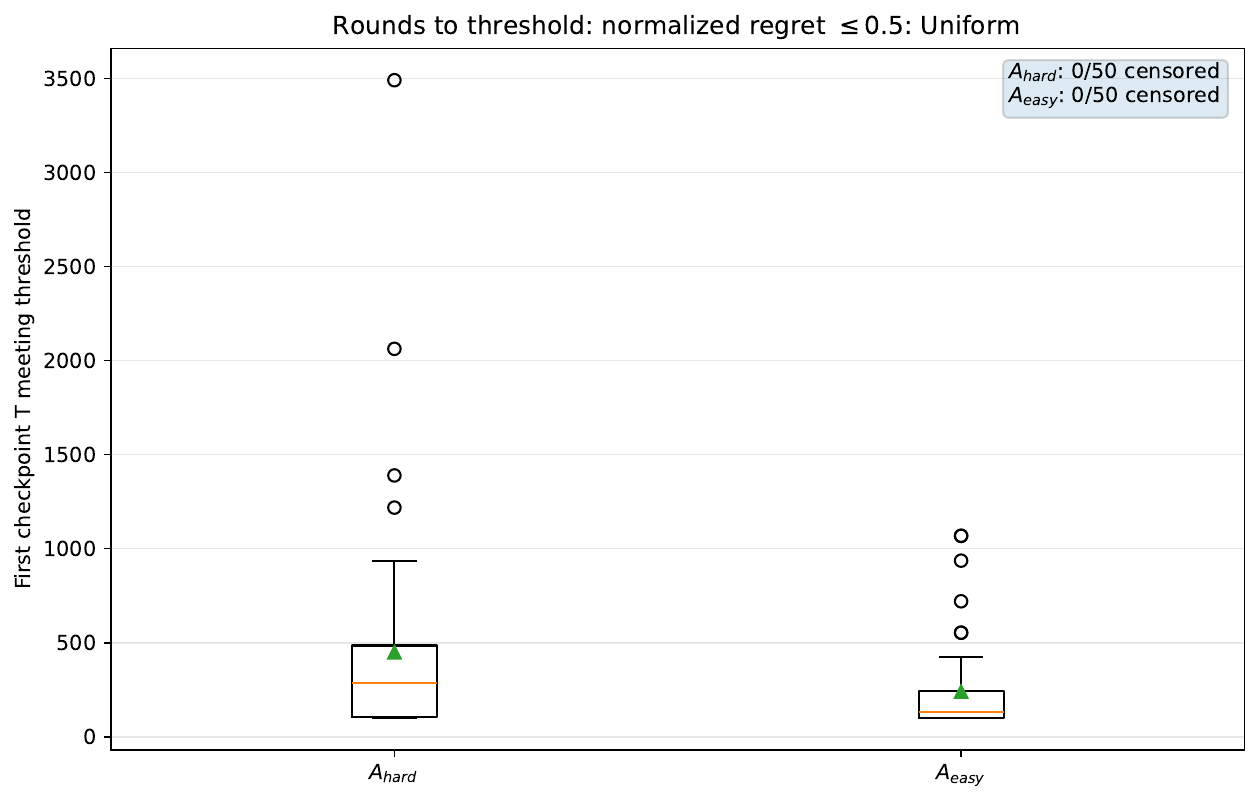}
        \caption{Rounds to normalized threshold}
        \label{fig:uniform_rounds_to_normalized_threshold}
    \end{subfigure}

    \caption{Metrics for uniform sampling.}
    \label{fig:uniform-metrics}
\end{figure}

\begin{figure}[t]
    \centering

    \begin{subfigure}[t]{0.48\textwidth}
        \centering
        \includegraphics[width=\linewidth]{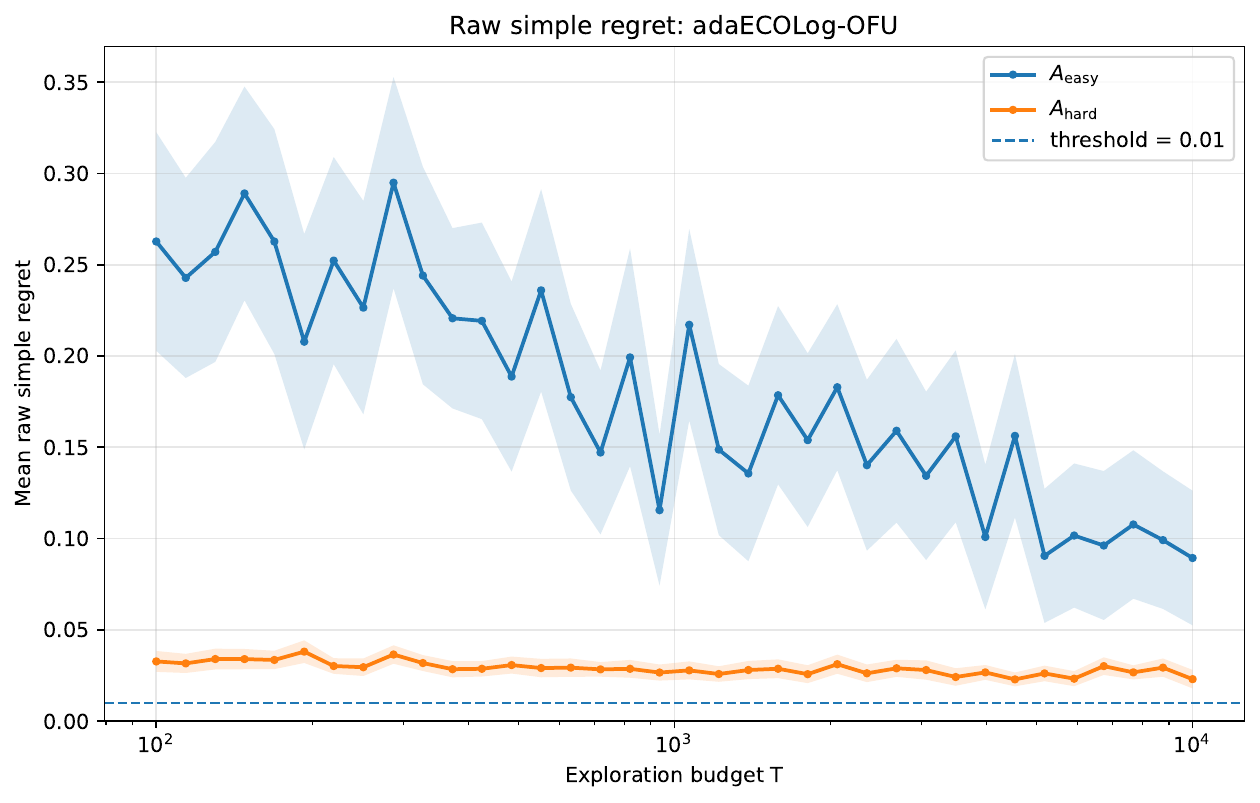}
        \caption{Raw regret}
        \label{fig:ada_ecolog_ofu_raw_regret}
    \end{subfigure}
    \hfill
    \begin{subfigure}[t]{0.48\textwidth}
        \centering
        \includegraphics[width=\linewidth]{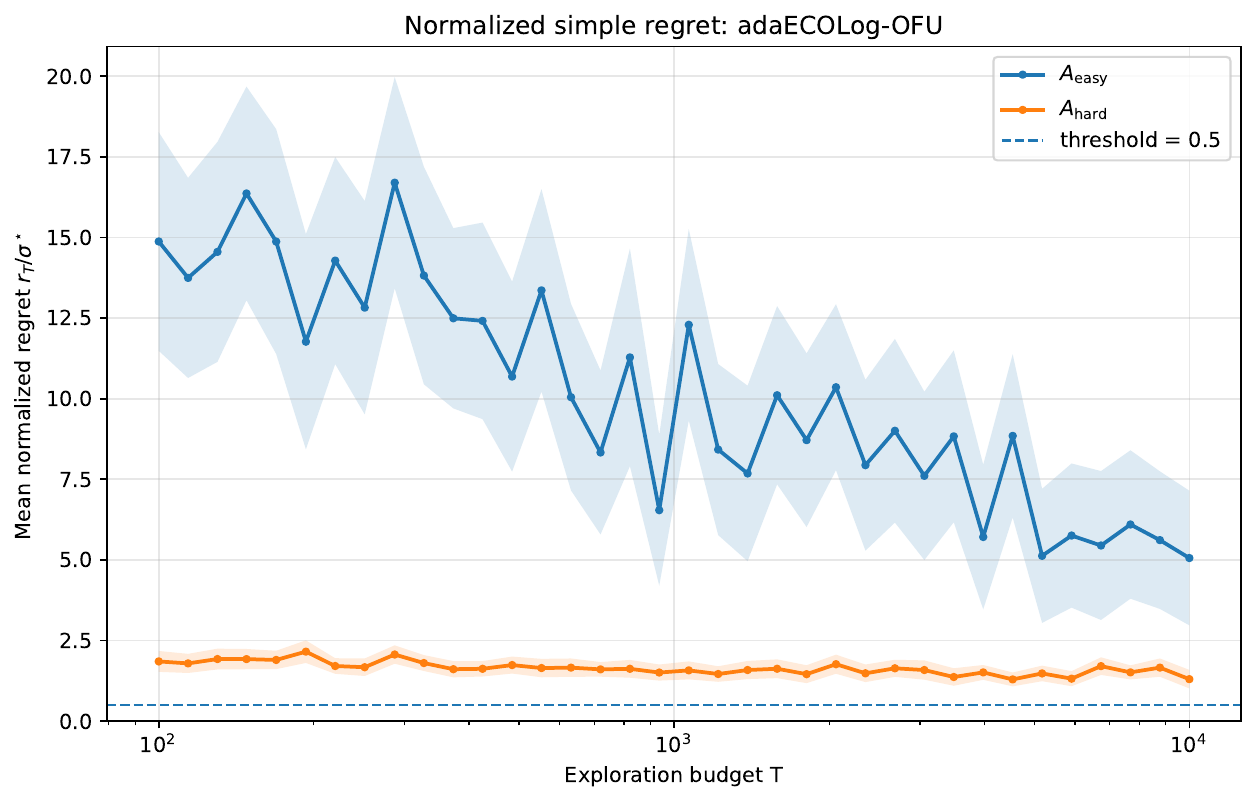}
        \caption{Normalized regret}
        \label{fig:ada_ecolog_ofu_normalized_regret}
    \end{subfigure}

    \vspace{0.75em}

    \begin{subfigure}[t]{0.48\textwidth}
        \centering
        \includegraphics[width=\linewidth]{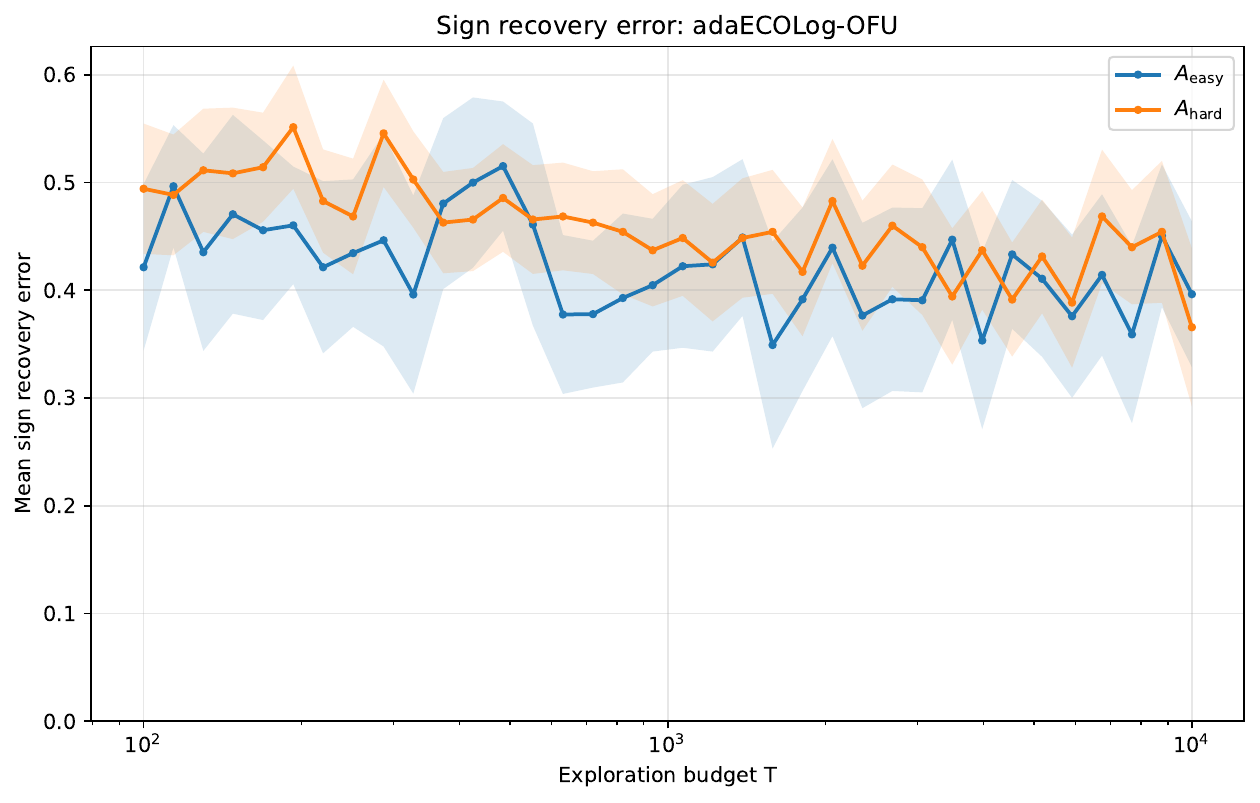}
        \caption{Sign recovery error}
        \label{fig:ada_ecolog_ofu_sign_recovery_error}
    \end{subfigure}
    \hfill
    \begin{subfigure}[t]{0.48\textwidth}
        \centering
        \includegraphics[width=\linewidth]{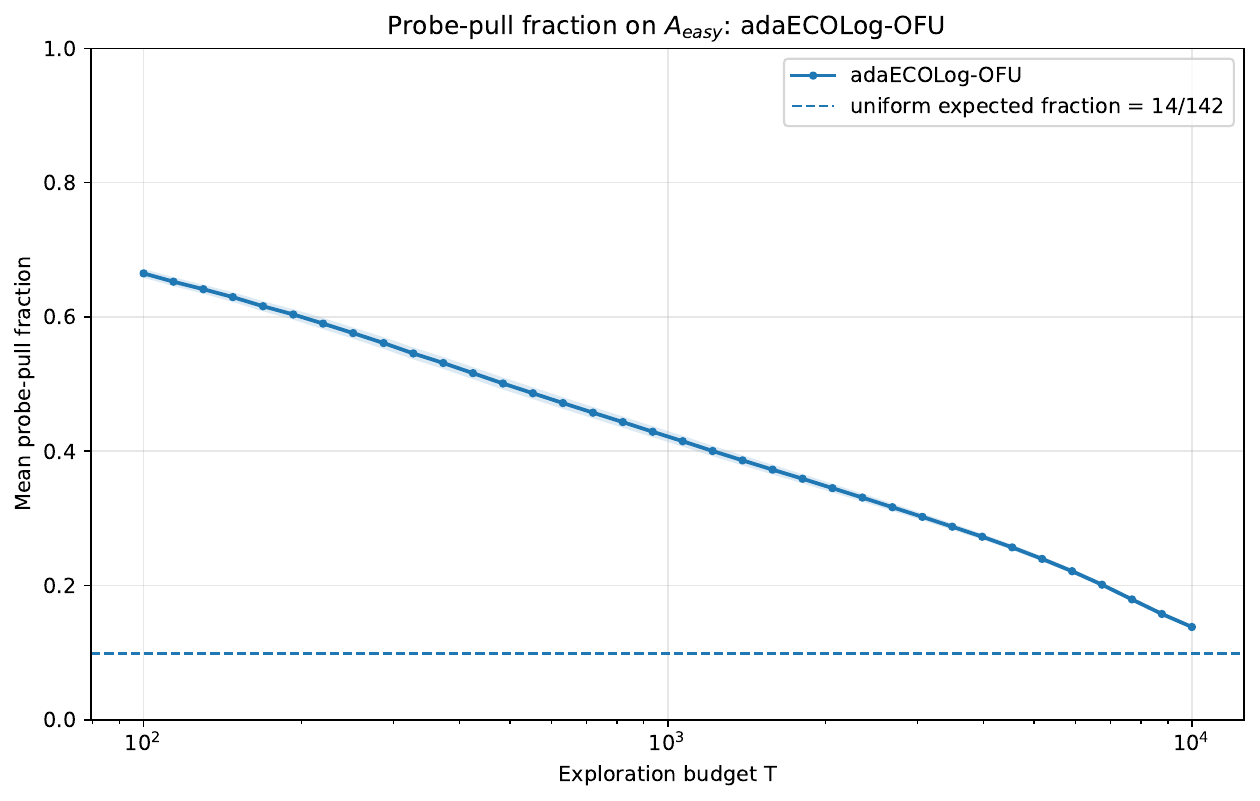}
        \caption{Probe fraction}
        \label{fig:ada_ecolog_ofu_probe_fraction}
    \end{subfigure}

    \vspace{0.75em}

    \begin{subfigure}[t]{0.48\textwidth}
        \centering
        \includegraphics[width=\linewidth]{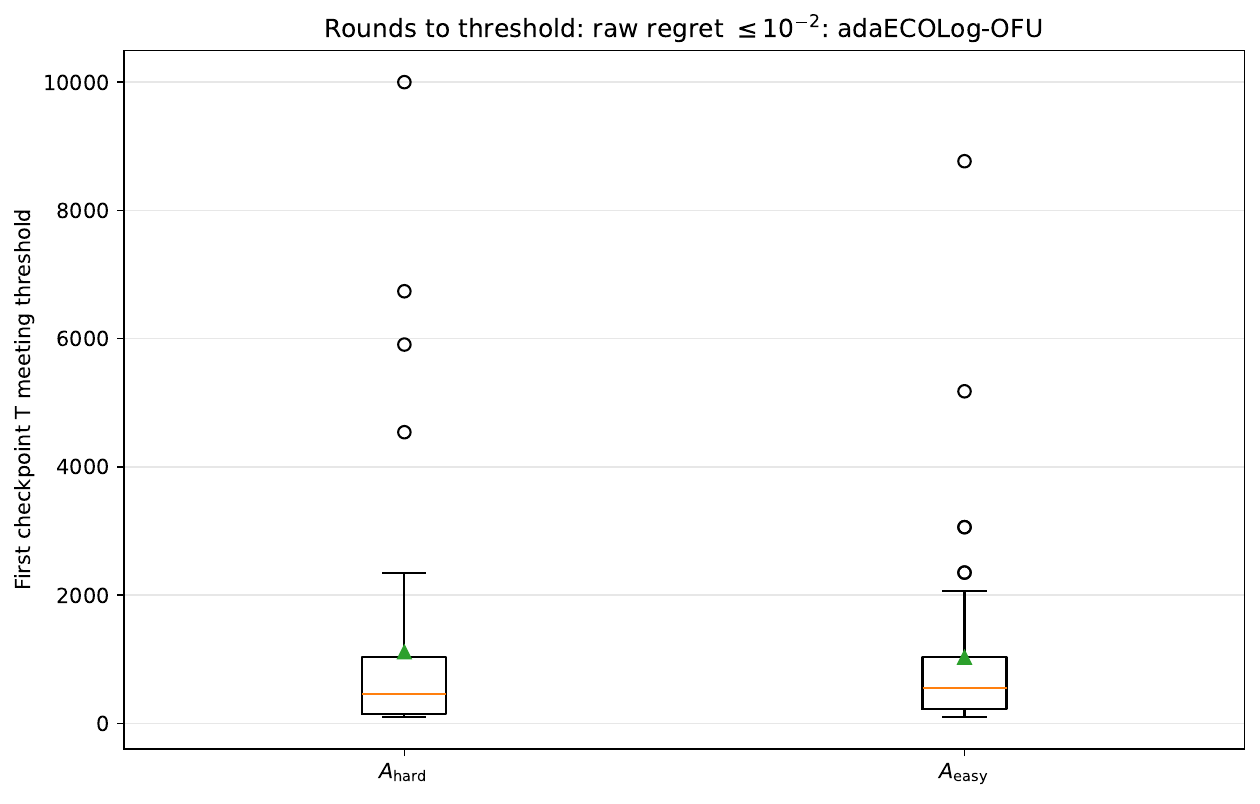}
        \caption{Rounds to raw threshold}
        \label{fig:ada_ecolog_ofu_rounds_to_raw_threshold}
    \end{subfigure}
    \hfill
    \begin{subfigure}[t]{0.48\textwidth}
        \centering
        \includegraphics[width=\linewidth]{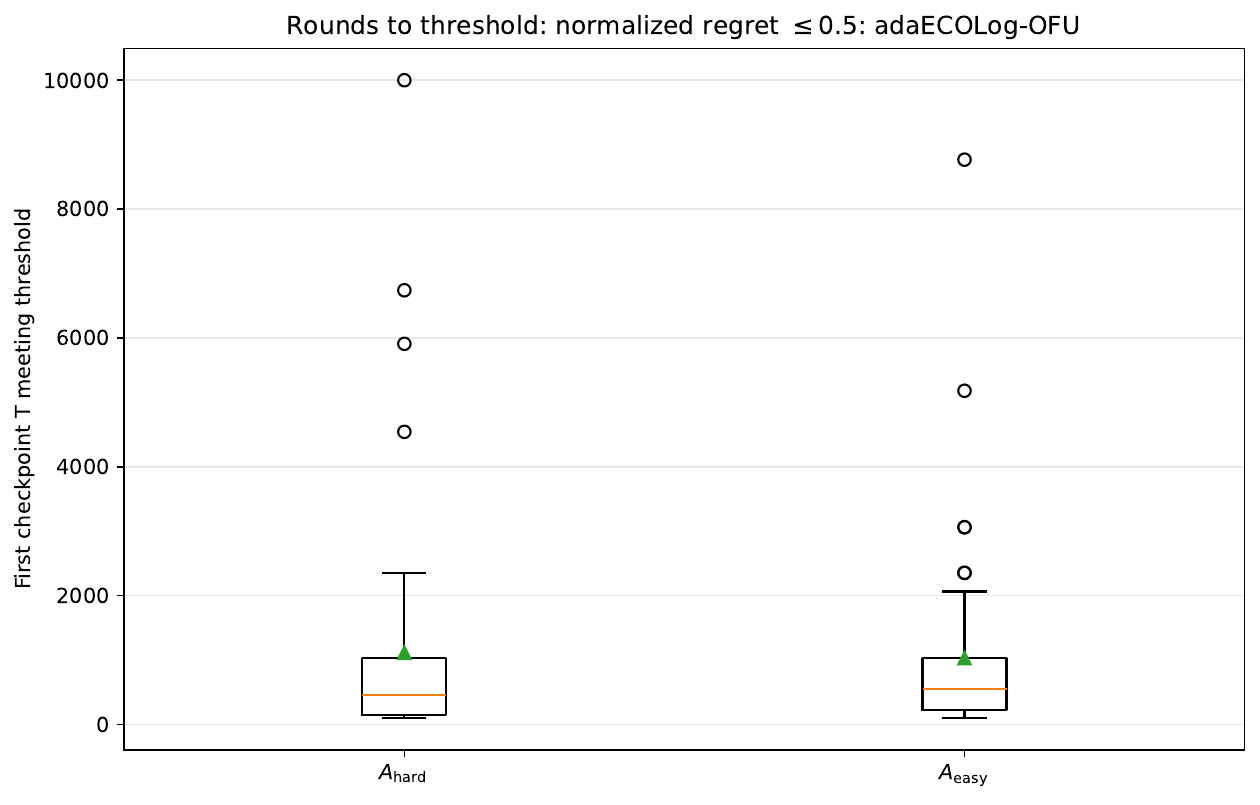}
        \caption{Rounds to normalized threshold}
        \label{fig:ada_ecolog_ofu_rounds_to_normalized_threshold}
    \end{subfigure}

    \caption{Metrics for Ada-EcoLog-OFU.}
    \label{fig:ada-ecolog-ofu-metrics}
\end{figure}

\begin{figure}[t]
    \centering

    \begin{subfigure}[t]{0.48\textwidth}
        \centering
        \includegraphics[width=\linewidth]{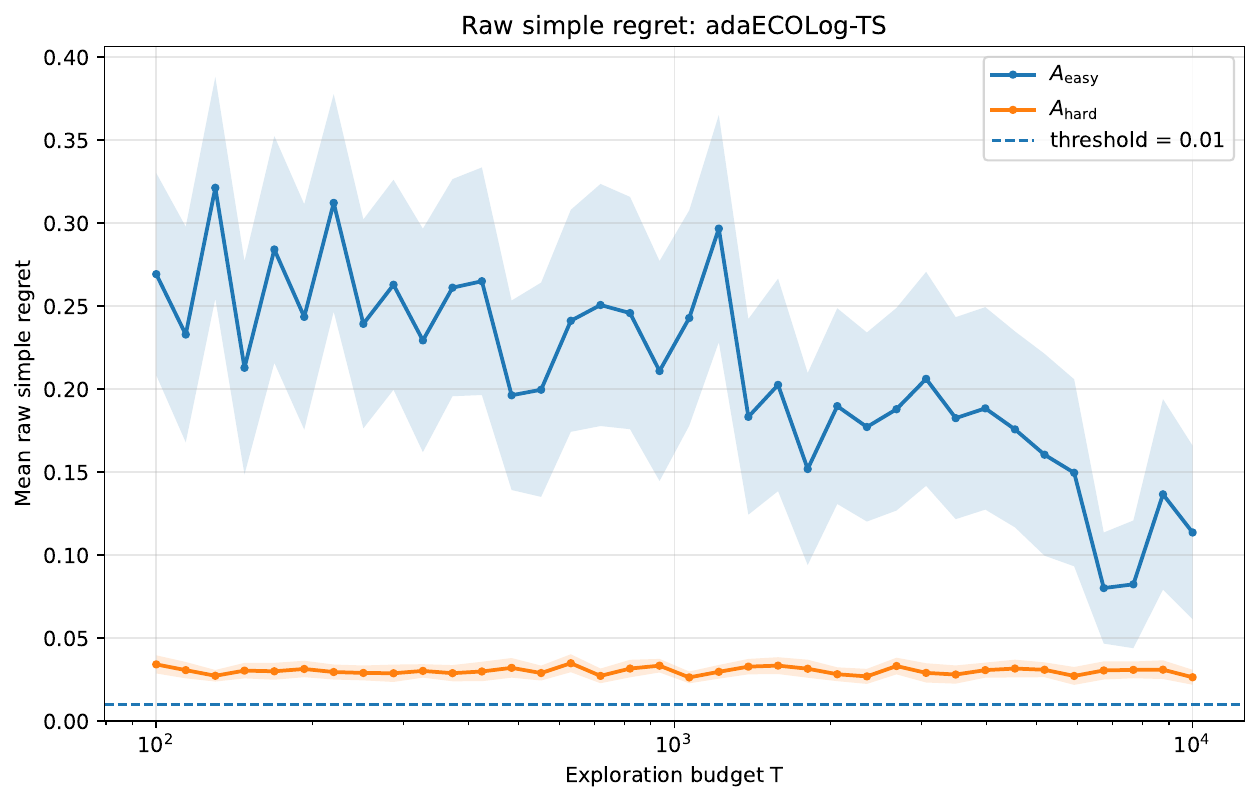}
        \caption{Raw regret}
        \label{fig:ada_ecolog_ts_raw_regret}
    \end{subfigure}
    \hfill
    \begin{subfigure}[t]{0.48\textwidth}
        \centering
        \includegraphics[width=\linewidth]{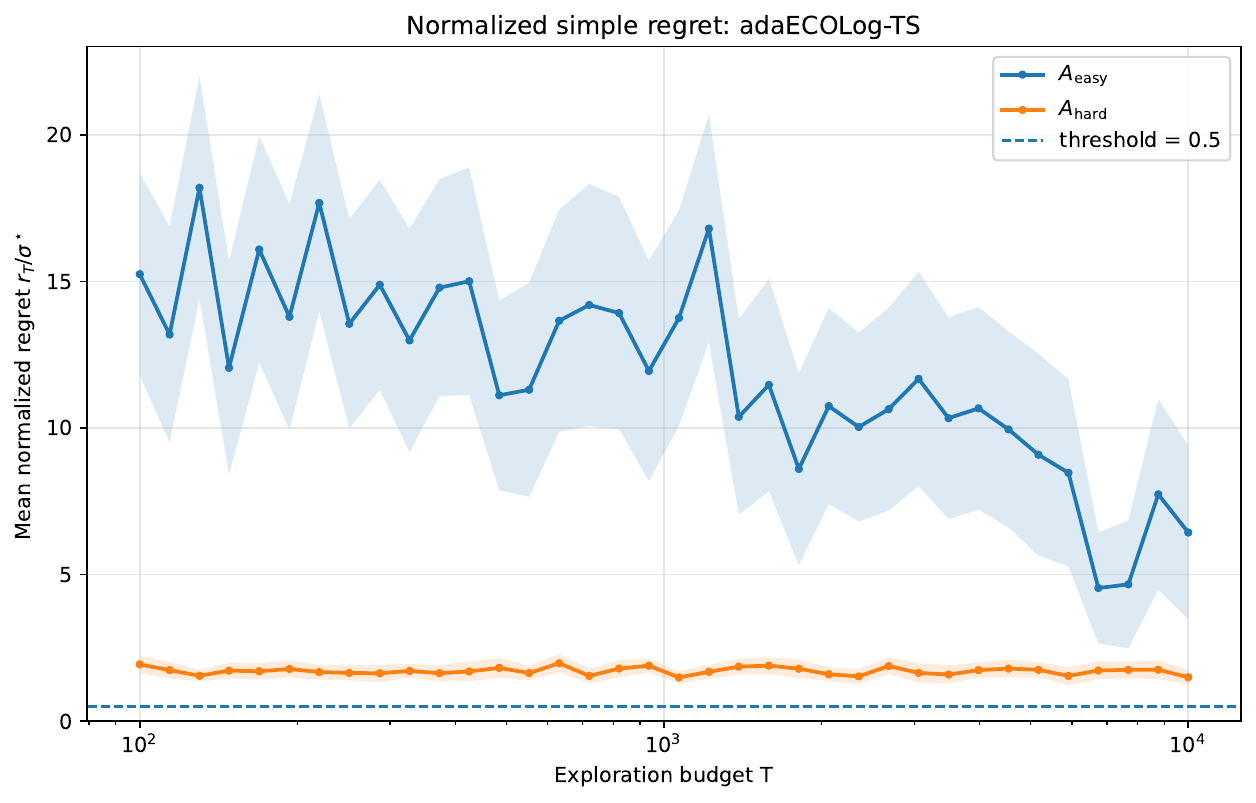}
        \caption{Normalized regret}
        \label{fig:ada_ecolog_ts_normalized_regret}
    \end{subfigure}

    \vspace{0.75em}

    \begin{subfigure}[t]{0.48\textwidth}
        \centering
        \includegraphics[width=\linewidth]{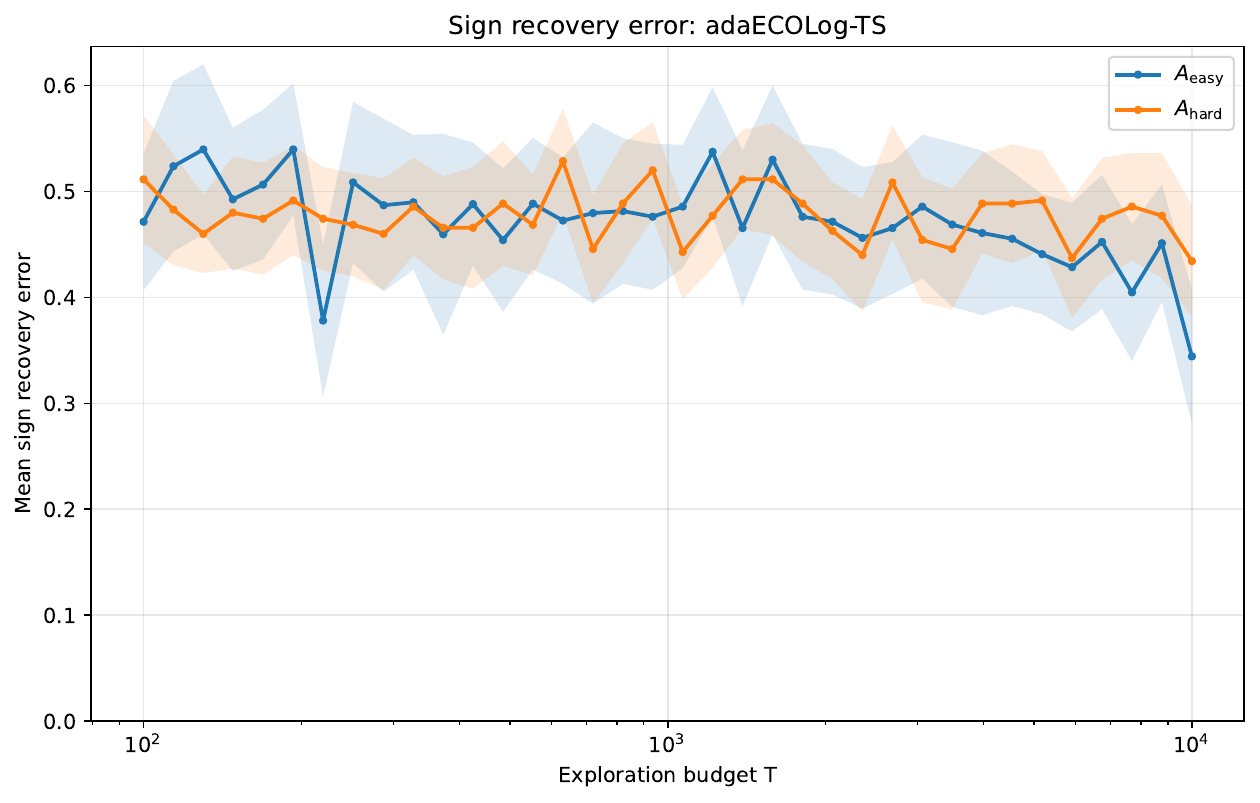}
        \caption{Sign recovery error}
        \label{fig:ada_ecolog_ts_sign_recovery_error}
    \end{subfigure}
    \hfill
    \begin{subfigure}[t]{0.48\textwidth}
        \centering
        \includegraphics[width=\linewidth]{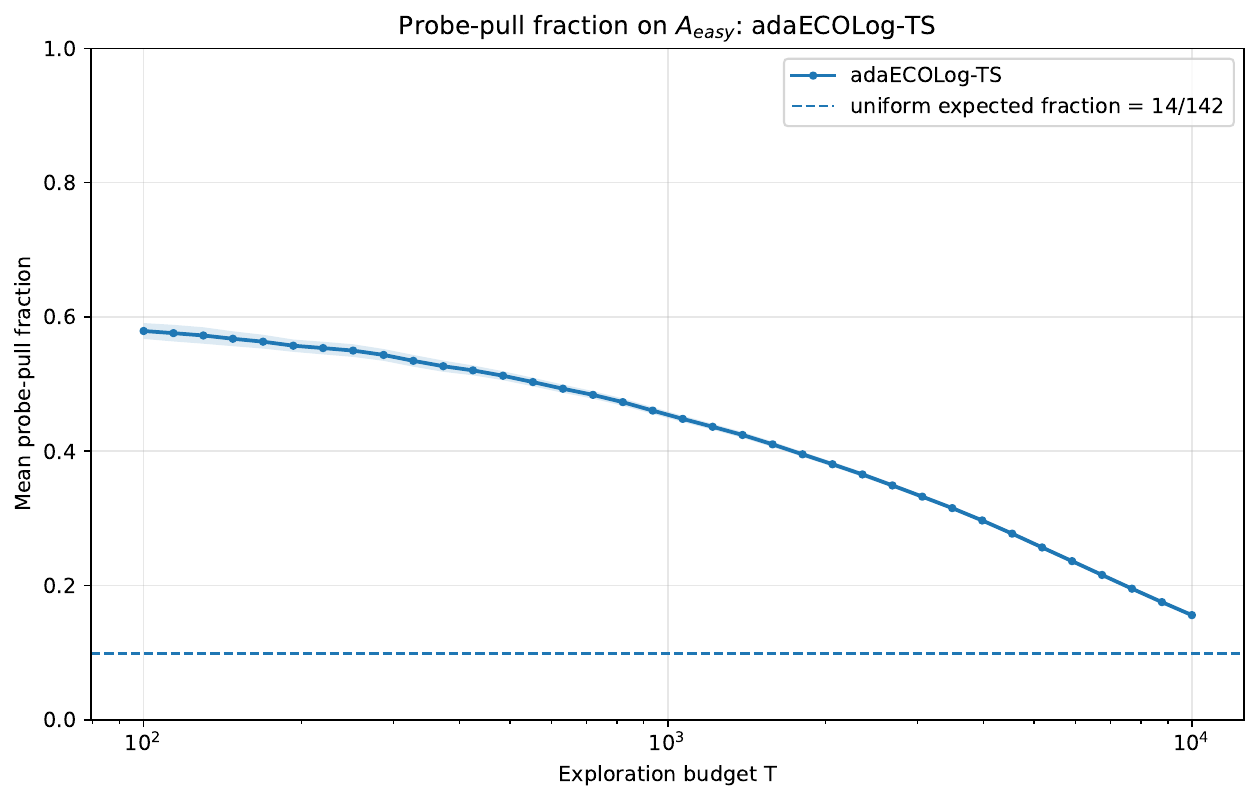}
        \caption{Probe fraction}
        \label{fig:ada_ecolog_ts_probe_fraction}
    \end{subfigure}

    \vspace{0.75em}

    \begin{subfigure}[t]{0.48\textwidth}
        \centering
        \includegraphics[width=\linewidth]{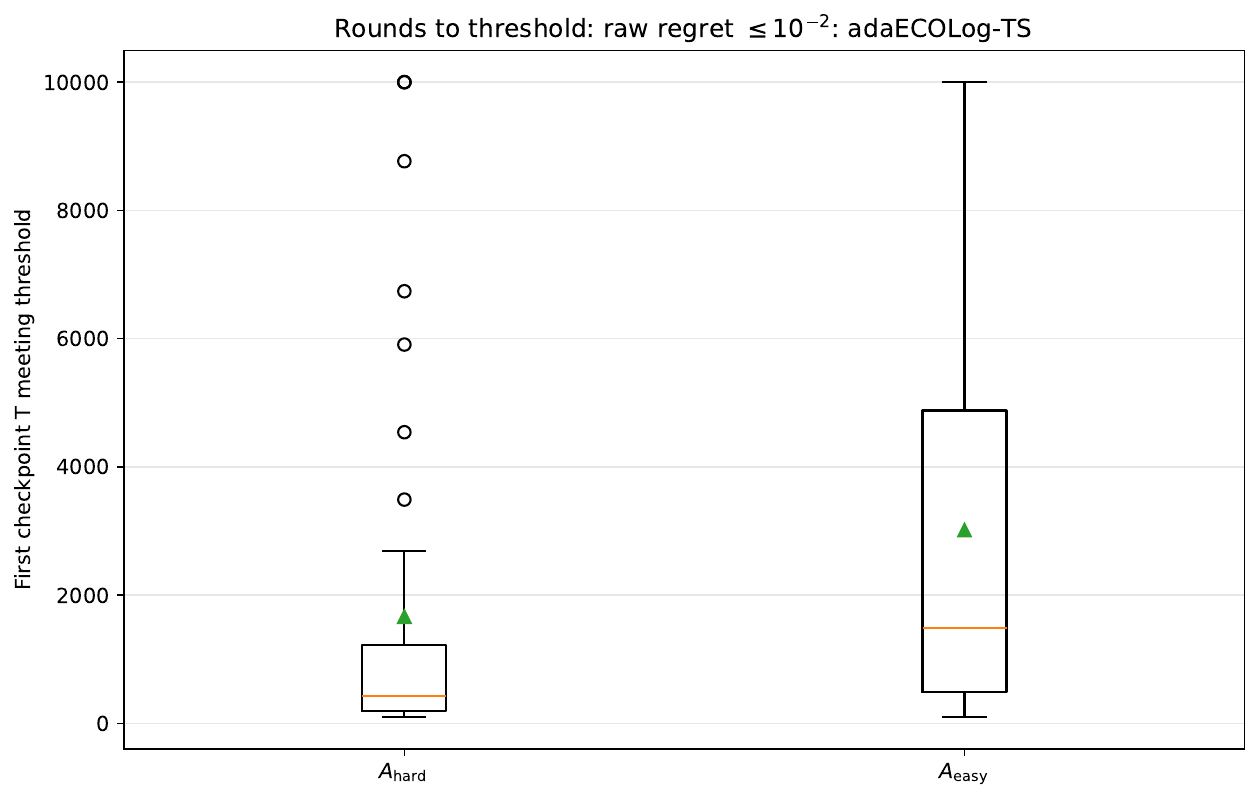}
        \caption{Rounds to raw threshold}
        \label{fig:ada_ecolog_ts_rounds_to_raw_threshold}
    \end{subfigure}
    \hfill
    \begin{subfigure}[t]{0.48\textwidth}
        \centering
        \includegraphics[width=\linewidth]{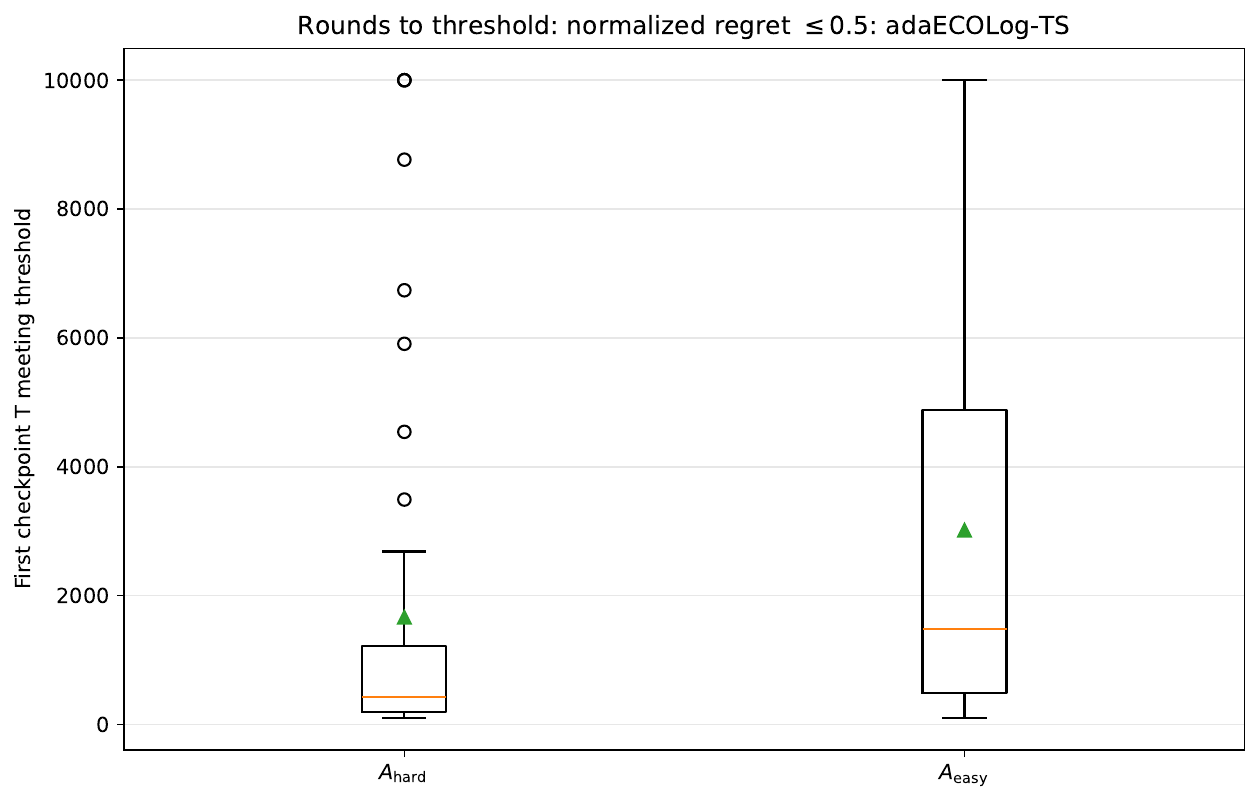}
        \caption{Rounds to normalized threshold}
        \label{fig:ada_ecolog_ts_rounds_to_normalized_threshold}
    \end{subfigure}

    \caption{Metrics for Ada-EcoLog-TS.}
    \label{fig:ada-ecolog-ts-metrics}
\end{figure}

\paragraph{Results and interpretation}
The results support the qualitative picture above. Moving from $\cA_{\mathrm{hard}}$ to $\cA_{\mathrm{easy}}$ substantially reduces both raw and normalized simple regret, confirming that the shifted saturated family used in the lower bound is a genuine worst-case geometry rather than a representative one. The gain is strongest and most reliable for $\MULog$ and $\THATS$: on the easy geometry, both methods reach the target regret thresholds within only a few hundred rounds in most runs, and the sign-recovery error falls to essentially zero far earlier than on $\cA_{\mathrm{hard}}$.

Cumulative Thompson sampling also benefits from the easier geometry, but the improvement is weaker and significantly more variable. This is exactly the regime where the pure-exploration perspective matters: once informative low-reward probe arms are available, the curvature-aware methods use them more efficiently for the final recommendation task. Uniform exploration benefits as well, since the informative probes are now sampled with nonzero frequency, but it remains less sample-efficient and less reliable than $\MULog$ and $\THATS$ on both the threshold summaries and the sign-recovery diagnostic. The main lesson is therefore not that only pure-exploration methods improve on $\cA_{\mathrm{easy}}$, but that they exploit the additional information most effectively.

\clearpage
\section*{NeurIPS Paper Checklist}

\begin{enumerate}

\item {\bf Claims}
    \item[] Question: Do the main claims made in the abstract and introduction accurately reflect the paper's contributions and scope?
    \item[] Answer: \answerYes{} 
    \item[] Justification: {theorem 1, \cref{thm:mulog-regret,thm:THATS-regret}}

\item {\bf Limitations}
    \item[] Question: Does the paper discuss the limitations of the work performed by the authors?
    \item[] Answer: \answerYes{} 
    \item[] Justification: see conclusion

\item {\bf Theory assumptions and proofs}
    \item[] Question: For each theoretical result, does the paper provide the full set of assumptions and a complete (and correct) proof?
    \item[] Answer: \answerYes{} 
    \item[] Justification: See section 3 and appendix A-F

    \item {\bf Experimental result reproducibility}
    \item[] Question: Does the paper fully disclose all the information needed to reproduce the main experimental results of the paper to the extent that it affects the main claims and/or conclusions of the paper (regardless of whether the code and data are provided or not)?
    \item[] Answer: \answerYes{} 
    \item[] Justification: See appendix G

\item {\bf Open access to data and code}
    \item[] Question: Does the paper provide open access to the data and code, with sufficient instructions to faithfully reproduce the main experimental results, as described in supplemental material?
    \item[] Answer: \answerYes{} 
    \item[] Justification: See supplementary materials

\item {\bf Experimental setting/details}
    \item[] Question: Does the paper specify all the training and test details (e.g., data splits, hyperparameters, how they were chosen, type of optimizer) necessary to understand the results?
    \item[] Answer: \answerYes{} 
    \item[] Justification: See appendix G

\item {\bf Experiment statistical significance}
    \item[] Question: Does the paper report error bars suitably and correctly defined or other appropriate information about the statistical significance of the experiments?
    \item[] Answer: \answerYes{} 
    \item[] Justification: See appendix G

\item {\bf Experiments compute resources}
    \item[] Question: For each experiment, does the paper provide sufficient information on the computer resources (type of compute workers, memory, time of execution) needed to reproduce the experiments?
    \item[] Answer: \answerNo{} 
    \item[] Justification: Any CPU would work
    
\item {\bf Code of ethics}
    \item[] Question: Does the research conducted in the paper conform, in every respect, with the NeurIPS Code of Ethics \url{https://neurips.cc/public/EthicsGuidelines}?
    \item[] Answer: \answerYes{} 
    \item[] Justification: The research conducted in the paper conform, in every respect, with the NeurIPS Code of Ethics \url{https://neurips.cc/public/EthicsGuidelines}

\item {\bf Broader impacts}
    \item[] Question: Does the paper discuss both potential positive societal impacts and negative societal impacts of the work performed?
    \item[] Answer: \answerNA{} 
    \item[] Justification: there is no societal impact of the work performed.
    
\item {\bf Safeguards}
    \item[] Question: Does the paper describe safeguards that have been put in place for responsible release of data or models that have a high risk for misuse (e.g., pre-trained language models, image generators, or scraped datasets)?
    \item[] Answer: \answerNA{} 
    \item[] Justification: no such risks.

\item {\bf Licenses for existing assets}
    \item[] Question: Are the creators or original owners of assets (e.g., code, data, models), used in the paper, properly credited and are the license and terms of use explicitly mentioned and properly respected?
    \item[] Answer: \answerYes{} 
    \item[] Justification: Apache License 2.0, https://github.com/criteo-research/logistic\_bandit, see appendix G

\item {\bf New assets}
    \item[] Question: Are new assets introduced in the paper well documented and is the documentation provided alongside the assets?
    \item[] Answer: \answerNA{} 
    \item[] Justification: the paper does not release new assets

\item {\bf Crowdsourcing and research with human subjects}
    \item[] Question: For crowdsourcing experiments and research with human subjects, does the paper include the full text of instructions given to participants and screenshots, if applicable, as well as details about compensation (if any)? 
    \item[] Answer: \answerNA{} 
    \item[] Justification: the paper does not involve crowdsourcing nor research with human subjects

\item {\bf Institutional review board (IRB) approvals or equivalent for research with human subjects}
    \item[] Question: Does the paper describe potential risks incurred by study participants, whether such risks were disclosed to the subjects, and whether Institutional Review Board (IRB) approvals (or an equivalent approval/review based on the requirements of your country or institution) were obtained?
    \item[] Answer: \answerNA{} 
    \item[] Justification: the paper does not involve crowdsourcing nor research with human subjects.

\item {\bf Declaration of LLM usage}
    \item[] Question: Does the paper describe the usage of LLMs if it is an important, original, or non-standard component of the core methods in this research? Note that if the LLM is used only for writing, editing, or formatting purposes and does \emph{not} impact the core methodology, scientific rigor, or originality of the research, declaration is not required.
    \item[] Answer: \answerNA{} 
    \item[] Justification:  the core method development in this research does not involve LLMs as any important, original, or non-standard components

\end{enumerate}
\end{document}